\def\year{2020}\relax
\documentclass[letterpaper]{article} % DO NOT CHANGE THIS
\usepackage{aaai20}  % DO NOT CHANGE THIS
\usepackage{times}  % DO NOT CHANGE THIS
\usepackage{helvet} % DO NOT CHANGE THIS
\usepackage{courier}  % DO NOT CHANGE THIS
\usepackage[hyphens]{url}  % DO NOT CHANGE THIS
\usepackage{graphicx} % DO NOT CHANGE THIS
\usepackage{graphicx}  % DO NOT CHANGE THIS
\usepackage{hyperref} % Add a link to your document

\usepackage[page]{appendix}

\usepackage{amsmath,amssymb}
\usepackage{color}
\usepackage{stackengine}

\usepackage{amsthm}
\usepackage{amsmath}
\usepackage{amsfonts}
\usepackage{mathtools}
\usepackage{algorithm}
\usepackage{algorithmic}
\usepackage{tikz}
\usepackage{multirow}
\usetikzlibrary{shapes.geometric, arrows, automata, positioning}

\usepackage{dsfont}
\usepackage{enumerate}   

\DeclarePairedDelimiter\br{(}{)}% ( )
\DeclarePairedDelimiter\brs{[}{]}% [ ]
\DeclarePairedDelimiter\brc{\{}{\}}% { }
\DeclarePairedDelimiter\inner{\langle}{\rangle}% < >
\DeclarePairedDelimiter\abs{\lvert}{\rvert}% < >

\usepackage{verbatim}

\usepackage{xcolor}
\definecolor{darkgreen}{rgb}{0,0.5,0}

\usepackage{bbm}

\DeclareMathOperator*{\argmin}{arg\,min}
\DeclareMathOperator{\aset}{\mathcal{A}}
\DeclareMathOperator{\sset}{\mathcal{S}}
\DeclareMathOperator{\Cmax}{C_{\text{max}}}
\DeclareMathOperator{\CmaxSquare}{C^2_{\text{max}}}

\DeclareMathOperator{\CmaxLambda}{C_{\text{max},\lambda}}
\DeclareMathOperator{\CmaxLambdaSquare}{C^2_{\text{max},\lambda}}

\DeclareMathOperator{\dom}{\text{dom}}

\def\A{{\mathcal A}}

\def\vpik{v^{\pi_k}}
\def\vHpik{v_\lambda^{\pi_k}}
\def\vHpi{v_\lambda^{\pi}}
\def\qHpi{q_\lambda^{\pi}}
\def\qHpik{q_\lambda^{\pi_k}}
\def\hatqHpik{\hat q_\lambda^{\pi_k}}

\def\E{{\mathbb E}}

\def\R{\mathbb R}

\def\vstar{v^*}
\def\vHstar{v_\lambda^*}

\def\pistar{\pi^*}

\def\pik{\pi_k}

\def\triangleq{:=}

\newcommand*\dkl[2]{d_{KL}(#1||#2)}
\newcommand*\bregman[2]{B_\omega\left(#1,#2\right)}
\newcommand*\bregmanShort[2]{B_\omega(#1,#2)}

\newcommand*\bregmanS[3]{B_\omega\left(#1;#2,#3\right)}
\newcommand*\omegaS[2]{\omega\left(#1;#2\right)}

\def\euclid{C_{\omega,1}}
\def\approxAgrowth{C_{\omega,2}}
\def\bregmanMax{C_{\omega,3}}

\def\omegaBound{D_\omega}
\def\gradBound{h_\omega}
\def\hatgradBound{\hat h_\omega}

\def\diffBound{A_\omega}
\def\HoeffdingBound{r_\omega}

\DeclareMathOperator{\simplex}{\Delta_{\A}}
\DeclareMathOperator{\ssimplex}{\Delta_\A^S}

\newcommand*\norm[1]{\left\|#1\right\|}

\usepackage{array}
\newcolumntype{P}[1]{>{\centering\arraybackslash}p{#1}}
\newcolumntype{L}[1]{>{\raggedright\arraybackslash}p{#1}}

\ifdefined\comments 
    \newcommand{\lior}[1]{\textcolor{blue}{\{Lior: #1\}}}
    \newcommand{\yon}[1]{\textcolor{darkgreen}{\{Yonathan: #1\}}}
    \newcommand{\shie}[1]{\textcolor{red}{\{Shie: #1\}}}
    \newcommand{\todo}[1]{\textcolor{red}{\{TODO: #1\}}}
    \newcommand{\nadav}[1]{\textcolor{purple}{\{Nadav: #1\}}}
\else
    \newcommand{\lior}[1]{}
    \newcommand{\yon}[1]{}
    \newcommand{\todo}[1]{}
    \newcommand{\nadav}[1]{}
    \newcommand{\shie}[1]{}
\fi
\usepackage{thmtools}

\usepackage{thmtools}
\usepackage{thm-restate}

\newtheorem{proposition}{Proposition}
\newtheorem{assumption}{Assumption}
\newtheorem{corollary}[proposition]{Corollary}
\newtheorem{lemma}[proposition]{Lemma}
\newtheorem{theorem}[proposition]{Theorem}
\newtheorem*{theorem*}{Theorem}

\newtheorem{defn}{Definition}
\newtheorem{remark}{Remark}

\newcommand{\citet}[1]{\citeauthor{#1} \shortcite{#1}}
\newcommand{\citep}{\cite}
\newcommand{\citealp}[1]{\citeauthor{#1} \citeyear{#1}}

\definecolor{darkgray}{rgb}{0.66, 0.66, 0.66}

\renewcommand{\appendixtocname}{List of Appendices}

\makeatletter
\let\oldappendix\appendices

\renewcommand{\appendices}{%
  \clearpage
  \renewcommand{\thesection}{\Alph{section}}
  \let\tf@toc\tf@app
  \addtocontents{app}{\protect\setcounter{tocdepth}{2}}
  \immediate\write\@auxout{%
    \string\let\string\tf@toc\string\tf@app^^J
  }
  \oldappendix
}%

\newcommand{\listofappendices}{%
  \begingroup
  \renewcommand{\contentsname}{\appendixtocname}
  \let\@oldstarttoc\@starttoc
  \def\@starttoc##1{\@oldstarttoc{app}}
  \tableofcontents% Reusing the code for \tableofcontents with different \contentsname and different file handle app
  \endgroup
}

\makeatother

\usepackage[symbol]{footmisc}

\title{Adaptive Trust Region Policy Optimization:\\
Global Convergence and Faster Rates for Regularized MDPs}

\author{{Lior Shani\textsuperscript{\rm $\dagger$}, Yonathan Efroni\textsuperscript{\rm $\dagger$}, Shie Mannor} \\
\textsuperscript{\rm $\dagger$} equal contribution\\
Technion - Israel Institute of Technology \\
Haifa, Israel}
\begin{document}

\maketitle

\begin{abstract}

Trust region policy optimization (TRPO) is a popular and empirically successful policy search algorithm in Reinforcement Learning (RL) in which a surrogate problem, that restricts consecutive policies to be `close' to one another, is iteratively solved. Nevertheless, TRPO has been considered a heuristic algorithm inspired by Conservative Policy Iteration (CPI). We show that the adaptive scaling mechanism used in TRPO is in fact the natural ``RL version" of traditional trust-region methods from convex analysis. We first analyze TRPO in the planning setting, in which we have access to the model and the entire state space. Then, we consider sample-based TRPO and establish $\tilde O(1/\sqrt{N})$ convergence rate to the global optimum. Importantly, the adaptive scaling mechanism allows us to analyze TRPO in {\em regularized MDPs} for which we prove fast rates of $\tilde O(1/N)$, much like results in convex optimization. This is the first result in RL of better rates when regularizing the instantaneous cost or reward.

\end{abstract}

\section{Introduction}

The field of Reinforcement learning (RL) \cite{sutton2018reinforcement} tackles the problem of learning how to act optimally in an unknown dynamic environment. The agent is allowed to apply actions on the environment, and by doing so, to manipulate its state. Then, based on the rewards or costs it accumulates, the agent learns how to act optimally. The foundations of RL lie in the theory of Markov Decision Processes (MDPs), where an agent has an access to the model of the environment and can plan to act optimally.

\textbf{Trust Region Policy Optimization (TRPO):} Trust region methods are a popular class of techniques to solve an RL problem and span a wide variety of algorithms including Non-Euclidean TRPO (NE-TRPO)~\cite{schulman2015trust} and Proximal Policy Optimization~\cite{schulman2017proximal}. In these methods a sum of two terms is iteratively being minimized: a linearization of the objective function and a proximity term which restricts two consecutive updates to be `close' to each other, as in Mirror Descent (MD)~\cite{beck2003mirror}. In spite of their popularity, much less is understood in terms of their convergence guarantees and they are considered heuristics~\cite{schulman2015trust,papini2019smoothing} (see Figure~\ref{fig: TRPO and CPI}).

\textbf{TRPO and Regularized MDPs:} Trust region methods are often used in conjunction with regularization. This is commonly done by adding the negative entropy to the instantaneous cost \cite{nachum2017trust,schulman2017proximal}. The intuitive justification for using entropy regularization is that it induces inherent exploration~\cite{fox2015taming}, and the advantage of `softening' the Bellman equation \cite{nachum2018path,dai2017sbeed}. Recently, \citet{ahmed2019understanding} empirically observed that adding entropy regularization results in a smoother objective which in turn leads to faster convergence when the learning rate is chosen more aggressively. Yet, to the best of our knowledge, there is no finite-sample analysis that demonstrates faster convergence rates for regularization in MDPs. This comes in stark contrast to well established faster rates for strongly convex objectives w.r.t.~convex ones~\cite{nesterov1998introductory}. In this work we refer to regularized MDPs as describing a more general case in which a strongly convex function is added to the immediate cost.

The goal of this work is to bridge the gap between the practicality of trust region methods in RL and the scarce theoretical guarantees for standard (unregularized) and regularized MDPs. To this end, we revise a fundamental question in this context:
\begin{itemize}
\item[] \textbf{\emph{What is the proper form of the proximity term 
in trust region methods for RL?}}
\end{itemize}
In~\citet{schulman2015trust}, two proximity terms are suggested which result in two possible versions of trust region methods for RL. The first (\citealp{schulman2015trust}, Algorithm~1) is motivated by Conservative Policy Iteration (CPI)~\cite{kakade2003sample} and results in an improving and thus converging algorithm in its exact error-free version. Yet, it seems computationally infeasible to produce a sample-based version of this algorithm. The second algorithm, with an adaptive proximity term which depends on the current policy (\citealp{schulman2015trust}, Equation~12), is described as a heuristic approximation of the first, with no convergence guarantees, but leads to NE-TRPO, currently among the most popular algorithms in RL (see Figure~\ref{fig: TRPO and CPI}).

In this work, we focus on tabular discounted MDPs and study a general TRPO method which uses the latter adaptive proximity term. Unlike the common belief, we show this adaptive scaling mechanism is `natural' and imposes the structure of RL onto traditional trust region methods from convex analysis. We refer to this method as adaptive TRPO, and analyze two of its instances: NE-TRPO~(\citealp{schulman2015trust}, Equation~12) and Projected Policy Gradient (PPG), as illustrated in Figure~\ref{fig: TRPO and CPI}. In Section~\ref{sec: MD}, we review results from convex analysis that will be used in our analysis. Then, we start by deriving in Section~\ref{sec: directional derivative} a closed form solution of the linearized objective functions for RL. In Section~\ref{sec: uniform TRPO}, using the closed form of the linearized objective, we formulate and analyze Uniform TRPO. This method assumes simultaneous access to the state space and that a model is given. In Section~\ref{sec: exact approximate}, we relax these assumptions and study Sample-Based TRPO, a sample-based version of Uniform TRPO, while building on the analysis of Section~\ref{sec: uniform TRPO}. The main contributions of this paper are:
\begin{itemize}
    \item We establish $\tilde O(1/\sqrt{N})$ convergence rate to the global optimum for both Uniform and Sample-Based TRPO.
    \item We prove a faster rate of $\tilde O(1/N)$ for regularized MDPs. To the best of our knowledge, it is the first evidence for faster convergence rates using regularization in RL.
    \item The analysis of Sample-Based TRPO, unlike CPI, does not rely on improvement arguments. This allows us to choose a more aggressive learning rate relatively to CPI which leads to an improved sample complexity even for the unregularized case.
\end{itemize}

\begin{figure}
	\centering
	\begin{tikzpicture}[->,>=stealth',shorten >=1pt,auto,node distance=2.7cm,
		semithick, state/.style={square, draw, minimum size=0.5cm}]
	\tikzstyle{startstop} = [rectangle, rounded corners, minimum width=1cm, minimum height=1cm,text centered, draw=black, font=\small]
	\tikzstyle{sum} = [draw, fill=white, circle, node distance=1cm]
    \tikzstyle{pinstyle} = [pin edge={to-,thin,black}]

    \tikzstyle{arrow} = [thick,->,>=stealth]
    \tikzstyle{bad_arrow} = [thick, dashed,->,>=stealth]

	    \node [sum] (sum) {+};
		
		\node (MD) [startstop, align=center, right of=sum, node distance = 2.2cm] {Mirror Descent \\ (\citeauthor{beck2003mirror})};
		
		\node (Adaptive) [startstop, align=center, left of=sum, node distance = 2cm] {Adaptive Scaling \\ of Proximal term};

		\node (TRPO) [startstop, align=center, below of=sum, node distance=1.4cm] {Adaptive TRPO \\ (this paper)};
		
		\node (NETRPO) [startstop, align=center, below of=TRPO, right of=TRPO, node distance=1.3cm,xshift=0.2cm] {NE-TRPO \\ (\citeauthor{schulman2015trust})};
		
		\node (PPG) [startstop, align=center, below of=TRPO, left of=TRPO, node distance=1.2cm, xshift=-0.2cm] {Projected \\ Policy Gradient};
		
		\node (CPI) [startstop, align=center, right of=TRPO, xshift=0.5cm] {CPI \\ (\citeauthor{kakade2002approximately})};
		
        \node [coordinate, below of=TRPO, node distance=0.7cm] (TRPO_bottom) {};
        
        \node [coordinate, below of=CPI, node distance=1.5cm] (CPI_bottom) {};
        
        % \node [coordinate, below of=CPI, node distance=0.5cm, xshift=0cm] (CPI_out) {};

        \draw [arrow] (Adaptive) -- (sum);
        \draw [arrow] (MD) -- (sum);
        \draw [arrow] (sum) -- (TRPO);
        \draw [arrow] (TRPO) -- (TRPO_bottom) --  (NETRPO);
        \draw [arrow] (TRPO) -- (TRPO_bottom) --  (PPG);
        \draw [bad_arrow] (CPI) -- (CPI_bottom) -- (NETRPO);

	\end{tikzpicture}
	\caption{The adaptive TRPO: a solid line implies a formal relation; a dashed line implies a heuristic relation.}
	\label{fig: TRPO and CPI}
\end{figure}
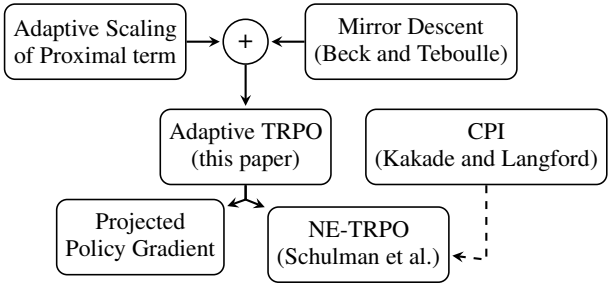

\section{Mirror Descent in Convex Optimization}\label{sec: MD}
Mirror descent (MD) \cite{beck2003mirror} is a well known first-order trust region optimization method for solving constrained convex problems, i.e, for finding
\begin{align}\label{eq: convex optimization}
    x^*\in\argmin_{x \in C} f(x),
\end{align}
where $f$ is a convex function and $C$ is a convex compact set. In each iteration, MD minimizes a linear approximation of the objective function, using the gradient $\nabla f(x_k)$, together with a proximity term by which the updated $x_{k+1}$ is `close' to $x_k$. Thus, it is considered a trust region method, as the iterates are `close' to one another. The iterates of MD are
\begin{align}\label{eq: MD iteration}
x_{k+1}\in \argmin_{x\in C}  \inner*{ \nabla f(x_k), x - x_k } + \frac{1}{t_k}\bregman{x}{x_k},
\end{align}
where $\bregman{x}{x_k}\triangleq\omega(x)-\omega(x_k)-\langle \nabla \omega(x_k), x-x_k\rangle$ is the Bregman distance associated with a strongly convex $\omega$ and $t_k$ is a stepsize (see Appendix~\ref{supp: assumptions omega}). In the general convex case, MD converges to the optimal solution of \eqref{eq: convex optimization} with a rate of $\tilde O(1/\sqrt{N})$, where $N$ is the number of MD iterations \cite{beck2003mirror,juditsky2011first}, i.e., $f(x_k)-f^* \leq  \tilde O(1/\sqrt{k})$, where $f^*=f(x^*)$.

The convergence rate can be further improved when $f$ is a part of special classes of functions. One such class is the set of $\lambda$-strongly convex functions w.r.t. the Bregman distance. We say that $f$ is $\lambda$-strongly convex w.r.t. the Bregman distance if $f(y)\geq f(x) +\langle \nabla f(x),y - x\rangle+\lambda \bregman{y}{x}$. For such $f$, improved convergence rate of $\tilde O(1/N)$ can be obtained \cite{juditsky2011first,nedic2014stochastic}. Thus, instead of using MD to optimize a convex $f$, one can consider the following regularized problem,
\begin{align}\label{eq: regularized convex optimization}
    x^*=\argmin_{x \in C} f(x)+\lambda g(x),
\end{align}
where $g$ is a strongly convex regularizer with coefficient ${\lambda>0}$. Define $F_\lambda(x)\triangleq f(x)+\lambda g(x)$, then, each iteration of MD becomes,
\begin{align}\label{eq: regularized MD iteration}
x_{k+1} = \argmin_{x\in C} \inner*{\nabla F_\lambda(x_k), x - x_k } + \frac{1}{t_k}\bregman{x}{x_k} .
\end{align}
Solving \eqref{eq: regularized MD iteration} allows faster convergence, at the expense of adding a bias to the solution of \eqref{eq: convex optimization}. Trivially, by setting $\lambda=0$, we go back to the unregularized convex case. 

In the following, we consider two common choices of $\omega$ which induce a proper Bregman distance: (a) \textbf{The euclidean case}, with $\omega(\cdot)=\frac{1}{2}\norm{\cdot}_2^2$ and the resulting Bregman distance is the squared euclidean norm $\bregman{x}{y} = \frac{1}{2}\norm{x-y}_2^2$. In this case, \eqref{eq: MD iteration} becomes the \emph{Projected Gradient Descent} algorithm (\citealp{beck2017first}, Section 9.1), where in each iteration, the update step goes along the direction of the gradient at $x_k$, $\nabla f (x_k)$, and then projected back to the convex set $C$,
$
    x_{k+1} = P_c\left(x_k - t_k \nabla f(x_k)\right),
$
where $P_c(x)=\min_{y\in C} \frac{1}{2}\norm{x-y}_2^2$ is the orthogonal projection operator w.r.t. the euclidean norm. 

(b) \textbf{The non-euclidean case}, where $\omega(\cdot)=H(\cdot)$ is the negative entropy, and the Bregman distance then becomes the Kullback-Leibler divergence, $\bregman{x}{y}=\dkl{x}{y}$. In this case, MD becomes the \emph{Exponentiated Gradient Descent} algorithm. Unlike the euclidean case, where we need to project back into the set, when choosing $\omega$ as the negative entropy, \eqref{eq: MD iteration} has a closed form solution (\citealp{beck2017first}, Example 3.71),
$
    x^i_{k+1} = \frac{x^i_k e^{-t_k \nabla_i f(x_k)}}{\sum_j x^j_k e^{-t_k \nabla_j f(x_k)}},
$
where $x^i_k$ and $\nabla_i f$ are the $i$-th coordinates of $x_k$ and $\nabla f$. 

% \begin{remark}[Mirror Descent and Trust Region methods]
% MD is a trust region metho. We chose to keep the name TRPO as we believe the trust region perspective
% \end{remark}

\section{Preliminaries and Notations}\label{sec: notation and pre}

We consider the infinite-horizon discounted MDP which is defined as the 5-tuple $(\mathcal{S}, \mathcal{A},P,C,\gamma)$ \cite{sutton2018reinforcement}, where ${\mathcal S}$ and ${\mathcal A}$ are finite state and action sets with cardinality of $S=|\mathcal{S}|$ and $A=|\mathcal{A}|$, respectively. The transition kernel is $P \equiv P(s'|s,a)$, $C \equiv c(s,a)$ is a cost function bounded in $[0,\Cmax]$\footnote{We work with costs instead of rewards to comply with convex analysis. All results are valid to the case where a reward is used.}, and $\gamma\in(0,1)$ is a discount factor. Let $\pi: \mathcal{S}\rightarrow \simplex$ be a stationary policy, where $\simplex$ is the set probability distributions on $\mathcal{A}$. Let $v^\pi \in \mathbb{R}^{S}$ be the value of a policy $\pi,$ with its $s\in \mathcal{S}$ entry given by ${v^\pi(s) \triangleq \E^{\pi}[\sum_{t=0}^\infty\gamma^tr(s_t,\pi(s_t))\mid s_0=s]}$, and ${\E^\pi[\cdot \mid s_0=s]}$ denotes expectation w.r.t. the distribution induced by $\pi$ and conditioned on the event $\{s_0=s\}.$  
%For brevity, we respectively denote the cost and value at time $t$ by $c_t\equiv c(s_t,\pi_t(s_t))$ and $v_t\equiv v(s_t).$  
It is known that $v^\pi=\sum_{t=0}^\infty \gamma^t (P^\pi)^t c^\pi=(I-\gamma P^\pi)^{-1}c^\pi,$ with the component-wise values $[P^\pi]_{s,s'}  \triangleq P(s'\mid s, \pi(s))$ and $[c^\pi]_s \triangleq  c(s,\pi(s))$. Our goal is to find a policy $\pi^*$ yielding the optimal value $v^*$ such that
\begin{equation}
v^* = \min_\pi (I-\gamma P^\pi)^{-1} c^\pi = (I-\gamma P^{\pi^*})^{-1} c^{\pi^*}.\label{mdp}
%= \lim_{j \to \infty} T^j 0/ \label{mdp}
\end{equation}
This goal can be achieved using the classical operators:
\begin{align}
\forall v,\pi,~  T^\pi v  =  c^\pi +\gamma P^\pi v, \mathrm{and} \ \forall v,~  T v  =  \min_\pi T^\pi v,
%\forall v,\pi,~ \pi \in \G(v) & \Leftrightarrow  T^\pi v = T v.
% \forall v,~\G(v)&= \{\pi : T^\pi v = T v\}, \label{def: greedy policy}
\end{align}
where $T^\pi$ is a linear operator, $T$ is the optimal Bellman operator and both $T^\pi$ and $T$ are $\gamma$-contraction mappings w.r.t. the max-norm. The fixed points of $T^\pi$ and $T$ are $v^\pi$ and $v^*$. 
%The set $\G(v)$ is the standard set of 1-step greedy policies w.r.t. $v$.  Furthermore, given $v^*$, the set $\G(v^*)$ coincides with that of stationary optimal policies. In other words, every policy that is 1-step greedy w.r.t. $v^*$ is optimal and vice versa. 
 
A large portion of this paper is devoted to analysis of regularized MDPs: A regularized MDP is an MDP with a shaped cost denoted by $c_\lambda^\pi$ for $\lambda \geq 0$. Specifically, the cost of a policy $\pi$ on a regularized MDP translates to ${c^{\pi}_\lambda(s) \triangleq c^{\pi}(s)+\lambda\omegaS{s}{\pi}}$ where $\omegaS{s}{\pi} \triangleq \omega(\pi(\cdot \mid s))$ and $\omega: \simplex \rightarrow \mathbb{R}$ is a $1$-strongly convex function. We denote $\omega(\pi)\in \mathbb{R}^S$ as the corresponding state-wise vector. See that for $\lambda=0$, the cost $c^\pi$ is recovered. In this work we consider two choices of $\omega$: the {\bf euclidean case} $\omegaS{s}{\pi}=\frac{1}{2}\norm{\pi(\cdot\mid s)}_2^2$ and {\bf non-euclidean case} $\omegaS{s}{\pi}= H(\pi(\cdot\mid s))+\log A$. By this choice we have that $0 \leq c^\pi_\lambda(s)\leq  \CmaxLambda$ where $\CmaxLambda=\Cmax+\lambda$ and $\CmaxLambda=\Cmax+\lambda \log A$, for the euclidean and non-euclidean cases, respectively. With some abuse of notation we omit $\omega$ from~$\CmaxLambda$.

The value of a stationary policy $\pi$ on the regularized MDP is $v_\lambda^\pi=(I-\gamma P^{\pi})^{-1} c_\lambda^{\pi}$. Furthermore, the optimal value $\vHstar$, optimal policy $\pi_\lambda^*$ and Bellman operators of the regularized MDP are generalized as follows,
\begin{align}
    &\vHstar = \min_\pi (I-\gamma P^\pi)^{-1} c_\lambda^\pi = (I-\gamma P^{\pi_\lambda^*})^{-1} c_\lambda^{\pi_\lambda^*}, \label{eq: ergularized mdp}\\
    &\forall v,\pi,~  T_\lambda^\pi v  =  c_\lambda^\pi +\gamma P^\pi v, \mathrm{and} \ \forall v,~  T_\lambda  v  =  \min_\pi T_\lambda^\pi  v.\nonumber
\end{align}
As Bellman operators for MDPs, both $T_{\lambda}^\pi,T$ are $\gamma$-contractions with fixed points $v_\lambda^\pi,v^*_\lambda$ \cite{geist2019theory}. Denoting $c^\pi_\lambda(s,a) = c(s,a)+\lambda \omegaS{s}{\pi}$, the $q$-function of a policy $\pi$ for a regularized MDP is defined as $q_\lambda^\pi(s,a) = c^\pi_\lambda(s,a)+\gamma \sum\nolimits_{s'}p^\pi(s'\mid s)v_\lambda^\pi(s').$

% and thus,
% $ \quad  T_\lambda^{\pi'} v_\lambda^{\pi}(s) = \sum_a \pi'(a|s) q_\lambda^{\pi}(s,a) + \lambda \omegaS{s}{\pi'}-\lambda\omegaS{s}{\pi}$.

When the state space is small and the dynamics of environment is known~\eqref{mdp}, \eqref{eq: ergularized mdp} can be solved using DP approaches. However, in case of a large state space it is expected to be computationally infeasible to apply such algorithms as they require access to the entire state space. In this work, we construct a sample-based algorithm which minimizes the following \emph{scalar objective} instead of~\eqref{mdp}, \eqref{eq: ergularized mdp}, 
\begin{align}
    \min_{\pi\in \ssimplex} \E_{s \sim \mu} [v^\pi_\lambda (s)]= \min_{\pi\in \ssimplex}\mu v^\pi_\lambda, \label{eq: scalar objective}
\end{align}
where $\mu(\cdot)$ is a probability measure over the state space. Using this objective, one wishes to find a policy $\pi$ which minimizes the expectation of $v^{\pi}_\lambda(s)$ under a measure $\mu$. This objective is widely used in the RL literature \cite{sutton2000policy,kakade2002approximately,schulman2015trust}.

Here, we always choose the regularization function $\omega$ to be associated with the Bregman distance used, $B_\omega$. This simplifies the analysis as $c_\lambda^\pi$ is $\lambda$-strongly convex w.r.t. $B_\omega$ by definition. Given two policies $\pi_1,\pi_2$, we denote their Bregman distance as $\bregmanS{s}{\pi_1}{\pi_2}\triangleq \bregman{\pi_1(\cdot\mid s)}{\pi_2(\cdot\mid s)}$ and $\bregman{\pi_1}{\pi_2}\! \in \! \mathbb{R}^S$ is the corresponding state-wise vector. The euclidean choice for $\omega$ leads to $\bregmanS{s}{\pi_1}{\pi_2}\!=\!\frac{1}{2}\norm{\pi_1(\cdot\mid s) - \pi_2(\cdot\mid s)}_2^2$, and the non-euclidean choice to $\bregmanS{s}{\pi_1}{\pi_2}\!=\!\dkl{\pi_1(\cdot \! \mid \! s)}{\pi_2(\cdot \! \mid \! s)}$. In the results we use the following $\omega$-dependent constant, $\euclid =\sqrt{A}$ in the euclidean case, and $\euclid =1$ in the non-euclidean case.

For brevity, we omit constant and logarithmic factors when using $O(\cdot)$, and omit any factors other than non-logarithmic factors in $N$, when using $\tilde O(\cdot)$. For $x,y\in \R^{S \times A}$, the state-action inner product is $\inner*{x,y}=\sum_{s,a}x(s,a)y(s,a)$, and the fixed-state inner product is $\inner*{x(s,\cdot),y(s,\cdot)}=\sum_{a} x(s,a)y(s,a)$. 
Lastly, when $x\in\mathbb{R}^{S\times S\times A}$ (e.g., first claim of Proposition~\ref{prop: directional derivative for MDPs}) the inner product $\inner{x,y}$ is a vector in $\mathbb{R}^S$ where $\inner{x,y}(s) \triangleq \inner{x(s,\cdot,\cdot),y},$ with some abuse of notation. 

% \section{Directional Derivatives in MDPs}\label{sec: directional derivative}
\section{Linear Approximation of a Policy's Value}\label{sec: directional derivative}

As evident from the updating rule of MD~\eqref{eq: MD iteration}, a crucial step in adapting MD to solve MDPs is studying the linear approximation of the objective, $\inner{\nabla f(x),x'-x}$, i.e., the directional derivative in the direction of an element from the convex set. The objectives considered in this work are~\eqref{eq: ergularized mdp}, \eqref{eq: scalar objective}, and the optimization set is the convex set of policies $\ssimplex$. Thus, we study $\left\langle \nabla {\vHpi},\pi '-\pi  \right\rangle$ and $\left\langle \nabla \mu{\vHpi},\pi '-\pi  \right\rangle$, for which the following proposition gives a closed form:

\begin{restatable}[Linear Approximation of a Policy's Value]{proposition}{propDirectionalDerivative}
\label{prop: directional derivative for MDPs}
Let $\pi,\pi' \in \Delta^S_A$, and $d_{\mu,\pi} = (1-\gamma) \mu(I-\gamma P^\pi)^{-1}$. Then, 
\begin{align}
        &\!\inner*{ \nabla_\pi {\vHpi},\pi '\!-\!\pi  }\! = \!(I\!-\!\gamma P^\pi)^{-1} \!\left( T_{\lambda}^{\pi'}v^{\pi}_\lambda\!-\!\vHpi \!-\! \lambda \bregman{\pi'}{\pi}\right), \label{eq: Regularized value directional derivative first}\\
        &\! \inner*{ \nabla_\pi \mu{\vHpi},\pi '\!-\!\pi  } 
        \!=\! \frac{1}{1\!-\!\gamma}d_{\mu,\pi} \!\left( T_{\lambda}^{\pi'}v^{\pi}_\lambda\!-\!\vHpi \!-\!\lambda \bregman{\pi'}{\pi}\right). \label{eq: Regularized value directional derivative second}
\end{align}
\end{restatable}

The proof, supplied in Appendix~\ref{supp: policy gradients and directional derivatives}, is a direct application of a Policy Gradient Theorem~\cite{sutton2000policy} derived for regularized MDPs. Importantly, the linear approximation is scaled by $(I-\gamma P^\pi)^{-1}$ or $\frac{1}{1-\gamma}d_{\mu,\pi}$, the discounted visitation frequency induced by the current policy. In what follows, we use this understanding to properly choose an \emph{adaptive scaling} for the proximity term of TRPO, which allows us to use methods from convex optimization.

\section{Uniform Trust Region Policy Optimization}\label{sec: uniform TRPO}

In this section we formulate \emph{Uniform TRPO}, a trust region \emph{planning} algorithm with an adaptive proximity term by which \eqref{eq: ergularized mdp} can be solved, i.e., an optimal policy which jointly minimizes the vector $v^\pi_\lambda$ is acquired. We show that the presence of the adaptive term simplifies the update rule of Uniform TRPO and then analyze its performance for both the regularized ($\lambda>0$) and  unregularized ($\lambda=0$) cases. Despite the fact~\eqref{eq: ergularized mdp} is not a convex optimization problem, the presence of the adaptive term  allows us to use techniques applied for MD in convex analysis and establish convergence to the global optimum with rates of $\tilde O(1/\sqrt{N})$ and $\tilde O(1/N)$ for the  unregularized and regularized case, respectively.

\begin{algorithm}
\caption{Uniform TRPO}\label{alg: uniform TRPO}
\begin{algorithmic}
\REQUIRE $t_k$, $\gamma$, $\lambda$, $\pi_0$ is the uniform policy.
\FOR{$k=0,1,...$}
    \STATE $v^{\pi_k} = (I-\gamma P^{\pi_k})^{-1} c_\lambda^{\pi_k}$
    \FOR{ $\forall s\in\sset$}
        \FOR{$\forall a\in\aset$}
            \STATE $q^{\pi_k}_\lambda(s,a) \gets c^{\pi_k}_\lambda(s,a) + \gamma \sum_{s'} p(s'|s,a) v_\lambda^{\pi_k}(s')$ 
        \ENDFOR
    \STATE $\pi_{k+1}(\cdot|s) \gets $       PolicyUpdate($\pi(\cdot|s),q^{\pi_k}_\lambda(s,\cdot), t_k, \lambda$)
    \ENDFOR
\ENDFOR
% \STATE {\bf Return:} $\pi_k$
\end{algorithmic}
\end{algorithm}

Uniform TRPO repeats the following iterates
\begin{align}\label{eq: PG iteration paper}
&\pi_{k+1}\in \nonumber \argmin_{\pi\in \ssimplex} \Big\{\inner*{ \nabla \vHpik, \pi - \pi_k } \\
&\quad\quad\quad\quad\quad \quad \quad+  \frac{1}{t_k}(I-\gamma P^{\pi_k})^{-1}\bregman{\pi}{\pi_k} \Big\}.
\end{align}
The update rule resembles MD's updating-rule~\eqref{eq: MD iteration}. The updated policy minimizes the linear approximation while being not `too-far' from the current policy due to the presence of $\bregman{\pi}{\pi_k}$. However, and unlike MD's update rule, the Bregman distance is scaled by the adaptive term $(I-\gamma P^{\pi_k})^{-1}$. Applying Proposition~\ref{prop: directional derivative for MDPs}, we see why this adaptive term is so natural for RL,
\begin{align}
\label{eq: PG iteration paper directional derivative}
&\pi_{k+1}  \in \nonumber\\[-17pt]
&\!\argmin_{\pi\in \ssimplex} (I \! -\! \gamma P^{\pi_k})^{-1}\!\! \overbrace{\Big( T_{\lambda}^{\pi} \vHpik \! - \! \vHpik \!  +\! \Big( \frac{1}{t_k}\!-\!\lambda \Big) \bregmanShort{\pi}{\pi_k} \Big) }^{(*)}\!.
\end{align}
% Importantly, the presence of the adaptive term results in a simple update rule, in which the update of each state depends only on $T_\lambda^{\pi}\vHpik(s) $ and $\bregmanS{s}{\pi'}{\pi}$ and does not `mix' values and Bregman distances of other states. Thus, Uniform TRPO update rule has a simplified form (Appendix~\ref{supp: uniform TRPO}), in which the policy is updated for all $s\in\mathcal{S}$ by,
Since $(I \! -\! \gamma P^{\pi_k})^{-1}\geq 0$ component-wise, minimizing \eqref{eq: PG iteration paper directional derivative} is equivalent to minimizing the vector $(*)$. This results in a simplified update rule: instead of minimizing over $\ssimplex$ we minimize over $\simplex$ for each $s\in\mathcal{S}$ independently (see Appendix~\ref{supp: uniform TRPO}). For each $s\in \mathcal{S}$ the policy is updated by
\begin{align}
      &{\pi}_{k+1}(\cdot \! \mid \! s)\! \in \! \argmin_{\pi \in \simplex}  t_k T_\lambda^{\pi}\vHpik(s) +  (1\!-\!\lambda t_k)\bregmanS{s}{\pi}{\pi_k}.\label{eq: uniform TRPO simplied update}
\end{align}
This is the update rule of Algorithm~\ref{alg: uniform TRPO}.  Importantly, the update rule is a direct consequence of choosing the adaptive scaling for the Bregman distance in~\eqref{eq: PG iteration paper}, and without it, the trust region problem would involve optimizing over~$\ssimplex$.

\begin{algorithm}
    \caption{PolicyUpdate: PPG}
    \label{alg: Policy Gradient policy update exact}
	\begin{algorithmic}
	\STATE {\bf input: }$\pi(\cdot\mid s),q(s,\cdot), t_k,\lambda$
    \FOR{$a \in \aset$}
        \STATE $\pi(a|s) \gets \pi(a\mid s) - \frac{t_k}{1-\lambda t_k} q(s,a)$
    \ENDFOR
    \STATE $\pi(\cdot |s) \gets P_{\simplex} (\pi(\cdot\mid s))$
    \RETURN $\pi(\cdot \mid s)$    
    \end{algorithmic}
  \end{algorithm}
  
 \begin{algorithm}
    \caption{PolicyUpdate: NE-TRPO}
    \label{alg: NE-TRPO policy update exact}
	\begin{algorithmic}
	\STATE {\bf input: }$\pi(\cdot\mid s),q(s,\cdot),t_k,\lambda$
    \FOR{$a \in \aset$}
        \STATE $\pi(a\mid s) \gets \frac{ \pi( a\mid s) \exp\br*{-t_k \br*{q(s,a) + \lambda \log \pi_k (a \mid s) }}}{\sum_{a' \in \aset}{\pi ( a' \mid s ) \exp\br*{-t_k \br*{q(s,a')+\lambda \log \pi_k (a' \mid s)}}}}$
    \ENDFOR
    \RETURN $\pi(\cdot \mid s)$   
    \end{algorithmic}
  \end{algorithm}

% Furthermore, to specify Uniform TRPO, $\omega$ and the induced Bregman distance $B_\omega$ should be chosen. This choice specifies the procedure \emph{PolicyUpdate} of Algorithm~\ref{alg: uniform TRPO}. 
By instantiating the \emph{PolicyUpdate} procedure  with Algorithms~\ref{alg: Policy Gradient policy update exact} and~\ref{alg: NE-TRPO policy update exact} we get the PPG and NE-TRPO, respectively, which are instances of Uniform TRPO.  Instantiating PolicyUpdate is equivalent to choosing $\omega$ and the induced Bregman distance $B_\omega$. In the euclidean case, $\omega(\cdot)= \frac{1}{2}\norm{\cdot}_2^2$  (Alg.~\ref{alg: Policy Gradient policy update exact}), and in the non-euclidean case, $\omega(\cdot)= H(\cdot)$ (Alg.~\ref{alg: NE-TRPO policy update exact}). This comes in complete analogy to the fact Projected Gradient Descent and Exponentiated Gradient Descent are instances of MD with similar choices of $\omega$ (Section~\ref{sec: MD}). 

With the analogy to MD~\eqref{eq: MD iteration} in mind, one would expect Uniform TRPO, to converge with rates $\tilde O(1/\sqrt{N})$ and $\tilde O(1/N)$ for the unregularized and regularized cases, respectively, similarly to MD. Indeed, the following theorem formalizes this intuition for a proper choice of learning rate. The proof of Theorem \ref{theorem: uniform TRPO} extends the techniques of \citet{beck2017first} from convex analysis to the non-convex optimization problem \eqref{mdp}, by relying on the adaptive scaling of the Bregman distance in~\eqref{eq: PG iteration paper} (see Appendix \ref{supp: uniform TRPO proof}). 

\begin{theorem}[Convergence Rate: Uniform TRPO]\label{theorem: uniform TRPO}  
Let $\{\pik\}_{k\geq 0}$ be the sequence generated by Uniform TRPO. Then, the following holds for all $N\geq1$:
\begin{enumerate}
    \item \emph{(Unregularized)} Let $\lambda=0$, $t_k=\frac{(1-\gamma)}{\euclid {C_{\text{max}}} \sqrt{k+1}}$, then
    \begin{align*}
       \norm{v^{\pi_N} - \vstar}_\infty   \leq 
        O\br*{\frac{\euclid \mathrm{C_{\text{max}}}}{(1-\gamma)^2\sqrt{N}} }  . \end{align*}
    \item \emph{(Regularized)} Let $\lambda>0$, $t_k =  \frac{1}{\lambda (k+2)}$, then
% \begin{align*}
%     \norm{v_{\textbf{best}}^N - \vstar_\lambda}_\infty \leq O\left(\frac{1}{\lambda(1-\gamma)^3}\frac{\lambda^2\omegaBound (1-\gamma)^2 +\euclid ^2\CmaxLambdaSquare\log {N}}{N}\right).
% \end{align*}
\begin{align*}
   \norm{v_\lambda^{\pi_N} - \vHstar}_\infty   \leq  O \br*{ \frac{\euclid ^2\CmaxLambda^2}{\lambda(1
    -\gamma)^3 N} }.
\end{align*}
\end{enumerate}
% Where $\euclid =\sqrt{A}$ for the euclidean case, and $\euclid =1$ for the non-euclidean case.
\end{theorem}

Theorem \ref{theorem: uniform TRPO} establishes that regularization allows faster convergence of $\tilde O(1/ N$). It is important to note using such regularization leads to a `biased' solution: Generally $\norm{v^{\pi^*_\lambda} - \vstar}_\infty>0$, where we denote $\pi^*_\lambda$ as the optimal policy of the regularized MDP. In other words, the optimal policy of the regularized MDP evaluated on the  unregularized MDP is not necessarily the optimal one. However, when adding such regularization to the problem, it becomes easier to solve, in the sense Uniform TRPO converges faster (for a proper choice of learning rate). 

In the next section, we extend the analysis of Uniform TRPO to Sample-Based TRPO, and relax the assumption of having access to the entire state space in each iteration, while still securing similar convergence rates in $N$. 

\section{Exact and Sample-Based TRPO}\label{sec: exact approximate}

In the previous section we analyzed Uniform TRPO, which uniformly minimizes the vector $v^{\pi}$. Practically, in large-scale problems, such an objective is infeasible as one cannot access the entire state space, and less ambitious goal is usually defined \cite{sutton2000policy,kakade2002approximately,schulman2015trust}. The objective usually minimized is the \emph{scalar objective}~\eqref{eq: scalar objective}, the expectation of $v^{\pi}_\lambda(s)$ under a measure $\mu$,
% \begin{align*}
    $\min_{\pi\in \ssimplex} \E_{s \sim \mu} [v^\pi_\lambda (s)]= \min_{\pi\in \ssimplex}\mu v^\pi_\lambda.$ 
% \end{align*}

Starting from the seminal work on CPI, it is common to assume access to the environment in the form of a \emph{$\nu$-restart model}. Using a $\nu$-restart model, the algorithm interacts with an MDP in an episodic manner. In each episode $k$, the starting state is sampled from the initial distribution $s_0\sim \nu$, and the algorithm samples a trajectory $(s_0,r_0,s_1,r_1,...)$ by following a policy $\pi_k$. As mentioned in \citet{kakade2003sample}, a $\nu$-restart model is a weaker assumption than an access to the true model or a generative model, and a stronger assumption than the case where no restarts are allowed. 

To establish global convergence guarantees for CPI, \citet{kakade2002approximately} have made the following assumption, which we also assume through the rest of this section: 
\begin{restatable}[Finite Concentrability Coefficient]{assumption}{assumptionStochasticMD}
    \label{assumption: stochastic MD}
 ${C^{\pi^*}\triangleq \norm{\frac{d_{\mu,\pi^*}}{\nu}}_\infty = \max_{s\in \mathcal{S}}\left|\frac{d_{\mu,\pi^*}(s)}{\nu(s)}\right| < \infty.}$
\end{restatable}

The term $C^{\pi^*}$ is known as a concentrability coefficient and appears often in the analysis of policy search algorithms~\cite{kakade2002approximately,scherrer2014local,bhandari2019global}. Interestingly, $C^{\pi^*}$ is considered the `best' one among all other existing concentrability coefficients in approximate Policy Iteration schemes~\cite{scherrer2014approximate}, in the sense it can be finite when the rest of them are infinite.

% In what follows, we assume Assumption~\ref{assumption: stochastic MD} holds, by which we will be able to supply global convergence guarantees for a sample-based trust region method to solve~\eqref{eq: scalar objective}.

\subsection{Warm Up: Exact TRPO}\label{sec: exact MD}

%Assuming such an access to the model usually mitigates the exploration problem; repeatedly sampling from $\mu$ allows for sufficient exploration and alleviates the exploration-exploitation dilemma. This assumption allows us to analyze the algorithm from a pure optimization point of view, without having to deal with the exploration problem \cite{kakade2003sample}.

We split the discussion on the sample-based version of TRPO: we first discuss \emph{Exact TRPO} which minimizes the scalar $\mu v^\pi_\lambda$~\eqref{eq: scalar objective} instead of minimizing the vector $v^\pi_\lambda$~\eqref{eq: ergularized mdp} as Uniform TRPO, while having an exact access to the gradients. Importantly, its updating rule is \textbf{the same update rule used in NE-TRPO}~(\citealp{schulman2015trust}, Equation~12), which uses the adaptive proximity term, and is described there as a heuristic. Specifically, there are two minor discrepancies between NE-TRPO and Exact TRPO: 1) We use a penalty formulation instead of a constrained optimization problem. 2) The policies in the Kullback-Leibler divergence are reversed.
Exact TRPO is a straightforward adaptation of Uniform TRPO to solve \eqref{eq: scalar objective} instead of~\eqref{eq: ergularized mdp} as we establish in Proposition~\ref{proposition: uniform to exact update}. Then, in the next section, we extend Exact TRPO to a sample-based version with provable guarantees. 

With the goal of minimizing the objective $\mu v^\pi_\lambda$, Exact TRPO repeats the following iterates
\begin{align}
\pi_{k+1} \in \argmin& _{\pi\in \ssimplex} \big\{ \langle \nabla \nu\vHpik, \pi - \pi_k \rangle \nonumber \\
& \quad \enskip+   \frac{1}{t_k(1-\gamma)} d_{\nu,\pi_k} \bregman{\pi}{\pi_k} \big\}, \label{eq: exact scalar TRPO}
\end{align}

Its update rule resembles MD's update rule~\eqref{eq: PG iteration paper}, but uses the $\nu$-restart distribution for the linearized term. Unlike in MD~\eqref{eq: MD iteration}, the Bregman distance is scaled by an adaptive scaling factor $d_{\nu,\pi_k}$, using $\nu$ and the policy $\pi_k$ by which the algorithm interacts with the MDP. This update rule is motivated by the one of Uniform TRPO analyzed in previous section~\eqref{eq: PG iteration paper} as the following straightforward proposition suggests (Appendix~\ref{supp: exact TRPO update rule}):

\begin{restatable}[Uniform to Exact Updates]{proposition}{UniformToExactUpdate}\label{proposition: uniform to exact update}

For any ${\pi,\pi_k\in \ssimplex}$
\begin{align*}
    &\nu\Big(\langle \nabla \vHpik, \pi - \pi_k \rangle + \frac{1}{t_k}(I-\gamma P^{\pi_k})^{-1} \bregman{\pi}{\pi_k}\Big)\\
    &=\langle \nabla \nu\vHpik, \pi - \pi_k \rangle + \frac{1}{t_k(1-\gamma)} d_{\nu,\pi_k} \bregman{\pi}{\pi_k}.
\end{align*}
\end{restatable}

Meaning, the proximal objective solved in each iteration of Exact TRPO~\eqref{eq: exact scalar TRPO} is the expectation w.r.t. the measure $\nu$ of the objective solved in Uniform TRPO~\eqref{eq: PG iteration paper}.

Similarly to the simplified update rule for Uniform TRPO~\eqref{eq: PG iteration paper directional derivative}, by using the linear approximation in Proposition~\ref{prop: directional derivative for MDPs}, it can be easily shown that using the adaptive proximity term allows to obtain a simpler update rule for Exact TRPO. Unlike Uniform TRPO which updates all states, Exact TRPO updates only states for which $d_{\nu,\pi_k}(s)>0$. Denote $\mathcal{S}_{d_{\nu,\pi_k}} =\brc*{s: d_{\nu,\pi_k}(s)>0}$ as the set of these states. Then, Exact TRPO is equivalent to the following update rule (see Appendix~\ref{supp: exact TRPO update rule}), $\forall s\in \mathcal{S}_{d_{\nu,\pi_k}}$:
\begin{align*}
      & {\pi}_{k+1}(\cdot\mid s) \!\in\!  \argmin_\pi  t_k T_\lambda^{\pi}\vHpik(s) \!+ \! (1 \!-\! \lambda t_k)\bregmanS{s}{\pi}{\pi_k},
\end{align*}
i.e., it has the same updates as Uniform TRPO, but updates only states in $\mathcal{S}_{d_{\nu,\pi_k}}$. Exact TRPO converges with similar rates for both the regularized and unregularized cases, as Uniform TRPO. These are formally stated in Appendix~\ref{supp: mu exact policy mirror descenet}.

\subsection{Sample-Based TRPO}\label{sec: approximate MD}

In this section we derive and analyze the \emph{sample-based} version of Exact TRPO, and establish high-probability convergence guarantees in a batch setting. Similarly to the previous section, we are interested in minimizing the scalar objective $\mu v_\lambda^\pi$~\eqref{eq: scalar objective}. Differently from Exact TRPO which requires an access to a model and to simultaneous updates in all states in $\mathcal{S}_{d_{\nu,\pi_k}}$, \emph{Sample-Based TRPO} assumes access to a $\nu$-restart model. Meaning, it can only access sampled trajectories and restarts according to the distribution $\nu$. 
 
\begin{algorithm}
\caption{Sample-Based TRPO}\label{alg: mu Sample-Based TRPO}
\begin{algorithmic}
\REQUIRE $t_k$, $\gamma$, $\lambda$, $\pi_0$ is the uniform policy, $\epsilon,\delta>0$
\FOR{$k=0,1,...$}
    % \STATE {\color{gray}\# Initialization per episode}
    \STATE $\mathcal{S}^k_M=\{\},\forall s,a,\ \hat{q}_\lambda^{\pi_k}(s,a)=0, n_k(s,a)=0$
    \STATE $M_k \geq \tilde O(\frac{A^2 \CmaxLambdaSquare\br*{S\log 2A + \log 1/\delta}}{(1-\gamma)^2\epsilon^2})$ {\color{gray}\# Appendix~\ref{supp: Approximation Error bounds}}
    \STATE {\color{gray}\# Sample Trajectories}
    \FOR{$m=1,..,M_k$}
        \STATE Sample $s_m\sim d_{\nu,\pi_k}(\cdot)$, $a_m\sim U(\mathcal{A})$
        \STATE $\hat{q}_\lambda^{\pi_k}(s_m,a_m, m )$= Truncated rollout of $q_\lambda^{\pi_k}(s_m,a_m)$
        \STATE $\hat{q}_\lambda^{\pi_k}(s_m,a_m)\gets \hat{q}_\lambda^{\pi_k}(s_m,a_m) + \hat{q}_\lambda^{\pi_k}(s_m,a_m,m )$
        \STATE $n_k(s_m,a_m) \gets n_k(s_m,a_m) + 1 $
        \STATE $\mathcal{S}^k_M = \mathcal{S}^k_M  \cup \{s_m\}$
    \ENDFOR
    \STATE {\color{gray}\# Update Next Policy}
    \FOR{$\forall s\in \mathcal{S}^k_M$}
        \FOR{$\forall a\in \mathcal{A}$}
           \STATE $\hat{q}^{\pi_k}_\lambda(s,a) \gets A\hat{q}^{\pi_k}_\lambda(s,a)/ \br*{\sum_a n_k(s,a) } $
        \ENDFOR
        \STATE $\pi_{k+1}(\cdot|s) = $PolicyUpdate($\pi_{k}(\cdot|s),\hat{q}^{\pi_k}_\lambda(s,\cdot), t_k,\lambda$)
    \ENDFOR
\ENDFOR
% \STATE {\bf Return:} $\pi_k$
\end{algorithmic}
\end{algorithm}

Sample-Based TRPO samples $M_k$ trajectories per episode. In every trajectory of the $k$-th episode, it first samples $s_m\sim d_{\nu,\pi_k}$ and takes an action $a_m\sim U(\mathcal{A})$ where $U(\mathcal{A})$ is the uniform distribution on the set $\mathcal{A}$. Then, by following the current policy $\pi_k$, it estimates $\qHpik(s_m,a_m)$ using a rollout (possibly truncated in the infinite horizon case). We denote this estimate as $\hatqHpik(s_m,a_m,m)$ and observe it is  (nearly) an unbiased estimator of $\qHpik(s_m,a_m)$. We assume that each rollout runs sufficiently long so that the bias is small enough (the sampling process is fully described in Appendix~\ref{supp: approximate update rule TRPO}). Based on this data, Sample-Based TRPO updates the policy at the end of the $k$-th episode, by the following proximal problem, 
% \begin{align}
% &\pi_{k+1} \in \argmin_{\pi\in \ssimplex} \Big\{ \frac{1}{M} \sum_{m=1}^M \br*{1-\lambda t_k}\bregmanS{s_m}{\pi}{\pi_k}  \nonumber \\
% &+t_k A\inner{\hatqHpik(s_m,\cdot,m)\mathbbm{1}\brc*{\cdot=a_m},\pi(\cdot \! \mid \! s_m )-\pi_k (\cdot \!\mid \!s_m )}\Big\}   \label{eq: approximate scalar TRPO}
% \end{align}
% \todo{this equation is wrong}
% \begin{align}
% &\pi_{k+1} \in \argmin_{\pi\in \ssimplex} \Big\{ \frac{1}{M} \sum_{m=1}^M \br*{1-\lambda t_k}\bregmanS{s_m}{\pi}{\pi_k}  \nonumber \\
% &+t_k A \hatqHpik(s_m,a_m,m)(\pi(a_m \! \mid \! s_m )-\pi_k (a_m \!\mid \!s_m )\Big\}   \label{eq: approximate scalar TRPO}
% \end{align}

% \todo{Fix Suggestion A}
% \begin{align*}
% &\pi_{k+1} \in \nonumber\\
% & \argmin_{\pi\in \ssimplex} \frac{1}{M} \sum_{m=1}^M t_k \hat T_\lambda^\pi \hat v_{\lambda}^{\pi_k}(s_m,m) + \br*{1-\lambda t_k}\bregmanS{s_m}{\pi}{\pi_k}    %\label{eq: approximate scalar TRPO}
% \end{align*}
% \todo{Fix Suggestion B}
% \begin{align*}
% &\pi_{k+1} \in \argmin_{\pi\in \ssimplex} \Big\{ \frac{1}{M} \sum_{m=1}^M \bregmanS{s_m}{\pi}{\pi_k}  \nonumber \\
% &+t_k \inner{A \hatqHpik(s_m,a_m,m) \! + \!\nabla\omegaS{s_m}{\pi_k},\pi(a_m \! \mid \! s_m )\! -\! \pi_k (a_m \!\mid \!s_m )}\Big\}   %\label{eq: approximate scalar TRPO}
% \end{align*}

% \vspace{-0.3cm} %camera ready

\begin{align}\label{eq: approximate scalar TRPO}
\pi_{k+1} \in &\argmin_{\pi\in \ssimplex} \Big\{ \frac{1}{M}\sum_{m=1}^M  \frac{1}{t_k(1-\gamma)}  \bregmanS{s_m}{\pi}{\pi_k}  \nonumber\\
& \quad + \inner{ \hat \nabla \nu\vHpik[m], \pi(\cdot\mid s_m) - \pi_k(\cdot\mid s_m) } \Big\},
\end{align}
where the estimation of the gradient is $\hat \nabla \nu\vHpik[m]\triangleq\frac{1}{1-\gamma} \br*{A \hatqHpik(s_m,\cdot,m)\mathbbm{1}\{\cdot =  a_m\}+\lambda \nabla \omegaS{s_m}{\pi_k}}$.

The following proposition motivates the study of this update rule and formalizes its relation to Exact TRPO:

\begin{restatable}[Exact to Sample-Based  Updates]{proposition}{ExactToApproximateupdate}\label{proposition: exact to approximate}
Let $\mathcal{F}_k$ be the $\sigma$-field containing all events until the end of the $k-1$ episode. Then, for any $\pi,\pi_k\in \ssimplex$ and every sample $m$,
% \vspace{-0.15cm} %camera ready
\begin{align*}
     \langle \nabla &\nu\vHpik, \pi - \pi_k \rangle +  \frac{1}{t_k(1-\gamma)} d_{\nu,\pi_k} \bregman{\pi}{\pi_k} \\
    = \E \Big[&\inner{ \hat \nabla \nu\vHpik[m], \pi(\cdot\mid s_m) - \pi_k(\cdot\mid s_m) } \\
    & + \frac{1}{t_k(1-\gamma)}\bregmanS{s_m}{\pi}{\pi_k} \mid \mathcal{F}_k \Big].
\end{align*}
\end{restatable}

Meaning, the expectation of the proximal objective of Sample-Based TRPO~\eqref{eq: approximate scalar TRPO} is the proximal objective of Exact TRPO~\eqref{eq: exact scalar TRPO}. This fact motivates us to study this algorithm, anticipating it inherits the convergence guarantees of its exacted counterpart.

Like Uniform and Exact TRPO, Sample-Based TRPO has a simpler update rule, in which, the optimization takes place on every \emph{visited state} at the $k$-th episode. This comes in contrast to Uniform and Exact TRPO which require access to all states in $\mathcal{S}$ or $\mathcal{S}_{d_{\nu,\pi_k}}$, and is possible due to the \emph{sample-based adaptive scaling} of the Bregman distance. Let $\mathcal{S}^k_M$ be the set of visited states at the $k$-th episode, $n(s,a)$ the number of times $(s_m,a_m)=(s,a)$ at the $k$-th episode, and 
\begin{align*}
\hat{q}^{\pi_k}_\lambda(s,a)=\frac{A}{\sum_a n(s,a)}\sum_{i=1}^{n(s,a)} \hat{q}_\lambda^{\pi_k}(s,a, m_i),    
\end{align*}
is the empirical average of all rollout estimators for $q^{\pi_k}_\lambda(s,a)$ gathered in the $k$-th episode ($m_i$ is the episode in which  $(s_m,a_m)=(s,a)$ for the $i$-th time). If the state action pair $(s,a)$ was not visited at the $k$-th episode then $\hat{q}^{\pi_k}_\lambda(s,a)=0$. Given these definitions, Sample-Based TRPO updates the policy for all $s\in\mathcal{S}^k_M$ by a simplified update rule:
\begin{align*}
      &{\pi}_{k+1}(\cdot\mid s) \\ 
      & \in\arg \min_\pi {  t_k \inner{ \hat{q}^{\pi_k}_\lambda(s,\cdot) +\lambda \nabla \omegaS{s}{\pi_k},\pi} +  \bregmanS{s}{\pi}{\pi_k}} ,
\end{align*}

As in previous sections, the euclidean and non-euclidean choices of $\omega$ correspond to a PPG and NE-TRPO instances of Sample-Based TRPO. The different choices correspond to instantiating PolicyUpdate with the subroutines~\ref{alg: Policy Gradient policy update exact} or~\ref{alg: NE-TRPO policy update exact}. Generalizing the proof technique of Exact TRPO and using standard concentration inequalities, we derive a high-probability convergence guarantee for Sample-Based TRPO (see Appendix~\ref{supp: approximate Mirror descent}). An additional important lemma for the proof is Lemma~\ref{lemma: difference of gradient of omegas approximate} provided in the appendix. This lemma bounds the change $\nabla \omega(\pi_k)- \nabla \omega(\pi_{k+1})$ between consecutive episodes by a term proportional to $t_k$. Had this bound been $t_k$-independent, the final results would deteriorate significantly. 

\begin{theorem}[Convergence Rate: Sample-Based TRPO]\label{theorem: Sample-Based TRPO All}  
Let $\{\pik\}_{k\geq 0}$ be the sequence generated by Sample-Based TRPO, using $M_k \geq O\br*{\frac{A^2 \CmaxLambdaSquare\br*{S\log A + \log 1/\delta}}{(1-\gamma)^2\epsilon^2}}$ samples in each iteration, and $\{\mu v_{\textbf{best}}^k\}_{k\geq0}$ be the sequence of best achieved values, $\mu v_{\textbf{best}}^N\triangleq \arg\min_{k=0,...,N} \mu v^{\pik}_\lambda-\mu v_\lambda^*$. Then, with probability greater than $1-\delta$ for every $\epsilon>0$ the following holds for all $N\geq1$:
\begin{enumerate}
    \item \emph{(Unregularized)} Let $\lambda=0$, $t_k=\frac{(1-\gamma)}{\euclid \Cmax  \sqrt{k+1}}$, then
    \begin{align*}
       &\mu v_{\textbf{best}}^N - \mu \vstar  \leq O\br*{\frac{\euclid \Cmax}{(1-\gamma)^2\sqrt{N}}+ \frac{C^{\pi^*}\epsilon}{(1-\gamma)^2}}.
    \end{align*}
    \item \emph{(Regularized)} Let $\lambda>0$, $t_k =  \frac{1}{\lambda (k+2)}$, then
% \begin{align*}
%     \norm{v_{\textbf{best}}^N - \vstar_\lambda}_\infty \leq O\left(\frac{1}{\lambda(1-\gamma)^3}\frac{\lambda^2\omegaBound (1-\gamma)^2 +\euclid ^2\CmaxLambdaSquare\log {N}}{N}\right).
% \end{align*}
\begin{align*}
  \mu  v_{\textbf{best}}^N - \mu  \vHstar   \leq  O \br*{ \frac{\euclid ^2 \approxAgrowth\CmaxLambda^2 }{\lambda(1
    -\gamma)^3 N} + \frac{C^{\pi^*}\epsilon}{(1-\gamma)^2}}.
\end{align*}
\end{enumerate}
Where $\approxAgrowth = 1$ for the euclidean case, and $\approxAgrowth = A^2$ for the non-euclidean case.
% Where $\euclid =\sqrt{A}$ for the euclidean case, and $\euclid =1$ for the non-euclidean case. 
\end{theorem}

Similarly to Uniform TRPO, the convergence rates are $\tilde O(1/\sqrt{N})$ and $\tilde O(1/N)$ for the unregularized and regularized cases, respectively. However, the Sample-Based TRPO converges to an approximate solution, similarly to CPI. The \emph{sample complexity} for a $\frac{C^{\pi^*}\epsilon}{(1-\gamma)^2}$ error, the same as the error of CPI, is given in Table~\ref{table: sample complexity}. Interestingly, Sample-Based TRPO has better polynomial sample complexity in $(1-\gamma)^{-1}$ relatively to CPI. Importantly, \textbf{the regularized versions have a superior sample-complexity in~$\epsilon$}, which can explain the empirical success of using regularization.

\begin{table}\label{table: sample complexity}
\small
\centering
\renewcommand*{\arraystretch}{2}
% \setlength{\extrarowheight}{10pt}
% \begin{tabular}{|m{2.5cm}|m{6cm}|m{6cm}|}
\begin{tabular}{|c|c|}
    \hline
     {\bf Method} &  {\bf Sample Complexity} \\
     \hline
     TRPO (this work) &  $ \frac{ \euclid^2 A^2 C_{\text{max}}^4 \br*{S + \log \frac{1}{\delta}} }{\br*{1-\gamma}^3\epsilon^4}$ \\ \hline
     Regularized TRPO (this work) & $\frac{\euclid ^2 \approxAgrowth A^2 C_{\text{max},\lambda}^4 \br*{S + \log \frac{1}{\delta}} }{\lambda\br*{1-\gamma}^4\epsilon^3}$ \\ \hline
     CPI (\citeauthor{kakade2002approximately}) & $\frac{A^2 C_{\text{max}}^4 \br*{S + \log \frac{1}{\delta}} }{ \br*{1-\gamma}^5\epsilon^4} $ \\ \hline
     
\end{tabular}
\caption{The sample complexity of Sample-Based TRPO (TRPO) and CPI. For TRPO, the best policy so far is returned, where for CPI, the last policy $\pi_N$ is returned.}

%\vspace{-0.5cm} %camera ready

\end{table}

\begin{remark}[Optimization Perspective]
From an optimization perspective, CPI can be interpreted as a sample-based  Conditional Gradient Descent (Frank-Wolfe) for solving MDPs~\cite{scherrer2014local}. With this in mind, the two analyzed instances of Sample-Based TRPO establish the convergence of sample-based projected and exponentiated gradient descent methods for solving MDPs: PPG and NE-TRPO. It is well known that a convex problem can be solved with any one of the three aforementioned methods. The convergence guarantees of CPI together with the ones of Sample-Based TRPO establish the same holds for RL. 
% Understanding the differences between these approaches is of importance.
% Understanding more thoroughly the advantages and disadvantages of the different approaches is of importance. 
\end{remark}
\begin{remark}[Is Improvement and Early Stopping Needed?]
Unlike CPI, Sample-Based TRPO does not rely on improvement arguments or early stopping. Even so, its asymptotic performance is equivalent to CPI, and its sample complexity has better polynomial dependence in $(1-\gamma)^{-1}$. This questions the necessity of ensuring improvement for policy search methods, heavily used in the analysis of these methods, yet less used in practice, and motivated by the analysis of CPI.
\end{remark}

\section{Related Works}

The empirical success of policy search and regularization techniques in RL \cite{peters2008natural,mnih2016asynchronous,schulman2015trust,schulman2017proximal} led to non-negligible theoretical analysis of these methods. Gradient based policy search methods were mostly analyzed in the function approximation setting, e.g.,~\cite{sutton2000policy,bhatnagar2009natural,pirotta2013adaptive,dai2017sbeed,papini2019smoothing,bhandari2019global}. There, convergence to a local optimum was established under different conditions and several aspects of policy search methods were investigated. In this work, we study a trust-region based, as opposed to gradient based, policy search method in tabular RL and establish global convergence guarantees. Regarding regularization in TRPO, in \citet{neu2017unified} the authors analyzed entropy regularized MDPs from a linear programming perspective for average-reward MDPs. Yet, convergence rates were not supplied, as opposed to this paper.

In~\citet{geist2019theory} different aspects of regularized MDPs were studied, especially, when combined with MD-like updates in an approximate PI scheme (with partial value updates). The authors focus on update rules which require uniform access to the state space of the form ${\pi_{k+1}=\argmin_{\pi\in \ssimplex} \langle q_k, \pi - \pi_k \rangle + \bregman{\pi}{\pi_k}}$, similarly to the simplified update rule of Uniform TRPO~\eqref{eq: uniform TRPO simplied update} with a fixed learning rate, $t_k=1$. In this paper, we argued it is instrumental to view this update rule as an instance of the more general update rule~\eqref{eq: PG iteration paper}, i.e., MD with an adaptive proximity term. This view allowed us to formulate and analyze the adaptive Sample-Based TRPO, which does not require uniform access to the state space. Moreover, we proved Sample-Based TRPO inherits the same asymptotic performance guarantees of CPI. Specifically, the quality of the policy Sample-Based TRPO outputs depends on the concentrability coefficient $C^{\pi^*}$. The results of~\citet{geist2019theory} in the approximate setting led to a worse concentrability coefficient, $C^i_q$, which can be infinite even when $C^{\pi^*}$ is finite~\cite{scherrer2014approximate} as it depends on the worst case of all policies.

In a recent work of~\citet{agarwal2019optimality}, Section~4.2, the authors study a variant of Projected Policy Gradient Descent and analyze it under the assumption of exact gradients and uniform access to the state space. The proven convergence rate depends on both $S$ and $C^{\pi^*}$ whereas the convergence rate of Exact TRPO (Section~\ref{sec: exact MD}) does not depend on $S$ nor on $C^{\pi^*}$ (see Appendix~\ref{supp: convergence of exact trpo}), and is similar to the guarantees of Uniform TRPO (Theorem~\ref{theorem: uniform TRPO}). Furthermore, the authors do not establish faster rates for regularized MDPs. It is important to note their projected policy gradient algorithm is \emph{different} than the one we study, which can explain the discrepancy between our results. Their projected policy gradient updates by $\pi_{k+1}\in P_{ \ssimplex}\br*{\pi_k - \eta \nabla \mu v^{\pi_k}}$, whereas, the Projected Policy Gradient studied in this work applies a different update rule based on the adaptive scaling of the Bregman distance. 

Lastly, in another recent work of~\citet{liu2019neural} the authors established global convergence guarantees for a sampled-based version of TRPO when neural networks are used as the $q$-function and policy approximators. The sample complexity of their algorithm is $O(\epsilon^{-8})$ (as opposed to $O(\epsilon^{-4})$ we obtained)  neglecting other factors. It is an interesting question whether their result can be improved.

\section{Conclusions and Future Work}
We analyzed the Uniform and Sample-Based TRPO methods. The first is a planning, trust region method with an adaptive proximity term, and the latter is an RL sample-based version of the first.
% , as formalized in propositions~\ref{proposition: uniform to exact update},~\ref{proposition: exact to approximate}
% Furthermore, by choosing a specific form of the proximity term (i.e., the Bregman distance) we obtained two instances -- PPG and NE-TRPO -- of the more general TRPO methods.
Different choices of the proximity term led to two instances of the TRPO method: PPG and NE-TRPO. For both, we proved $\tilde O(1/\sqrt{N})$ convergence rate to the global optimum, and a faster $\tilde O(1/N)$ rate for regularized MDPs. Although Sample-Based TRPO does not necessarily output an improving sequence of policies, as CPI, its best policy in hindsight does improve. Furthermore, the asymptotic performance of Sample-Based TRPO  is equivalent to that of CPI, and its sample complexity exhibits better dependence in $(1-\gamma)^{-1}$. These results establish the popular NE-TRPO~\cite{schulman2015trust} should not be interpreted as an approximate heuristic to CPI but as a viable alternative.

% Importantly, we showed the adaptive proximity term naturally imposes the structure of the problem onto the trust region algorithm, by studying the linear approximation of the objective function~(Proposition~\ref{prop: directional derivative for MDPs}).

In terms of future work, an important extension of this study is deriving algorithms with linear convergence, or, alternatively, establish impossibility results for such rates in RL problems. Moreover, while we proved positive results on regularization in RL, we solely focused on the question of optimization. We believe that establishing more positive as well as negative results on regularization in RL is of value. Lastly, studying further the implication of the adaptive proximity term in RL is of importance due to the empirical success of NE-TRPO and its now established convergence guarantees.

\section{Acknowledgments}
We would like to thank Amir Beck for illuminating discussions regarding Convex Optimization and Nadav Merlis for helpful comments. This work was partially funded by the Israel Science
Foundation under ISF grant number 1380/16.
%  We showed the relation between U The first is a planning algorithm, whereas the  We showed that Sample-Based TRPO NE-TRPO~\cite{schulman2015trust}  these algorithms converge to a global optimum and established $\tilde O(1/\sqrt{N})$ convergence rate, even though the RL optimization problem is non-convex. Furthermore, we showed that by regularizing the MDP, the optimization process converges faster to the global optimum with rate of $\tilde O(1/N)$, for both PPG and NE-TRPO algorithms

%, this is the first provable guarantee that regularization in RL is beneficial from an optimization perspective. 
% \small{
\bibliography{Bibliography.bib}
% }
\bibliographystyle{aaai}

\newpage
% \appendix
\onecolumn
\setlength{\parskip}{1em}
\setlength{\parindent}{0ex}
% \renewcommand{\baselinestretch}{1.5}

% \section{TODO}

% \todo{organize the proof of the approximate case. (yonathan)}

% \todo{reread the entire paper including appendices and verify proofs and text.}

% \todo{(TO YONATHAN) I fixed the usage of $\bregmanS{s}{\pi}{\pi'}$ all through the paper. HOWEVER: in section~\ref{supp: extended value functions} we use a different notation. It's okay, because it is self consistent and we define it within this section. We chose that because we want to use $v_s(y)$ to be with the same notation as $\omega_s(y)$. We should decide if we want to keep it or change to $\omegaS{s}{y}$ and thus to $v(s;y)$ץ In my opinion - keep at as it is. It is okay that it is different because we want to reader get outside the scope of RL conventions in order to understand this part. YOUR CHOICE}

%comment out \addcontentsline in aaai20
\listofappendices

% \pagenumbering{arabic}
\begin{appendices}
\pagestyle{plain}

\section{Assumptions of Mirror Descent}\label{supp: assumptions omega}

\begin{assumption}[properties of Bregman distance]\label{assumption: Bregman Distance}
\begin{enumerate}[(A)]
    \item[]
    \item $\omega$ is proper closed and convex.
    \item $\omega$ is differentiable over $\dom(\partial\omega)$.
    \item $C \subseteq \dom(\omega)$
    \item $\omega+\delta_C$ is $\sigma$-strongly convex ($\sigma>0$)
\end{enumerate}
\end{assumption}

Assumption \ref{assumption: Bregman Distance} is the main assumption regarding the underlying Bregman distance used in Mirror Descent. In our analysis, we have two common choice of $\omega$: a) the negative entropy function, denoted as $H(\cdot)$, for which the corresponding Bregman distance is $\bregman{\cdot}{\cdot}=\dkl{\cdot}{\cdot}$. b) the euclidean norm $\omega(\cdot)=\frac{1}{2} \norm{\cdot}^2$, for which the resulting Bregman distance is the euclidean distance. The convex optimization domain $C$ is in our case $\ssimplex$, the state-wise unit simplex over the space of actions. For both choices, the assumption holds. Finally, $\delta_C(x)$ is an extended real valued function which describes the optimization domain $C$. It is defined as follows: For $x \in C$, $\delta_C(x)=0$. For $x \notin C$, $\delta_C(x)=\infty$. For more details, see \cite{beck2017first}.
% The negative entropy is proper closed and convex, and differentiable over $\dom(\partial\omega)$ \todo{cite}. Furthermore, it is known to be $1$-strongly convex over $\simplex$ \todo{cite}. Hence, the assumption holds. However, assumption~\ref{assumption: Bregman Distance} can be used to derive more general algorithms.

We go on to define the second assumption regarding the optimization problem:
\begin{assumption}\label{assumption: main}
\begin{enumerate}[(A)]
    \item[]
    \item $f:\E\rightarrow (-\infty,\infty]$ is proper closed.
    \item $C \subseteq \E$ is nonempty closed and convex.
    \item $C \subseteq int(\dom(f))$.
    \item The optimal set of (P) is nonempty.
\end{enumerate}
\end{assumption}

\section{Policy Gradient, and Directional Derivatives for Regularized MDPs}\label{supp: policy gradients and directional derivatives}

In this section we re-derive the Policy Gradient Theorem \cite{sutton2000policy} for regularized MDPs when tabular representation is used. Meaning, we explicitly calculate the derivative $\nabla_\pi \vHpi (s)$. Based on this result, we derive the directional derivative, or the linear approximation of the objective functions, $\inner{\nabla_\pi \vHpi  (s), \pi-\pi'}, \inner{\nabla_\pi \mu \vHpi (s), \pi-\pi'}$. 

\subsection{Extended Value Functions}\label{supp: extended value functions}
To formally study $\nabla_\pi \vHpi (s)$ we need to define value functions $v^\pi$ when $\pi$ is outside of the simplex $\ssimplex$, since when $\pi(a\mid s)$ changes infinitesimally, $\pi(\cdot \mid s)$ does not remain a valid probability distribution.  To this end, we study \emph{extended value functions} denoted by $v(y)\in \mathbb{R}^S$ for $y\in \mathbb{R}^{S\times A}$, and denote $v_s(y)$ as the component of $v(y)$ which corresponds to the state $s$. Furthermore, we define the following cost and dynamics,
\begin{align*}
    &c^y_{\lambda,s} \triangleq \sum_{a'} y(a'\mid s)\br*{c(s,a)+\lambda \omega_s(y}),\\
    &p^y_{s,s'} \triangleq \sum_{a'} y(a'\mid s)p(s'\mid s,a'),
\end{align*}
where $\omega_s(y) \triangleq \omega(y(\cdot \mid s))$ for $\omega: \mathbb{R}^A\rightarrow \mathbb{R}$, $p^y\in \mathbb{R}^{S\times S}$ and $c^y_{\lambda}\in \mathbb{R}^S$. 

\begin{defn}[Extended value and $q$ functions.]
An extended value function is a mapping $v:\mathbb{R}^{S\times A} \rightarrow \mathbb{R}^S$, such that for $y\in \mathbb{R}^{S\times A}$ 
\begin{align}
    v(y) \triangleq  \sum_{t=0}^{\infty} \gamma^t (p^y)^t c^y_{\lambda} , \label{eq: definition v(y)}
\end{align}

Similarly, an extended $q$-function is a mapping $q:\mathbb{R}^{S\times A} \rightarrow \mathbb{R}^{S\times A}$, such that its $s,a$ element is given by
\begin{align}
    q_{s,a}(y) \triangleq  c(s,a)  + \lambda \omega_s(y) + \gamma \sum_{s'} p(s'\mid s,a)v^y_{s'}, \label{eq: definition q(y)}
\end{align}

When $y\in \ssimplex$ is a policy, $\pi$, we denote $v(\pi)\triangleq \vHpi \in \mathbb{R}^{S}, q(\pi)=\qHpi\in \mathbb{R}^{S\times A}$.
\end{defn}

Note that in this section we use different notations than the rest of the paper, in order to generalize the discussion and keep it out of the regular RL conventions.

The following proposition establishes that $v(y)$ the fixed point of a corresponding Bellman operator when $y$ is close to the simplex component-wise.

\begin{lemma}\label{lemma: extended value function}
Let $y \in \{y'\in \mathbb{R}^{S\times A}: \forall s, \sum_{a}|y'(a\mid s)|< \frac{1}{\gamma} \}$. Define the operator $T^y:\mathbb{R}^S \rightarrow \mathbb{R}^S$ , such that for any $v\in \mathbb{R}^S$,
\begin{align*}
    (T^y v)_s \triangleq c_{\lambda,s}^y  +\gamma \sum_{s'}p^y_{s,s'} v_{s'}.
\end{align*}
Then, 
\begin{enumerate}
    \item $T^y$ is a contraction operator in the max norm.
    \item  Its fixed-point is $v(y)$ and satisfies
$
    v_{s}(y) = (T^y v(y))_s.
$
\end{enumerate}
\end{lemma}

\begin{proof}
We start by proving the first claim. Unlike in classical results on MDPs, $y$ is not a policy. However, since it is not `too far' from being a policy we get the usual contraction property by standard proof techniques. 

Let $v',v\in \mathbb{R}^S$, and assume $ (T^y v')_s\geq  (T^y v)_s$.
\begin{align*}
    (T^y v')_s - (T^y v)_s &= \gamma \sum_{a}y(a\mid s)\sum_{s'}p(s'\mid s,a)(v'_{s'}-v_{s'})\\
    &\leq \gamma\sum_{a}y(a\mid s) \sum_{s'}p(s'\mid s,a)\norm{v'_{s'}-v_{s'}}_{\infty}\\
    &= \gamma \norm{v'_{s'}-v_{s'}}_{\infty}\sum_{a}y(a\mid s) \\
    &\leq \gamma \norm{v'_{s'}-v_{s'}}_{\infty}\sum_{a}|y(a\mid s)| \\   
    &<  \norm{v'_{s'}-v_{s'}}_{\infty} .   
\end{align*}
In the fourth relation we used the assumption that $\gamma \sum_{a}|y(a\mid s)|< 1$. Repeating the same proof for the other case where $(T^y v')_s <  (T^y v)_s$, concludes the proof of the first claim.

To prove the second claim, we use the definition of $v(y)$.
\begin{align*}
    v(y) &\triangleq   \sum_{t=0}^{\infty} \gamma^t (p^y)^t c_\lambda^y \\
    &= c_\lambda^y + \sum_{t=1}^{\infty} \gamma^t (p^y)^t c_\lambda^y  \\
    &= c_\lambda^y + \gamma p^y\br*{\sum_{t=0}^{\infty} \gamma^t (p^y)^{t}c_\lambda^y} \\
    &= c_\lambda^y + \gamma p^y v(y).
\end{align*}
In the third relation we used the distributive property of matrix multiplication and in the forth relation we used the definition of $v(y)$. Thus, $v(y) = T^y v(y)$, i.e., $v(y)$ is the fixed point of the operator~$T^y$.
\end{proof}

\subsection{Policy Gradient Theorem for Regularized MDPs}

We now derive the Policy Gradient Theorem for regularized MDPs for tabular policy representation. Specifically, we use the notion of an extended value function and an extended $q$-functions defined in the previous section.

\begin{lemma}
Let $y\in \{y'\in \mathbb{R}^{S\times A}: \forall s, \sum_{a}|y'(a\mid s)|< \frac{1}{\gamma} \}$. Then,
$$
v_s(y) = \sum_{a'} y(a'\mid s) q_{s,a'}(y)
$$
\end{lemma}
\begin{proof}
Using \eqref{eq: definition q(y)}, we get
\begin{align*}
    \sum_{a'} y(a'\mid s)q_{s,a'}(y)&= \sum_{a'} y(a'\mid s)(c(s,a)+\lambda \omega_s(y)) +\gamma \sum_{s'} p(s'\mid s,a')v_{s'}(y)\\
    &=c^y_{\lambda,s} + \gamma \sum_{s'} p(s'\mid s,a')v_{s'}(y)= v_s(y),
\end{align*}
where the last equality is by the fixed-point property of Lemma \ref{lemma: extended value function}.
\end{proof}

We now derive the Policy Gradient Theorem for extended (regularized) value functions.
\begin{theorem}[Policy Gradient for Extended Regularized Value Functions]
Let $y \in \{y: \forall s, \sum_{a}|y(a\mid s)|< \frac{1}{\gamma} \}$. Furthermore, consider a fixed $s,a$ and $\bar s$. Then,
\begin{align*}
    \partial_{y_{\bar{s},\bar{a}}} v_{s}(y)= \sum_{t=0}^\infty \gamma^t p_t^y(s_t\mid s)\delta_{\bar{s},s_t}\left( q_{s,\bar{a}}(y) + \lambda \partial_{y_{s,\bar{a}}}\omega_s(y)\br*{\sum_{a'}y(a'\mid s)}\right),
\end{align*}
where $p^y(s_t \mid s) = \sum_{s_1,..,s_t} p^y(s_{t}\mid s_{t-1})\cdot\cdot\cdot  p^y(s_{1}\mid s),$ and  $p^y_t(s_0\mid s)=1$.
\end{theorem}

\begin{proof}
Following similar derivation to the original Policy Gradient Theorem~\cite{sutton2000policy}, for every $s$,
\begin{align*}
    &\partial_{y_{\bar{s},\bar{a}}} v_{s}(y) \\
    &= \sum_{a'} (\partial_{y_{\bar{s},\bar{a}}}y(a'\mid s)) q_{s,a'}(y) + y(a'\mid s)\partial_{y_{\bar{s},\bar{a}}}q_{s,a'}(y)\\
    &= \sum_{a'} \delta_{s,\bar{s}}\delta_{a',\bar{a}} q_{s,a'}(y) + y(a'\mid s)\partial_{y_{\bar{s},\bar{a}}}q_{s,a'}(y).
\end{align*}

We now explicitly write the last term,
\begin{align*}
    &\partial_{y_{\bar{s},\bar{a}}}q_{s,a'}(y)\\
    & = \partial_{y_{\bar{s},\bar{a}}}\left(c(s,a')+ \lambda \omega_s(y) + \gamma \sum_{s'} p(s'\mid s,a')v_{s'}(y)\right)\\
    & =  \lambda \delta_{s,\bar{s}}\partial_{y_{s,\bar{a}}}\omega_s(y) + \gamma \sum_{s'} p(s'\mid s,a')\partial_{y_{\bar{s},\bar{a}}}v_{s'}(y).
\end{align*}

Plugging this back yields,
\begin{align*}
    &\partial_{y_{\bar{s},\bar{a}}} v_{s}(y) \\
    &= \sum_{a'} \delta_{s,\bar{s}}\delta_{a',\bar{a}} q_{s,a'}(y) +  \lambda y(a'\mid s)\delta_{s,\bar{s}}\partial_{y_{s,\bar{a}}}\omega_s(y) \\
    &\quad+ \gamma \sum_{s'}\sum_{a'} y(a'\mid s)p(s'\mid s,a')\partial_{y_{\bar{s},\bar{a}}}v_{s'}(y)\\
    &= \sum_{a'} \delta_{s,\bar{s}}\delta_{a',\bar{a}} q_{s,a'}(y) +  \lambda y(a'\mid s)\delta_{s,\bar{s}}\partial_{y_{s,\bar{a}}}\omega_s(y)+ \gamma \sum_{s'}p^y(s'\mid s)\partial_{y_{\bar{s},\bar{a}}}v_{s'}(y).
\end{align*}

Iteratively applying this relation yields
\begin{align*}
    \partial_{y_{\bar{s},\bar{a}}} v_{s}(y)= \sum_{t=0}^\infty \gamma^t p_t^y(s_t\mid s)\delta_{\bar{s},s_t}\left( q_{s,\bar{a}}(y) + \lambda \partial_{y_{s,\bar{a}}}\omega_s(y)\br*{\sum_{a'}y(a'\mid s)}\right),
\end{align*}
where,
\begin{align*}
    p^y(s_t \mid s) = \sum_{s_1,..,s_t} p^y(s_{t}\mid s_{t-1})\cdot\cdot\cdot  p^y(s_{1}\mid s),
\end{align*}
and $p^y_t(s_0\mid s)=1$.
\end{proof}

Returning to the specific notation for RL, defined in Section~\ref{sec: notation and pre}, by setting $y=\pi$, i.e., when $y$ is a policy, we get the Policy Gradient Theorem for regularized MDPs, since for all $s$, $\sum_{a'}\pi(a'\mid s)=1$.
\begin{corollary}[Policy Gradient for Regularized MDPs]\label{Cor: Policy Gradient thm}
Let $\pi \in \ssimplex$. Then, $\nabla_{\pi} v^\pi \in \mathbb{R}^{S\times S\times A}$ and
\begin{align*}
    \nabla_{\pi} v^\pi(s,\bar s,\bar a) \triangleq \nabla_{\pi(\bar a \mid \bar s)} v_\lambda^\pi(s)= \sum_{t=0}^\infty \gamma^t p^\pi(s_t=\bar{s}\mid s)\left(\lambda \partial_{\pi(\bar a \mid \bar s)}\omega^\pi(\bar{s})+  \qHpi(\bar{s},\bar{a})\right).
\end{align*}
\end{corollary}

\subsection{The Linear Approximation of the Policy's Value and The Directional Derivative for Regularized MDPs}

In this section, we derive the directional derivative in policy space for regularized MDPs with tabular policy representation. 

The linear approximation of the value function of the policy $\pi'$, around the policy $\pi$, is given by

$$ v_{\lambda}^{\pi'} \approx \vHpi + \inner*{ \nabla_{\pi} \vHpi , \pi' - \pi}$$

In the MD framework, we take the $\argmin$ w.r.t. to this linear approximation. Note that the minimizer is independent on the zeroth term, $\vHpi$, and thus the optimization problem depends only on the directional derivative, $\inner*{ \nabla_{\pi} \vHpi , \pi' - \pi}$. To keep track with the MD formulation, we chose to refer to Proposition~\ref{prop: directional derivative for MDPs} as the `linear approximation of a policy's value', even though it is actually the directional derivative.

\propDirectionalDerivative*
See that \eqref{eq: Regularized value directional derivative second} is a vector in $\mathbb{R}^S$, whereas \eqref{eq: Regularized value directional derivative first} is a scalar.

\begin{proof}
We start by proving the first claim. Consider the inner product, $\left\langle \nabla_{\pi(\cdot\mid \bar{s})} v^\pi(s), \pi'(\cdot\mid \bar{s}) - \pi(\cdot\mid \bar{s}) \right\rangle$. By the linearity of the inner product and using Corollary~\ref{Cor: Policy Gradient thm} we get,
\begin{align}
    &\left\langle \nabla_{\pi(\cdot\mid \bar{s})} v^\pi(s), \pi'(\cdot\mid \bar{s}) - \pi(\cdot\mid \bar{s}) \right\rangle \nonumber\\
    &= \sum_{t=0}^\infty \gamma^t p^\pi(s_t=\bar{s}\mid s)\left\langle \lambda \nabla_{\pi(\cdot \mid \bar s)}\omegaS{\bar {s}}{\pi}+  \qHpi(\bar{s},\cdot), \pi'(\cdot\mid \bar{s}) - \pi(\cdot\mid \bar{s}) \right\rangle \nonumber\\
    &= \sum_{t=0}^\infty \gamma^t p^\pi(s_t=\bar{s}\mid s)\left( \lambda \left\langle \nabla_{\pi(\cdot \mid \bar s)}\omegaS{\bar{s}}{\pi}, \pi'(\cdot\mid \bar{s}) - \pi(\cdot\mid \bar{s}) \right\rangle+  \left\langle \qHpi(\bar{s},\cdot) , \pi'(\cdot\mid \bar{s}) - \pi(\cdot\mid \bar{s}) \right\rangle \right), \label{eq: directional derivative proof 1}
\end{align}

The following relations hold. \yon{replace $a$ with $a'$ in the lines below}
\begin{align}
    &\left\langle q_\lambda^\pi(\bar{s},\cdot), \pi'(\cdot\mid \bar{s}) - \pi(\cdot\mid \bar{s}) \right\rangle \nonumber\\
    &= \left\langle q_\lambda^\pi(\bar{s},\cdot), \pi'(\cdot\mid \bar{s}) \right\rangle - \left\langle q_\lambda^\pi(\bar{s},\cdot), \pi(\cdot\mid \bar{s}) \right\rangle \nonumber\\
    &= \sum_{a'} \pi'(a'\mid \bar s)\br*{ c(\bar{s},a) + \lambda \omegaS{\bar s}{\pi} +\gamma\sum_{s'}P(s'\mid \bar{s},a)\vHpi(s')}\nonumber\\
    &\quad -\sum_{a'} \pi(a'\mid \bar s)\br*{ c(\bar{s},a) + \lambda \omegaS{\bar s}{\pi} +\gamma\sum_{s'}P(s'\mid \bar{s},a)\vHpi(s')}\nonumber\\
    &= \sum_{a'} \pi'(a'\mid \bar{s})\br*{ c(\bar{s},a) + \lambda\omegaS{\bar s}{\pi} +\gamma\sum_{s'}P(s'\mid \bar{s},a)\vHpi(s')} - \vHpi(\bar{s})\nonumber\\
    & = \sum_{a'} \pi'(a'\mid \bar s)\br*{ c(\bar{s},a) + \lambda \omegaS{\bar s}{\pi'} - \lambda \omegaS{\bar s}{\pi'} + \lambda \omegaS{\bar s}{\pi} +\gamma\sum_{s'}P(s'\mid \bar{s},a)\vHpi(s')}- \vHpi(\bar{s}) \nonumber \\
&= \sum_{a'} \pi'(a'\mid \bar s)\br*{ c(\bar{s},a) + \lambda \omegaS{\bar s}{\pi'} +\gamma\sum_{s'}P(s'\mid \bar{s},a)\vHpi(s')}- \vHpi(\bar{s}) + \lambda\br*{\omegaS{\bar{s}}{\pi}- \omegaS{\bar{s}}{\pi'}}\nonumber\\
    &=  c_\lambda^{\pi'}(\bar{s}) +\gamma\sum_{s'}P^{\pi'}(s'\mid \bar{s})\vHpi(s')- \vHpi(\bar{s}) + \lambda\br*{\omegaS{\bar s}{\pi}- \omegaS{\bar s}{\pi'}}\nonumber\\
    &= (T^{\pi'}_\lambda v^{\pi}_\lambda)(\bar{s})- \vHpi(\bar{s}) + \lambda\br*{\omegaS{\bar{s}}{\pi}- \omegaS{\bar{s}}{\pi'}}
    \label{eq: relation between bellman and directional q}
\end{align}
The third relation holds by the fixed-point property of $v_\lambda^\pi$, and the last relation is by the definition of the regularized Bellman operator.

Plugging this back into \eqref{eq: directional derivative proof 1}, we get,
\begin{align}
    &\left\langle \nabla_{\pi(\cdot\mid \bar{s})} v^\pi(s), \pi'(\cdot\mid \bar{s}) - \pi(\cdot\mid \bar{s}) \right\rangle \nonumber\\
    &= \sum_{t=0}^\infty \gamma^t p^\pi(s_t=\bar{s}\mid s) \times \nonumber\\
    & \left( -\lambda \left( \omegaS{s}{\pi'}- \omegaS{s}{\pi}-\left\langle \nabla_{\pi(\cdot \mid \bar s)}\omegaS{\bar {s} }{\pi}, \pi'(\cdot\mid \bar{s}) - \pi(\cdot\mid \bar{s}) \right\rangle\right)+  (T^{\pi'}_\lambda  v_\lambda^\pi)(\bar{s}) -  v_\lambda^\pi(\bar{s}) \right) \nonumber \\
    &= \sum_{t=0}^\infty \gamma^t p^\pi(s_t=\bar{s}\mid s)\left( (T^{\pi'}_\lambda  v_\lambda^\pi)(\bar{s}) -  v_\lambda^\pi(\bar{s})-\lambda \bregmanS{\bar{s}}{\pi'}{\pi} \right) \label{eq supp: inner prodcut per state s bar}
\end{align}

Thus, we have that
\begin{align*}
    \left\langle \nabla_{\pi} v^\pi(s), \pi' - \pi \right\rangle &\triangleq \sum_{\bar s} \sum_a  \nabla_{\pi(a\mid \bar{s})} v^\pi(s)\left( \pi'(a\mid \bar{s}) - \pi(a\mid \bar{s}) \right)\\
     &= \sum_{\bar s}  \left\langle \nabla_{\pi(\cdot\mid \bar{s})} v^\pi(s), \pi'(\cdot\mid \bar{s}) - \pi(\cdot\mid \bar{s}) \right\rangle\\
    &= \sum_{\bar s}  \sum_{t=0}^\infty \gamma^t p^\pi(s_t=\bar{s}\mid s)\left( (T^{\pi'}_\lambda  v_\lambda^\pi)(\bar{s}) -  v_\lambda^\pi(\bar{s})-\lambda \bregmanS{\bar{s}}{\pi'}{\pi} \right)\\
    &= \sum_{\bar s}  (I-\gamma P^\pi)^{-1}_{s,\bar s}\left( (T^{\pi'}_\lambda  v_\lambda^\pi)(\bar{s}) -  v_\lambda^\pi(\bar{s})-\lambda \bregmanS{\bar{s}}{\pi'}{\pi} \right)\\
    & =  \brs*{(I-\gamma P^\pi)^{-1}\left( T^{\pi'}_\lambda  v_\lambda^\pi -  v_\lambda^\pi-\lambda \bregman{\pi'}{\pi} \right)}(s).
\end{align*}
Where the third relation is by~\eqref{eq supp: inner prodcut per state s bar},
the forth by defining the matrix $\sum_{t=0}^{\infty}\gamma^t P^\pi = (I-\gamma P^\pi)^{-1}$, and the fifth by the definition of matrix-vector product.

To prove the second claim, multiply both sides of the first relation~\eqref{eq: Regularized value directional derivative first} by $
\mu$. For the LHS we get,
\begin{align*}
    \sum_{s} \mu(s)\left\langle \nabla_{\pi(\cdot\mid \bar{s})} v^\pi(s), \pi'(\cdot\mid \bar{s}) - \pi(\cdot\mid \bar{s}) \right\rangle &=     \left\langle \sum_{s} \mu(s) \nabla_{\pi(\cdot\mid \bar{s})} v^\pi(s), \pi'(\cdot\mid \bar{s}) - \pi(\cdot\mid \bar{s}) \right\rangle \\
 &=     \left\langle  \nabla_{\pi(\cdot\mid \bar{s})}  \sum_{s} \mu(s) v^\pi(s), \pi'(\cdot\mid \bar{s}) - \pi(\cdot\mid \bar{s}) \right\rangle \\
 &=     \left\langle  \nabla_{\pi(\cdot\mid \bar{s})}   \mu v^\pi, \pi'(\cdot\mid \bar{s}) - \pi(\cdot\mid \bar{s}) \right\rangle.
\end{align*}
In the first and second relation we used the linearity of the inner product and the derivative, and in the third relation the definition of $\mu v^\pi$. Lastly, observe that multiplying the  RHS by $\mu$ yields  ${\mu(I-\gamma P^{\pi})^{-1}=\frac{1}{1-\gamma}d_{\mu,\pi}}$.
\end{proof}

\section{Uniform Trust Region Policy Optimization}\label{supp: uniform TRPO proof}

In this Appendix, we derive the Uniform TRPO algorithm (Algorithm~\ref{alg: uniform TRPO}) and prove its convergence for both the unregularized and regularized versions. As discussed in Section \ref{sec: uniform TRPO}, both Uniform Projected Policy Gradient and Uniform NE-TRPO are instances of Uniform TRPO, by a proper choice of the Bregman distance. 
In Appendix~\ref{supp: uniform TRPO}, we explicitly show that the iterates
\begin{align}
\pi_{k+1}\in \argmin_{\pi\in \ssimplex} \left\{ \langle \nabla \vHpik, \pi - \pi_k \rangle + \frac{1}{t_k}(I-\gamma P^{\pi_k})^{-1}\bregman{\pi}{\pi_k} \right\}, \label{eq: supp iterates uniform TRPO}
\end{align}
result in algorithm \ref{alg: uniform TRPO}. In Appendix~\ref{supp: Solution of TRPO iteration}, we derive the updates of the \emph{PolicyUpdate} procedure, Algorithms~\ref{alg: Policy Gradient policy update exact} and \ref{alg: NE-TRPO policy update exact}.
Then, we turn to analyze Uniform TRPO and its instances in Appendix~\ref{supp: Fundamental inequality for Uniform TRPO}. Specifically, we derive the fundamental inequality for Unifom TRPO, similarly to the fundamental inequality for Mirror Descent (\citealp{beck2017first}, Lemma-9.13). Although the objective is not convex, we show that due to the adaptive scaling, by applying the linear approximation of the value of regularized MDPs (Proposition \ref{prop: directional derivative for MDPs}), we can repeat similar derivation to that of MD, with some modifications. Finally, in Appendix~\ref{supp: proof of uniform TRPO}, we go on to prove convergence rates for both the unregularized ($\lambda=0$) and regularized ($\lambda>0$) versions of Uniform TRPO, using a right choice of stepsizes.

\subsection{Uniform TRPO Update Rule}\label{supp: uniform TRPO}

In each TRPO step, we solve the following optimization problem:
\begin{align*}
  {{\pi}_{k+1}}&\in \arg {{\min }_{\pi \in \ssimplex }}\left\{ \left\langle \nabla \vHpik,\pi-\pi_k \right\rangle +\frac{1}{t_k}(I-\gamma P^{\pi_k})^{-1}{B}_{\omega }\left( \pi,\pi_k \right) \right\} \\
  &\in \arg {{\min }_{\pi \in \ssimplex}}\left\{ (I-\gamma P^{\pi_k})^{-1}(T_{\lambda}^{\pi} \vHpik - \vHpik - \lambda \bregman{\pi}{\pi_k}) +\frac{1}{t_k}(I-\gamma P^{\pi_k})^{-1}{B}_{\omega }\left( \pi,\pi_k \right) \right\}  \\
   &\in \arg {{\min }_{\pi\in \ssimplex }} \brc*{ (I-\gamma P^{\pi_k})^{-1}(T_{\lambda}^{\pi} \vHpik - \vHpik + \br*{\frac{1}{t_k} - \lambda } \bregman{\pi}{\pi_k}) } \\
    &\in \arg {{\min }_{\pi\in \ssimplex }} \brc*{ T_{\lambda}^{\pi} \vHpik - \vHpik + \br*{\frac{1}{t_k} - \lambda} \bregman{\pi}{\pi_k} },
%   &=\arg {{\min }_{\pi }}\left\{ (I-\gamma P^{\pi_k})^{-1}\br*{t_k(I-\gamma P^{\pi_k})\left\langle \nabla \vHpik,\pi-\pi_k \right\rangle +{B}_{\omega }\left( \pi,\pi_k \right)} \right\}\\
%   &=\arg {{\min }_{\pi }}\left\{ t_k(I-\gamma P^{\pi_k})\left\langle \nabla \vHpik,\pi-\pi_k \right\rangle +{B}_{\omega }\left( \pi,\pi_k \right) \right\} \\
%   & = \arg {{\min }_{\pi }}\left\{ t_k(I-\gamma P^{\pi_k}) (I-\gamma P^{\pi_k})^{-1}\left\langle \nabla \vHpik,\pi-\pi_k \right\rangle +{B}_{\omega }\left( \pi,\pi_k \right) \right\},
% %  & =\arg {{\min }_{\pi }}\left\{ \left\langle \nabla f\left( {{x}_{k}} \right),x-{{x}_{k}} \right\rangle -\left\langle \nabla \omega\left( {{x}_{k}} \right),x \right\rangle +\omega\left( x \right) \right\}  
\end{align*}
where the second transition holds by plugging in the linear approximation (Proposition \ref{prop: directional derivative for MDPs}), and the last transition holds since $(I-\gamma P^{\pi_k})^{-1}> 0$ and does not depend on $\pi$. Thus, we have,
\begin{align}
      {{\pi}_{k+1}}&\in\arg {{\min }_{\pi \in \ssimplex }}\brc*{ t_k\br*{T_\lambda^{\pi}\vHpik - \vHpik} +(1- \lambda t_k) \bregman{\pi}{\pi_k}}  \label{eq: step of uniform TRPO full}
\end{align}

By discarding terms which do not depend on $\pi$, we get 
\begin{align}
    \pi_{k+1} \in \argmin_{\pi \in \ssimplex} \brc*{  t_k T_\lambda^\pi \vHpik  + \br*{1- \lambda t_k}\bregman{\pi}{\pi_k} } \label{eq: step of uniform polcy MD}
\end{align}

We are now ready to write \eqref{eq: uniform TRPO simplied update}, using the fact that \eqref{eq: step of uniform polcy MD}, can be written as the following state-wise optimization problem: For every $s\in \mathcal{S}$, 
\begin{align*}
    \pi_{k+1}(\cdot \mid s) \in \argmin_{\pi \in \simplex} \brc*{ t_k T_\lambda^\pi \vHpik (s)  + \br*{1- \lambda t_k} \bregmanS{s}{\pi}{\pi_k} } 
\end{align*}

\subsection{The \emph{PolicyUpdate} procedure}\label{supp: Solution of TRPO iteration}

Next, we write the solution for the optimization problem for each of the cases: 

By plugging Lemma~\ref{lemma: connection between bellman and q function} into \eqref{eq: step of uniform TRPO full}
\begin{align*}
    \pi_{k+1} \in \argmin_{\pi \in \ssimplex} \brc*{  t_k \inner*{\qHpik + \lambda \nabla \omega(\pi_k), \pi - \pi_k}  + \bregman{\pi}{\pi_k} } 
\end{align*}

Or again in a state-wise form,
\begin{align}\label{eq: towards closed form solution of uniform TRPO}
    \pi_{k+1}(\cdot \mid s) \in \argmin_{\pi \in \simplex} \brc*{  t_k \inner*{\qHpik(s,\cdot) + \lambda \nabla \omegaS{s}{\pi_k}, \pi - \pi_k(\cdot \mid s)}  + \bregmanS{s}{\pi}{\pi_k} } 
\end{align}

Using \eqref{eq: towards closed form solution of uniform TRPO}, we can plug in the solution of the MD iteration for each of the different cases.

\subsubsection{Euclidean Case:}
For $\omega$ chosen to be the $L_2$ norm, the solution to \eqref{eq: towards closed form solution of uniform TRPO} is the orthogonal projection. For all $s\in \mathcal{S}$ the policy is updated according to

\begin{align*}
     \pi_{k+1} ( \cdot | s ) &= P_{\simplex} \br*{ \pik(\cdot|s) - t_k \qHpik(s,\cdot) -\lambda t_k \pi_k(\cdot \mid s)} \\
     &=  P_{\simplex} \br*{ (1-\lambda t_k) \pik(\cdot|s) - t_k \qHpik(s,\cdot) } ,
\end{align*}
where $P_{\simplex}$ is the orthogonal projection operator over the simplex. Refer to \cite{beck2017first} for details.

Finally, dividing by the constant $1-\lambda t_k$ does not change the optimizer. Thus,
\begin{align}
    \ \pi_{k+1} ( \cdot | s ) = P_{\simplex} \left( \pik(\cdot|s) - \frac{t_k}{1-\lambda t_k} \qHpik(s,\cdot)\right), \label{eq: mirror descent iteration L2}
\end{align}

\subsubsection{Non-Euclidean Case:}
For $\omega$ chosen to be the negative entropy, \eqref{eq: towards closed form solution of uniform TRPO} has the following analytic solution for all $s\in \mathcal{S}$,
\begin{align*}
        \pi_{k+1}(\cdot \mid s) &\in \argmin_{\pi \in \simplex} \brc*{  t_k \inner*{\qHpik(s,\cdot) + \lambda \nabla H(\pi_k(\cdot \mid s), \pi - \pi_k(\cdot \mid s)}  + \dkl{\pi(\cdot \mid s)}{\pi_k(\cdot \mid s)} } \\
        &\in \argmin_{\pi \in \simplex} \brc*{  \inner*{t_k\qHpik(s,\cdot) - (1 - \lambda t_k ) \nabla H(\pi_k(\cdot \mid s), \pi - \pi_k(\cdot \mid s)}  + H(\pi(\cdot \mid s)) - H_k(\pi(\cdot \mid s)) } \\
        &\in \argmin_{\pi \in \simplex} \brc*{  \inner*{t_k\qHpik(s,\cdot) - (1 - \lambda t_k ) \nabla H(\pi_k(\cdot \mid s), \pi}  + H(\pi(\cdot \mid s))}
        % & \in \argmin_{\pi \in \simplex} \brc*{  t_k \inner*{\qHpik(s,\cdot) + \lambda \nabla H(\pi_k(\cdot \mid s), \pi }  + \dkl{\pi(\cdot \mid s)}{\pi_k(\cdot \mid s)} } ,
\end{align*}
where the first transition is by substituting $\omega$ and the Bregman distance, the second is by the definition of the Bregman distance, and the last transition is by omitting constant factors.

By using (\citealp{beck2017first}, Example 3.71), we get

\begin{align*}
    {{\pi }_{k+1}}\left( a|s \right)=\frac{{{ {{\pi }_{k}}\left( a\mid s \right)}}{{e}^{-t_k q_\lambda^{\pi_k}(s,a) -\lambda t_k \nabla_{\pi_k(a\mid s)}H(\pi_k(\cdot \mid s)) }}}{\sum\nolimits_{a'}{\pi_k}\br*{a'\mid s }{{e}^{-t_k q_\lambda^{\pi_k}(s,a') -\lambda t_k \nabla_{\pi_k(a'\mid s)}H(\pi_k(\cdot \mid s))}}}.
\end{align*}

Now, using the derivative of the negative entropy function $H(\cdot)$, we have that for every $s,a$,

\begin{align}\label{eq: mirror descent iteration KL}
    {{\pi }_{k+1}}\left( a|s \right)=\frac{{{ {{\pi }_{k}}\left( a\mid s \right)}}{{e}^{-t_k \br*{ q_\lambda^{\pi_k}(s,a) -\lambda \log \pi_k(a \mid s)} }}}{\sum\nolimits_{a'}{\pi_k}\br*{a'\mid s }{{e}^{-t_k \br*{ q_\lambda^{\pi_k}(s,a') -\lambda \log\pi_k(a' \mid s)}}}},
\end{align}
which concludes the result.

% $$ e^{- \lambda t_k \nabla_{\pi_k(a\mid s)}H(\pi_k(\cdot \mid s))} = e^{-(1 + \lambda t_k \log \pi_k (a \mid s) ) } =  {\pi_k (a \mid s)}^{-\lambda t_k} e^{-1} $$

% Plugging back to \eqref{eq: mirror descent iteration KL temp 1}, we have that

% \begin{align}\label{eq: mirror descent iteration KL temp 2}
%     {{\pi }_{k+1}}\left( a|s \right)=\frac{{{ {{\pi }_{k}}\left( a\mid s \right)}}{{e}^{-t_k q_\lambda^{\pi_k}(s,a) -\lambda t_k \nabla_{\pi_k(a\mid s)}H(\pi_k(\cdot \mid s)) }}}{\sum\nolimits_{a'}{\pi_k}\br*{a'\mid s }{{e}^{-t_k q_\lambda^{\pi_k}(s,a') -\lambda t_k \nabla_{\pi_k(a'\mid s)}H(\pi_k(\cdot \mid s))}}}.
% \end{align}

% Finally \todo{asdas}
% \begin{align}\label{eq: mirror descent iteration KL}
%     {{\pi }_{k+1}}\left( a|s \right)=\frac{ \pi_k( a\mid s)^{1-\lambda t_k }{{e}^{-t_k q_\lambda^{\pi_k}(s,a) }}}{\sum\nolimits_{a'} \pi_k( a'\mid s)^{1-\lambda t_k}{{e}^{-t_k q_\lambda^{\pi_k}(s,a')}}}.
% \end{align}

\subsection{Fundamental Inequality for Uniform TRPO}\label{supp: Fundamental inequality for Uniform TRPO}
% First, we need to prove some lemmas that will be used to prove the convergence of the proposed algorithms.

% Next, we prove a simple bound for the KL distance,
% \begin{lemma}\label{lemma: bound on the KL distance}
% For both the euclidean and the non-euclidean case, the following holds:
% \begin{align*}
%     \bregman{\cdot}{\cdot} \leq \dkl{\cdot}{\cdot}
% \end{align*}
% \begin{proof}
% Now, we show that for both the euclidean ($\omega=\frac{1}{2}\norm{\cdot}_2^2$) and the non-euclidean ($\omega=H$) cases, $-\dkl{\cdot}{\cdot}\leq - \bregman{\cdot}{\cdot} $.
% For the non-euclidean case it is trivial, as $B_\omega=d_{KL}$. In the euclidean case:
% \begin{align*}
% -\dkl{\cdot}{\cdot}\leq -2\norm{\cdot}_1^2\leq -2 \norm{\cdot}_2^2 = -4\bregman{\cdot}{\cdot}\leq -\bregman{\cdot}{\cdot},
% \end{align*}
% where the first transition is due to Pinkser's inequality, and the second is due to the norm equivalence.
% \end{proof}
% \end{lemma}

Central to the following analysis is Lemma \ref{lemma: fundamental inequality uniform regularized}, which we prove in this section. This lemma replaces Lemma~\cite{beck2017first}[9.13] from which it inherits its name, for the RL non-convex case. It has two main differences relatively to Lemma~\cite{beck2017first}[9.13]: (a) The inequality is in vector form (statewise). (b) The non-convexity of $f$ demands replacing the gradient inequality with different proof mechanism, i.e., the directional derivative in RL (see Proposition \ref{prop: directional derivative for MDPs}).

\begin{lemma}[fundamental inequality for Uniform TRPO]\label{lemma: fundamental inequality uniform regularized} 
Let $\{\pi_k\}_{k\geq 0}$ be the sequence generated by the uniform TRPO method with stepsizes $\{t_k\}_{k\geq0}$. Then, for every $\pi$ and $k\geq0$, 
\begin{align*}
    & t_k (I-\gamma P^\pi) \left(\vHpik-\vHpi\right) \\
    &\leq (1-\lambda t_k)\bregman{\pi}{\pi_k} - \bregman{\pi}{\pi_{k+1}} +\lambda t_k(\omega(\pik) - \omega(\pi_{k+1})) + \frac{ t_k^2\gradBound^2}{2}e,
\end{align*}
where $\gradBound$ is defined in the second claim of Lemma \ref{lemma: usefull bounds exact case}, and $e$ is a vector of ones.
\end{lemma}

\begin{proof}

First, notice that assumptions~\ref{assumption: Bregman Distance} and \ref{assumption: main} hold. Assumption~\ref{assumption: Bregman Distance} is a regular assumption on the Bregman distance, which holds trivially both in the euclidean and non-euclidean case, where the optimization domain is the $\ssimplex$. Assumption~\ref{assumption: main} deals with the optimization problem itself and is similar to (\citealp{beck2017first}, Assumption 9.1) over $\simplex$. The only difference is that in our case, the optimization objective $v^\pi$ is non-convex.

Define
$\psi (\pi) \equiv t_k (I-\gamma P^{\pi_k})\langle \nabla \vHpik ,\pi \rangle + \delta_{\ssimplex}(\pi )$ where $\delta_{\ssimplex}(\pi )=0$ when $\pi\in\ssimplex$ and infinite otherwise. Observe it is a convex function in $\pi$, as a sum of two convex functions: The first term is linear in $\pi$ for any $\pi\in \ssimplex$, and thus convex, and $\delta_{\ssimplex}(\pi )$ is convex since $\ssimplex$ is a convex set. Applying the non-euclidean second prox theorem (Theorem \ref{theorem: second prox}), with $a=\pik$, $b=\pi_{k+1}$, we get that for any $\pi\in \ssimplex$,
\begin{align}\label{eq: fundamental vector 1}
\langle \nabla\omega(\pik) - \nabla\omega(\pi_{k+1}), \pi - \pi_{k+1} \rangle \leq t_k(I-\gamma P^{\pi_k}) \langle \nabla \vHpik, \pi-\pi_{k+1} \rangle
\end{align}
By the three-points lemma (\ref{lemma: three-points}),
\begin{align*}
\langle \nabla\omega(\pik) - \nabla\omega(\pi_{k+1}), \pi - \pi_{k+1} \rangle = \bregman{\pi}{\pi_{k+1}}+\bregman{\pi_{k+1}}{\pi_k}-\bregman{\pi}{\pik},
\end{align*}
which, combined with \eqref{eq: fundamental vector 1}, gives,
\begin{align*}
\bregman{\pi}{\pi_{k+1}}+\bregman{\pi_{k+1}}{\pi_k}-\bregman{\pi}{\pik}  \leq t_k(I-\gamma P^{\pi_k}) \langle \nabla \vHpik, \pi-\pi_{k+1} \rangle.
\end{align*}

Therefore, by simple algebraic mainpulation, we get
\begin{align}
&t_k (I-\gamma P^{\pi_k})\langle \nabla \vHpik, \pik-\pi \rangle \nonumber \\
&\leq \bregman{\pi}{\pik}-\bregman{\pi}{\pi_{k+1}} - \bregman{\pi_{k+1}}{\pi_k} + t_k(I-\gamma P^{\pi_k}) \langle \nabla \vHpik, \pik-\pi_{k+1} \rangle \nonumber\\
&= \bregman{\pi}{\pik}-\bregman{\pi}{\pi_{k+1}} - \bregman{\pi_{k+1}}{\pi_k}  + t_k \left(T^{\pik}_\lambda \vHpik - T^{\pi_{k+1}}_\lambda \vHpik \right)+\lambda t_k \bregman{\pi_{k+1}}{\pik}, \label{eq: needed relation for improvement}
\end{align}
where the last equality is due to Proposition \ref{prop: directional derivative for MDPs}, and using $(I-\gamma P^{\pi_k})(I-\gamma P^{\pi_k})^{-1} = I.$

Rearranging we get 
\begin{align}\label{eq: fundamental vector 2}
& t_k (I-\gamma P^{\pik}) \langle \nabla \vHpik, \pik-\pi \rangle   \nonumber \\
& \leq \bregman{\pi}{\pik}-\bregman{\pi}{\pi_{k+1}} - (1-\lambda t_k)\bregman{\pi_{k+1}}{\pi_k}  + t_k \left(T^{\pik}_\lambda \vHpik - T^{\pi_{k+1}}_\lambda \vHpik \right)  \nonumber\\
& \leq \bregman{\pi}{\pik}-\bregman{\pi}{\pi_{k+1}} - \frac{1-\lambda t_k}{2}\norm{\pi_{k+1}-\pi_k}^2  + t_k \left(T^{\pik}_\lambda \vHpik - T^{\pi_{k+1}}_\lambda \vHpik \right)  ,
\end{align}
where the last inequality follows since the Bregman distance is $1$-strongly-convex  for our choices of $B_\omega$ (e.g., \citealp{beck2017first}, Lemma 9.4(a)).

Furthermore, for every state $s\in\mathcal{S}$,
\begin{align*}
    & t_k\left(T^{\pik}_\lambda \vHpik - T^{\pi_{k+1}}_\lambda \vHpik\right)(s) \\
    & = t_k \lambda(\omegaS{s}{\pik} - \omegaS{s}{\pi_{k+1}}) \\
    &+\sum_a  t_k\left(\pik(a|s) - \pi_{k+1} (a|s)\right)(c(s,a)+\gamma \sum_{s'} p(s'|s,a)\vHpik(s')) \\
    & = t_k\lambda(\omegaS{s}{\pik} - \omegaS{s}{\pi_{k+1}}) \\
    & + \left\langle \frac{t_k}{\sqrt{1-\lambda t_k}}(c(s,\cdot)+\gamma \sum_{s'} p(s'|s,\cdot)\vHpik(s')),\sqrt{1-\lambda t_k}(\pik(\cdot|s) - \pi_{k+1} (\cdot|s)) \right\rangle \\
    & \leq \lambda t_k(\omegaS{s}{\pik} - \omegaS{s}{\pi_{k+1}})) \\
    & + \frac{1-\lambda t_k}{2}\norm{\pi_{k+1} - \pik}^2 + \frac{t_k^2}{2(1-\lambda t_k)}\norm{c(s,\cdot)+\gamma \sum_{s'} p(s'|s,\cdot)\vHpik(s')}_*^2 \\
    & \leq  \lambda t_k(\omegaS{s}{\pik} - \omegaS{s}{\pi_{k+1}})+\frac{1-\lambda t_k}{2}\norm{\pi_{k+1} - \pik}^2 + \frac{t_k^2 \gradBound^2}{2(1-\lambda t_k)},
\end{align*}
where the first inequality is due to the Fenchel's inequality on the convex  $\norm{\cdot}^2$ and its convex conjugate $\norm{\cdot}_*^2$, and the last equality uses the fact that $\norm{c(s,\cdot)+\gamma \sum_{s'} p(s'|s,\cdot)\vHpik(s')}_*\leq \norm{c_\lambda(s,\cdot)+\gamma \sum_{s'} p(s'|s,\cdot)\vHpik(s')}_*=\norm{\qHpik\br*{s,\cdot}}_*$, and using the repsective bound in Lemma~\ref{lemma: usefull bounds exact case}.

Plugging the last inequality into \eqref{eq: fundamental vector 2},

\begin{align*}
&t_k(I-\gamma P^{\pik})  \langle \nabla \vHpik, \pik-\pi \rangle \leq  \lambda t_k(\omega(\pik) - \omega(\pi_{k+1}))+ \bregman{\pi}{\pik}-\bregman{\pi}{\pi_{k+1}} + \frac{t_k^2 \gradBound^2}{2(1-\lambda t_k)}e,
\end{align*}
where $e$ is a vector of all ones.

By using Proposition \ref{prop: directional derivative for MDPs} on the LHS, we get,
\begin{align*}
&-t_k (T^{\pi}v^{\pi_k} - v^{\pi_k} - \lambda \bregman{\pi}{\pi_k})  \leq  \lambda t_k(\omega(\pik) - \omega(\pi_{k+1}))+ \bregman{\pi}{\pik}-\bregman{\pi}{\pi_{k+1}} + \frac{t_k^2 \gradBound^2}{2(1-\lambda t_k)}e \\
\iff & -t_k (T^{\pi}v^{\pi_k} - v^{\pi_k})  \leq  \lambda t_k(\omega(\pik) - \omega(\pi_{k+1}))+ (1-\lambda t_k)\bregman{\pi}{\pik}-\bregman{\pi}{\pi_{k+1}} + \frac{t_k^2 \gradBound^2}{2(1-\lambda t_k)}e.
\end{align*}

Lastly, 
\begin{align*}
    & t_k (I-\gamma P^\pi) \left(\vHpik-\vHpi\right) = -t_k (T^{\pi}v^{\pi_k} - v^{\pi_k}) \\
    & \leq (1-\lambda t_k) \bregman{\pi}{\pi_k} - \bregman{\pi}{\pi_{k+1}} +\lambda t_k(\omega(\pik) - \omega(\pi_{k+1}))+ \frac{t_k^2 \gradBound^2}{2(1-\lambda t_k)} e,
\end{align*}
where the first relation holds by the second claim in Lemma \ref{lemma: greedy to difference}.
\end{proof}

\subsection{Proof of Theorem~\ref{theorem: uniform TRPO}}\label{supp: proof of uniform TRPO}

Before proving the theorem, we establish that the policy improves in $k$ for the chosen learning rates.
\begin{lemma}[Uniform TRPO Policy Improvement]\label{lemma: policy improvement uniform}
Let $\{\pik\}_{k\geq 0}$ be the sequence generated by Uniform TRPO. Then, for both the euclidean and non-euclidean versions of the algorithm, for any $\lambda\geq0$, the value improves for all $k$, $$v_\lambda^{\pi_k} \geq v_\lambda^{\pi_{k+1}}.$$
\end{lemma}
\begin{proof}
Restating \eqref{eq: needed relation for improvement}, we have that for any $\pi$,
\begin{align*}
&t_k (I-\gamma P^{\pi_k})\langle \nabla \vHpik, \pik-\pi \rangle \\
&\leq  \bregman{\pi}{\pik}-\bregman{\pi}{\pi_{k+1}} - \bregman{\pi_{k+1}}{\pi_k}  + t_k \left(T^{\pik}_\lambda \vHpik - T^{\pi_{k+1}}_\lambda \vHpik \right)+\lambda t_k \bregman{\pi_{k+1}}{\pik}.
\end{align*}

Plugging the closed form of the directional derivative (Proposition~\eqref{prop: directional derivative for MDPs}), setting $\pi=\pi_k$, using $\bregman{\pik}{\pik}=0$,  we get,
\begin{align}
&t_k \left(T^{\pik}_\lambda \vHpik - T^{\pi_{k+1}}_\lambda \vHpik \right) \geq \bregman{\pi_k}{\pi_{k+1}} + \bregman{\pi_{k+1}}{\pi_k}(1-\lambda t_k). \label{eq: improvement first relation uniform}
\end{align}

The choice of the learning rate and the fact that the Bregman distance is non negative ($\lambda >0$, $\lambda t_k = \frac{1}{k+2} \leq 1$ and for $\lambda=0$ the RHS of~\eqref{eq: improvement first relation uniform} is positive) implies that
\begin{align}
     &\vHpik - T^{\pi_{k+1}}_\lambda \vHpik = \left(T^{\pik}_\lambda \vHpik - T^{\pi_{k+1}}_\lambda \vHpik \right)  \geq 0 \nonumber \\
     & \rightarrow \vHpik \geq T^{\pi_{k+1}}_\lambda \vHpik. \label{eq: improvement second relation uniform}
\end{align}
Applying iteratively $T^{\pi_{k+1}}_\lambda$ and using its monotonicty we obtain,
\begin{align*}
    \vHpik  \geq T^{\pi_{k+1}}_\lambda \vHpik \geq (T^{\pi_{k+1}}_\lambda)^2\vHpik\geq \cdot\cdot\cdot \geq \lim_{n\rightarrow \infty} (T^{\pi_{k+1}}_\lambda)^n \vHpik = v^{\pi_{k+1}}_\lambda,
\end{align*}
where in the last relation we used the fact $T^{\pi_{k+1}}_\lambda$ is a contraction operator and its fixed point is $v^{\pi_{k+1}}_\lambda$ which proves the claim.
\end{proof}

For the sake of completeness and readability, we restate here Theorem~\ref{theorem: uniform TRPO}, this time including all logarithmic factors:

\begin{theorem*}[Convergence Rate: Uniform TRPO]
Let $\{\pik\}_{k\geq 0}$ be the sequence generated by Uniform TRPO, 

Then, the following holds for all $N\geq1$.
\begin{enumerate}
    \item \emph{(Unregularized)} Let $\lambda=0$, $t_k=\frac{(1-\gamma)}{\euclid \mathrm{C_{\text{max}}} \sqrt{k+1}}$ then
    \begin{align*}
       \norm{v^{\pi_N} - \vstar}_\infty   \leq 
        O\br*{\frac{\euclid \mathrm{C_{\text{max}}}(\bregmanMax + \log N)}{(1-\gamma)^2\sqrt{N}} }    \end{align*}
    
    \item \emph{(Regularized)} Let $\lambda>0$, $t_k =  \frac{1}{\lambda (k+2)}$ then
\begin{align*}
   \norm{v_\lambda^{\pi_N} - \vHstar}_\infty   \leq  O \br*{ \frac{\euclid ^2\CmaxLambda^2 \log N}{\lambda(1
    -\gamma)^3 N} }.
\end{align*}
\end{enumerate}
Where $\euclid =\sqrt{A}, \bregmanMax=1$ for the euclidean case, and $\euclid =1, \bregmanMax=\log A$ for the non-euclidean case.
\end{theorem*}

We are now ready to prove Theorem~\ref{theorem: uniform TRPO}, while following arguments from (\citealp{beck2017first}, Theorem 9.18).

\subsubsection{The Unregularized case}
\begin{proof}
Applying Lemma~\ref{lemma: fundamental inequality uniform regularized} with $\pi=\pistar$ and $\lambda=0$ (the unregularized case) and let $e\in {\R}^{\sset}$, a vector ones, the following relations hold.
\begin{align}
  & t_k (I-\gamma P^{\pistar}) \left( \vpik - \vstar \right)  \leq \bregman{\pistar}{\pi_k}- \bregman{\pistar}{\pi_{k+1}} +\frac{t_k^2 \gradBound^2}{2} e
\end{align}

Summing the above inequality over $k=0,1,...,N$, and noticing we get a telescopic sum gives
\begin{align*}
    \sum\limits_{k=0}^{N}  t_k (I-\gamma P^{\pistar}) \left( \vpik - \vstar \right)  &\leq \bregman{\pistar}{\pi_0} - \bregman{\pistar}{\pi_{N+1}}+\sum\limits_{k=0}^{N}\frac{ t_k^2\gradBound^2}{2}e\\
    &\leq \bregman{\pistar}{\pi_0} +\sum\limits_{k=0}^{N}\frac{ t_k^2\gradBound^2}{2}e\\
    &\leq \norm{\bregman{\pistar}{\pi_0}}_\infty e +\sum\limits_{k=0}^{N}\frac{ t_k^2\gradBound^2}{2}e
\end{align*}
where the second relation holds since $\bregman{\pistar}{\pi_{N+1}}\geq 0$ component-wise. From which we get the following relations,
\begin{align}
    &(I-\gamma P^{\pi^*})\sum\limits_{k=0}^{N}  t_k \left( \vpik - \vstar \right) \leq \norm{\bregman{\pistar}{\pi_0}}_\infty e +\sum\limits_{k=0}^{N}\frac{ t_k^2\gradBound^2}{2}e \nonumber\\
    \iff& \sum\limits_{k=0}^{N}  t_k \left( \vpik - \vstar \right) \leq (I-\gamma P^{\pi^*})^{-1}\br*{ \norm{\bregman{\pistar}{\pi_0}}_\infty e +\sum\limits_{k=0}^{N}\frac{ t_k^2\gradBound^2}{2}e} \nonumber \\
    \iff& \sum\limits_{k=0}^{N}  t_k \left( \vpik - \vstar \right) \leq \frac{\norm{\bregman{\pistar}{\pi_0}}_\infty}{1-\gamma}e +\sum\limits_{k=0}^{N}\frac{ t_k^2\gradBound^2}{2(1-\gamma)}e.\label{eq: unfirm non regularized almost one}
\end{align}
In the second relation we multiplied both sides of inequality by $(I-\gamma P^{\pi^*})^{-1}\geq 0$ component-wise. In the third relation we used $(I-\gamma P^{\pi})^{-1}e =\frac{1}{1-\gamma}e$ for any $\pi$. By Lemma~\eqref{lemma: policy improvement uniform} the policies are improving, from which, we get
\begin{align}
    \left( v_\lambda^{\pi_N} - \vstar \right) \sum\limits_{k=0}^{N}  t_k \leq  \sum\limits_{k=0}^{N}  t_k \left( \vpik - \vstar \right). \label{eq: unfirm non regularized almost two}
\end{align}
Combining \eqref{eq: unfirm non regularized almost one}, \eqref{eq: unfirm non regularized almost two} , and dividing by $\sum\limits_{k=0}^{N}  t_k$ we get the following component-wise inequality,
\begin{align*}
    v_\lambda^{\pi_N} - \vstar \leq \frac{\norm{\bregman{\pistar}{\pi_0}}_\infty +\frac{ \gradBound^2}{2}\sum\limits_{k=0}^{N}t_k^2}{(1-\gamma)\sum\limits_{k=0}^{N}  t_k }e
\end{align*}

% \begin{align*}
%     &\sum\limits_{k=0}^{N}  t_k \norm{ \vpik - \vstar }_{\infty}e \leq \frac{\bregman{\pistar}{\pi_0}}{1-\gamma}+\sum\limits_{k=0}^{N}\frac{ t_k^2\gradBound^2}{2}e\\
%     \iff &\sum\limits_{k=0}^{N}  t_k \norm{ \vpik - \vstar }_{\infty}\leq \frac{\norm{\bregman{\pistar}{\pi_0}}_\infty}{1-\gamma}+\sum\limits_{k=0}^{N}\frac{ t_k^2\gradBound^2}{2}
% \end{align*}

% Using the definition of $v_{\textbf{best}}^N$, we obtain 
% \begin{align*}
%     \norm{ v_{\textbf{best}}^N - \vstar }_{\infty}\sum\limits_{k=0}^{N}  t_k \leq \sum\limits_{k=0}^{N}  t_k \norm{ \vpik - \vstar }_{\infty},
% \end{align*}
% by which we get,
% \begin{align*}
%     &\norm{v_{\textbf{best}}^N - \vstar}_\infty \leq \frac{\norm{\bregman{\pistar}{\pi_0}}_\infty +\frac{ \gradBound^2}{2}\sum\limits_{k=0}^{N}t_k^2}{(1-\gamma)\sum\limits_{k=0}^{N}  t_k }
% \end{align*}

By plugging in the stepsizes, $t_k=\frac{1}{\gradBound\sqrt{k+1}}$ we get, 
\begin{align*}
    v_\lambda^{\pi_N} - \vstar   \leq 
    O\br*{\frac{\gradBound}{1-\gamma}\frac{\norm{\bregman{\pi^*}{\pi_0}}_\infty + \sum\limits_{k=0}^{N}\frac{1}{k+1}}{\sum\limits_{k=0}^{N} \frac{1}{\sqrt{k+1}}}e}
\end{align*}

Plugging in Lemma~\ref{lemma: bounds on Dw} and bounding the sums (e.g., by using \citealp{beck2017first}, Lemma 8.27(a)) yields,
\begin{align*}
    v_\lambda^{\pi_N} - \vstar    \leq 
    O\br*{\frac{\gradBound}{1-\gamma}\frac{\omegaBound + \log N}{\sqrt{N}}e}.
\end{align*}

Plugging the expressions for $\gradBound, \omegaBound$ in Lemma \ref{lemma: usefull bounds exact case} and Lemma \ref{lemma: bounds on Dw} we conclude the proof.
\end{proof}

\subsubsection{The Regularized case}

\begin{proof}

Applying Lemma~\ref{lemma: fundamental inequality uniform regularized} with $\pi=\pistar$ and $\lambda>0$,
\begin{align*}
  & t_k (I-\gamma P^{\pistar}) \left( \vpik_\lambda - \vstar_\lambda \right)  \\
  & \leq (1-\lambda t_k )\bregman{\pistar}{\pi_k}- \bregman{\pistar}{\pi_{k+1}} +\lambda t_k(\omega(\pik) - \omega(\pi_{k+1}))+\frac{t_k^2\gradBound^2}{2(1-\lambda t_k)} e .
\end{align*}

Plugging $t_k=\frac{1}{\lambda (k+2)}$ and multiplying by $\lambda(k+2)$,

\begin{align*}
  &  (I-\gamma P^{\pistar}) \left( \vpik_\lambda - \vstar_\lambda \right)  \leq \lambda(k+1) \bregman{\pistar}{\pi_k}-  \lambda(k+2)\bregman{\pistar}{\pi_{k+1}} + \lambda\omega(\pik) - \lambda\omega(\pi_{k+1}) 
  +\frac{\gradBound^2}{2\lambda}\frac{1}{k+1} e .
\end{align*}

Summing the above inequality over $k=0,...,N$ yields
\begin{align*}
    &\sum\limits_{k=0}^{N}  (I-\gamma P^{\pistar}) \left( \vpik_\lambda - \vstar_\lambda \right) \\
    &\leq \lambda \bregman{\pistar}{\pi_0} - \lambda (N+3) \bregman{\pistar}{\pi_{N+1}}+ \lambda \omega(\pi_2) - \lambda \omega(\pi_{N+1})+\frac{\gradBound^2}{2\lambda}e\sum\limits_{k=0}^{N} \frac{1}{k+1},
\end{align*}
as the summation results in a telescopic sum.

Observe that for any $\pi,\pi'$ and both our choices of $\omega$, $\omega(\pi)-\omega(\pi')\leq \max_{\pi}|\omega(\pi)|$. For the euclidean case $\max_{\pi}|\omega(\pi)|< 1$ and for the non euclidean case $\max_{\pi}|\omega(\pi)|\leq \log A$. These bounds are the same bounds as the bound for the Bregman distance, $D_\omega$ (see Lemma \ref{lemma: bounds on Dw}). Thus, for both our choices of $\omega$ we can bound $\omega(\pi)-\omega(\pi')<D_\omega$.

Furthermore, since $\bregman{\pistar}{\pi_{N+1}}\geq 0$ the following bound holds:

\begin{align}
&\sum\limits_{k=0}^{N} (I-\gamma P^{\pistar})\left( \vpik_\lambda - \vstar_\lambda \right) \leq  2\lambda \omegaBound e + \frac{\gradBound^2}{2\lambda}e\sum\limits_{k=1}^{N} \frac{1}{k+1} \nonumber \\
    \iff &(I-\gamma P^{\pistar})\sum\limits_{k=0}^{N} \left( \vpik_\lambda - \vstar_\lambda \right) \leq  2\lambda \omegaBound e+ \frac{\gradBound^2}{2\lambda}e\sum\limits_{k=1}^{N} \frac{1}{k+1} \nonumber \\
    \iff &\sum\limits_{k=0}^{N} \left( \vpik_\lambda - \vstar_\lambda \right) \leq  \frac{2\lambda \omegaBound }{1-\gamma}e+ \frac{\gradBound^2}{2\lambda (1-\gamma)}e\sum\limits_{k=1}^{N} \frac{1}{k+1}, \label{eq: uniform regularized proof first relation}
\end{align}
and in the third relation we multiplied both side by $(I-\gamma P^{\pi^*})^{-1}\geq 0$ component-wise and used $(I-\gamma P^{\pi})^{-1}e = \frac{1}{1-\gamma}e$ for any $\pi$.

By Lemma~\ref{lemma: policy improvement uniform} the value $v^{\pi_k}_\lambda$ decreases in $k$, and, thus,
\begin{align}
    (N+1)(v^{\pi_N}_\lambda - v^*_\lambda)\leq \sum\limits_{k=0}^{N} \left( \vpik_\lambda - \vstar_\lambda \right). \label{eq: uniform regularized proof second relation}
\end{align}

Combining \eqref{eq: uniform regularized proof first relation}, \eqref{eq: uniform regularized proof second relation}, and dividing by $N+1$ we get the following component-wise inequality,

\begin{align*}
    &v^{\pi_N}_\lambda - v^*_\lambda \leq  \br*{\frac{2\lambda \omegaBound}{(1
    -\gamma)(N+1)} +\frac{\gradBound^2}{2\lambda(1-\gamma)(N+1)}\sum\limits_{k=1}^{N+1} \frac{1}{k}}e
\end{align*}

Using the fact that $\sum\limits_{k=1}^{N+1} \frac{1}{k} \in O(\log{n})$, we get
\begin{align*}
   v^{\pi_N}_\lambda - v^*_\lambda  \leq  O \br*{ \frac{\lambda^2 \omegaBound + \gradBound^2 \log N}{\lambda(1
    -\gamma)N} e}.
\end{align*}

Plugging the expressions for $\gradBound, \omegaBound$ in Lemma \ref{lemma: usefull bounds exact case} and Lemma \ref{lemma: bounds on Dw} we conclude the proof.
\end{proof}

\section{Exact Trust Region Policy Optimization}\label{supp: mu exact policy mirror descenet}

The derivation of Exact TRPO is similar in spirit to the derivation of Uniform TRPO (Appendix~\ref{supp: uniform TRPO proof}). However, instead of minimizing a vector, the objective to be minimized in this section is the scalar $\mu v^\pi$ \eqref{eq: scalar objective}. This fact complicates the analysis and requires us assuming a finite concentrability coefficient $ C^{\pi_*} = \norm{\frac{d_{\mu,\pi^*}}{\nu}}_\infty < \infty$ (Assumption~\ref{assumption: stochastic MD}), a common assumption in the RL literature \cite{kakade2002approximately,farahmand2010error,scherrer2014approximate,scherrer2014local}. This assumption alleviates the need to deal with exploration and allows us to focus on the optimization problem in MDPs in which the stochasticity of the dynamics induces sufficient exploration. We note that assuming a finite $C^{\pi_*}$ is the weakest assumptions among all other existing concentrability coefficients~\cite{scherrer2014approximate}. 

The Exact TRPO algorithm is as follows:
\begin{center}
% \begin{minipage}{0.7\textwidth}
\begin{algorithm}[H]
\caption{Exact TRPO}\label{alg: mu exact TRPO}
\begin{algorithmic}
\REQUIRE $t_k$, $\gamma$, $\lambda$, $\pi_0$ is the uniform policy.
\FOR{$k=0,1,...$}
    \STATE $v^{\pi_k} \gets \mu(I-\gamma P^{\pi_k})^{-1} c_\lambda^{\pi_k}$
    \STATE $\mathcal{S}_{d_{\nu,\pi_k}} = \brc*{ s\in\mathcal{S}:  d_{\nu,\pi_k}(s)> 0}$
    \FOR{ $\forall s\in\mathcal{S}_{d_{\nu,\pi_k}}$}
        \FOR{$\forall a\in\aset$}
            \STATE $q^{\pi_k}_\lambda(s,a) \gets c^\pi_\lambda(s,a) + \gamma \sum_{s'} p(s'|s,a) v_\lambda^{\pi_k}(s')$  
        \ENDFOR
    \STATE $\pi_{k+1}(\cdot|s) = $PolicyUpdate($\pi_{k}(\cdot|s),q^{\pi_k}_\lambda(s,\cdot), t_k,\lambda$)
    \ENDFOR
\ENDFOR
% \STATE {\bf Return:} $\pi_k$
\end{algorithmic}
\end{algorithm}
% \end{minipage}
\end{center}
Similarly to Uniform TRPO, the euclidean and non-euclidean
choices of $\omega$ correspond to a PPG and NE-TRPO instances of Exact TRPO: by instantiating PolicyUpdate
with the subroutines \ref{alg: Policy Gradient policy update exact} or \ref{alg: NE-TRPO policy update exact} we get the instances of Exact TRPO respectively. A

The main goal of this section is to create the infrastructure for the analysis of Sample-Based TRPO, which is found in Appendix~\ref{supp: approximate Mirror descent}. Sample-Based TRPO is a sample-based version of Exact TRPO, and for pedagogical reasons we start by analyzing the latter from which the analysis of the first is better motivated.

In this section we prove convergence for Exact TRPO which establishes similar convergence rates as for the Uniform TRPO in the previous section. We now describe the content of each of the subsections: First, in Appendix~\ref{supp: Uniform to Exact}, we show the connection between Exact TRPO and Uniform TRPO by proving Proposition~\ref{proposition: exact to approximate}. In Appendix~\ref{supp: exact TRPO update rule}, we formalize the exact version of TRPO. Then, we derive a fundamental inequality that will be used to prove convergence for the exact algorithms (Appendix \ref{supp: mu exact TRPO}). This inequality is a scalar version of the vector fundamental inequality derived for Uniform TRPO (Lemma~\ref{lemma: fundamental inequality uniform regularized}). This is done by first deriving a state-wise inequality, and then using Assumption~\ref{assumption: stochastic MD} to connect the state-wise local guarantee to a global guarantee w.r.t. the optimal policy $\pi^*$. Finally, we use the fundamental inequality for Exact TRPO to prove the convergence rates of Exact TRPO for both the unregularized and regularized version (Appendix~\ref{supp: exact TRPO unregularized proof}).

\subsection{Relation Between Uniform and Exact TRPO}\label{supp: Uniform to Exact}

Before diving into the proof of Exact TRPO, we prove Proposition~\ref{proposition: uniform to exact update}, which connects the update rules for Uniform and Exact TRPO:
\UniformToExactUpdate*
\begin{proof}
First, notice that for every $s'$
\begin{align*}
    \nu \inner*{\nabla_{\pik(\cdot | s')} \vHpik, \pi-\pik} & = \sum_s \nu(s)  \inner*{\nabla_{\pik(\cdot | s')} \vHpik(s), \pi(\cdot\mid s')-\pik(\cdot\mid s')} \nonumber \\
    & =  \inner*{\sum_s \nabla_{\pik(\cdot | s')} \nu(s)\vHpik(s), \pi(\cdot\mid s')-\pik(\cdot\mid s')} \nonumber \\
    & =  \inner*{\nabla_{\pik(\cdot | s')} \sum_s \nu(s)\vHpik(s), \pi(\cdot\mid s')-\pik(\cdot\mid s')} \nonumber \\
    & = \inner*{\nabla_{\pik(\cdot | s')} \nu\vHpik, \pi(\cdot\mid s')-\pik(\cdot\mid s')},
\end{align*}
where in the second and third transition we used the linearity of the inner product and the derivative, and in the last transition we used the definition of $\nu \vHpik$.

Thus, we have,
\begin{align}\label{eq: switching the location of mu in the derivative}
     \nu \inner*{\nabla \vHpik, \pi-\pik} & = \inner*{\nabla \nu\vHpik, \pi-\pik}.
\end{align}

\begin{align*}
    \nu\br*{\inner*{ \nabla \vHpik, \pi - \pi_k } + \frac{1}{t_k}(I-\gamma P^{\pi_k})^{-1} \bregman{\pi}{\pi_k}}
    & = \br*{\nu \inner*{ \nabla \vHpik, \pi - \pi_k } + \frac{1}{t_k}\nu (I-\gamma P^{\pi_k})^{-1} \bregman{\pi}{\pi_k}} \\ 
    & = \br*{ \inner*{ \nabla \nu\vHpik, \pi - \pi_k } + \frac{1}{t_k}\nu (I-\gamma P^{\pi_k})^{-1} \bregman{\pi}{\pi_k}} \\ 
    & =\inner*{ \nabla \nu\vHpik, \pi - \pi_k } + \frac{1}{t_k(1-\gamma)} d_{\nu,\pi_k} \bregman{\pi}{\pi_k},
\end{align*}
where the second transition is by plugging in \eqref{eq: switching the location of mu in the derivative} and the last transition is by the definition of the stationary distribution $d_{\nu,\pi_k}$.
\end{proof}

\subsection{Exact TRPO Update rule}\label{supp: exact TRPO update rule}

Exact TRPO repeatedly updates the policy by the following update rule (see \eqref{eq: exact scalar TRPO}),
\begin{align}
&\pi_{k+1} \in \argmin_{\pi\in \ssimplex} \left\{ \langle \nabla \nu\vHpik, \pi - \pi_k \rangle + \frac{1}{t_k(1-\gamma)} d_{\nu,\pi_k} \bregman{\pi}{\pi_k}\right\}. \label{eq: exact TRPO update rule original}
\end{align}
Note that differently than regular MD, the gradient here is w.r.t. to $\nu \vHpik$, and not $\mu \vHpik$ which is the true scalar objective \eqref{eq: scalar objective}. This is due to the fact that $d_{\nu,\pi_k}$ is the proper scaling for solving the MDP using the $\nu$-restart model, as can be seen in \eqref{eq: minimzation exact TRPO update 1}.

Using Proposition~\ref{prop: directional derivative for MDPs}, the update rule can be written as follows,
\begin{align}
\pi_{k+1} &\in \argmin_{\pi\in \ssimplex} \left\{ \frac{1}{1-\gamma}d_{\nu,\pi_k}\left(T_\lambda^\pi v^{\pi_k}_\lambda -v^{\pi_k}_\lambda   \right)+ \frac{1}{1-\gamma} \br*{\frac{1}{t_k}-\lambda} d_{\nu,\pi_k} \bregman{\pi}{\pi_k}\right\} \nonumber\\
&\in \argmin_{\pi\in \ssimplex} \left\{ d_{\nu,\pi_k}\left(T_\lambda^\pi v^{\pi_k}_\lambda -v^{\pi_k}_\lambda   +\br*{\frac{1}{t_k}-\lambda}\bregman{\pi}{\pi_k} \right) \right\}. \label{eq: minimzation exact TRPO update 1}
\end{align}

Much like the arguments we followed in Section~\ref{sec: uniform TRPO}, since $d_{\nu,\pi_k}\geq 0$ component-wise, minimizing \eqref{eq: minimzation exact TRPO update 1} is equivalent to minimizing $T_\lambda^\pi v^{\pi_k}_\lambda(s) -v^{\pi_k}_\lambda(s)  +\frac{1}{t_k}\bregmanS{s}{\pi}{\pi_k}$ for all $s$ for which $d_{\nu,\pi_k}(s)>0$. Meaning, the update rule takes the following form,
\begin{align}
    &\forall s \in \brc*{s'\in \mathcal{S} : d_{\nu,\pi_k}(s')>0 } \nonumber\\
    &\pi_{k+1}(\cdot\mid s)\in \arg\min_{\pi\in \simplex} \left(T_{\lambda}^{\pi}v^{\pi_k}_\lambda(s)-v^{\pi_k}(s)+ \br*{ \frac{1}{t_k}-\lambda} \bregmanS{s}{\pi}{\pi_k} \right)\label{eq: minimzation exact TRPO update final},
\end{align}
which can be written equivalently using Lemma~\ref{lemma: connection between bellman and q function}
\begin{align}
    &\forall s \in \brc*{s'\in \mathcal{S} : d_{\nu,\pi_k}(s')>0 } \nonumber\\
    &\pi_{k+1}(\cdot\mid s)\in \arg\min_{\pi\in \simplex} \br*{\inner*{ \qHpik(s,\cdot) + \lambda \nabla\omegaS{s}{\pi_k}, \pi - \pi_k(\cdot \mid s)} + \frac{1}{t_k} \bregmanS{s}{\pi}{\pi_k} }\label{eq: minimzation exact TRPO update final 2},
\end{align}
which will be use in the next section.

The minimization problem is solved component-wise as in Appendix \ref{supp: uniform TRPO}, equations \eqref{eq: mirror descent iteration L2} and \eqref{eq: mirror descent iteration KL} for the euclidean and non-euclidean cases, respectively. Thus, the solution of \eqref{eq: exact TRPO update rule original} is equivalent to a single iteration of Exact TRPO as given in Algorithm~\ref{alg: mu exact TRPO}.
\begin{remark}
Interestingly, the analysis does not depend on the updates in states for which ${d_{\nu,\pi_k}(s)=0}$. Although this might seem odd, the reason for this indifference is Assumption~\ref{assumption: stochastic MD}, by which $\forall s,k, d_{\mu,\pi^*}(s)>0 \rightarrow d_{\nu,\pi_k}(s)>0$. Meaning, by Assumption~\ref{assumption: stochastic MD} in each iteration we update all the states for which $d_{\mu,\pi^*}(s)>0$. This fact is sufficient to prove the convergence of Exact TRPO, with no need to analyze the performance at states for which $d_{\mu,\pi^*}(s)=0$ and $d_{\nu,\pi_k}(s)>0$.
\end{remark}

\subsection{Fundamental Inequality of Exact TRPO}\label{supp: mu exact TRPO}

In this section we will develop the fundamental inequality for Exact TRPO (Lemma~\ref{lemma: mu exact fundamental inequality}) based on its updating rule~\eqref{eq: minimzation exact TRPO update final}. We derive this inequality using two intermediate lemmas:
First, in Lemma~\ref{lemma: mu exact md point-wise inequality} we derive a state-wise inequality which holds for all states $s$ for which $d_{\nu,\pi_k}(s)>0$. Then, in Lemma~\ref{lemma: toward the mu exact fundamental inequality}, we use Lemma~\ref{lemma: mu exact md point-wise inequality} together with Assumption~\ref{assumption: stochastic MD} to prove an inequality related to the stationary distribution of the optimal policy $d_{\mu,\pi^*}$.
Finally, we prove the fundamental inequality for Exact TRPO using Lemma~\ref{lemma: greedy to difference}, which allows us to use the local guarantees of the inequality in Lemma~\ref{lemma: toward the mu exact fundamental inequality} for a global guarantee w.r.t. the optimal value, $\mu v_{\lambda}^*$.

\begin{lemma}[exact state-wise inequality]\label{lemma: mu exact md point-wise inequality} 
For all states $s$ for which $d_{\nu,\pi_k}(s)>0$ the following inequality holds: 
\begin{align*}
        0 & \leq t_k\br*{T_\lambda^{\pi}v_\lambda^{\pi_k}(s) - v_\lambda^{\pi_k}(s)} + \frac{t_k^2 \gradBound^2(k;\lambda)}{2} +(1-\lambda t_k) \bregmanS{s}{\pi}{\pi_k} - \bregmanS{s}{ \pi }{\pi_{k+1}}.
\end{align*}
where $h_\omega$ is defined at the third claim of Lemma \ref{lemma: usefull bounds exact case}.
\end{lemma}

\begin{proof}
Start by observing that the update rule \eqref{eq: minimzation exact TRPO update final 2} is applied in any state $s$ for which $d_{\nu,\pi_k}(s)>0$. By the first order optimality condition for the solution of \eqref{eq: minimzation exact TRPO update final 2}, for any policy $\pi\in \simplex$ at state $s$,

\begin{align}
    0 &\leq \inner*{ t_k\left(\qHpik(s,\cdot)+\lambda \nabla\omegaS{s}{\pi_k}\right)+   \nabla_{\pi_{k+1}} \bregmanS{s}{\pi_{k+1}}{\pi_k}, \pi - \pi_{k+1}(\cdot \mid s) } \nonumber\\
    & = \underbrace{t_k\inner*{ \qHpik(s,\cdot) +\lambda\nabla\omegaS{s}{\pi_k} ,\pi - \pi_{k+1}(\cdot \mid s)}}_{(1)}+ \underbrace{\inner*{\nabla_{\pi_{k+1}} \bregmanS{s}{\pi_{k+1}}{\pi_k}, \pi - \pi_{k+1}(\cdot \mid s) }}_{(2)} \label{eq: lemma improvement stochastic eq 1}
\end{align}

The first term can be bounded as follows.
\begin{align*}
    &(1) = t_k\inner*{ \qHpik(s,\cdot)+\lambda\nabla\omegaS{s}{\pi_k} ,\pi - \pi_{k+1}(\cdot \mid s)}\\
    &=t_k\inner*{ \qHpik(s,\cdot) + \lambda\nabla\omegaS{s}{\pi_k} ,\pi - \pi_{k}(\cdot \mid s) } \\
    &+ t_k \inner*{ \qHpik(s,\cdot)+\lambda\nabla\omegaS{s}{\pi_k} ,\pi_k(\cdot \mid s) - \pi_{k+1}(\cdot \mid s)}\\
     &\leq t_k\inner*{ \qHpik(s,\cdot)+\lambda\nabla\omegaS{s}{\pi_k}  ,\pi - \pi_{k}(\cdot \mid s)} \\
     & + |\inner*{ t_k \qHpik(s,\cdot)+t_k \lambda\nabla\omegaS{s}{\pi_k}  ,\pi_k(\cdot \mid s) - \pi_{k+1}(\cdot \mid s)}|\\ 
     &\leq t_k\inner*{ \qHpik(s,\cdot)+\lambda\nabla\omegaS{s}{\pi_k}  ,\pi - \pi_{k}(\cdot \mid s)} \\
     & + \frac{t_k^2 \norm{\qHpik(s,\cdot)+\lambda\nabla\omegaS{s}{\pi_k}}_{*}^2}{2}+ \frac{1}{2}\norm{\pi_k(\cdot \mid s) - \pi_{k+1}(\cdot \mid s)}^2,   
\end{align*}

where the last relation follows from Fenchel's inequality using the euclidean or non-euclidean norm $\norm{\cdot}$, and where $\norm{\cdot}_*$ is its dual norm, which is $L_2$ in the euclidean case, and $L_\infty$ in the non-euclidean case. Note that the norms are applied over the action space. Furthermore, by adding and subtracting $\lambda \omegaS{s}{\pi}$,
\begin{align}\label{eq: from inner product to operators}
   &\inner*{ \qHpik(s,\cdot)+\lambda\nabla\omegaS{s}{\pi_k},\pi - \pi_{k}} \nonumber\\
   & = \inner*{ \qHpik(s,\cdot),\pi - \pi_{k}(\cdot \mid s)} + \lambda \inner*{ \nabla\omegaS{s}{\pi_k},\pi - \pi_{k}(\cdot \mid s)}\nonumber\\
   & = T^{\pi}v_\lambda^{\pi_k}(s) - T^{\pi_k}v_\lambda^{\pi_k}(s) -\lambda \omegaS{s}{\pi} + \lambda \omegaS{s}{\pi_k} + \lambda \inner*{ \nabla\omegaS{s}{\pi_k},\pi - \pi_{k}(\cdot \mid s)}\nonumber\\
   & = T_\lambda^{\pi}v_\lambda^{\pi_k}(s) - T_\lambda^{\pi_k}v_\lambda^{\pi_k}(s) -\lambda \bregmanS{s}{\pi}{\pi_k}\nonumber\\
    & = T_\lambda^{\pi}v_\lambda^{\pi_k}(s) - v_\lambda^{\pi_k}(s) -\lambda \bregmanS{s}{\pi}{\pi_k},
\end{align}
where the second transition follows the same steps as in equation~\eqref{eq: relation between bellman and directional q} in the proof of Proposition~\ref{prop: directional derivative for MDPs}, and the third transition is by the definition of the Bregman distance of $\omega$. Note that \eqref{eq: from inner product to operators} is actually given in Lemma~\ref{lemma: connection between bellman and q function}, but is re-derived here for readability.

From which, we conclude that

\begin{align*}
    (1) & \leq t_k\br*{T_\lambda^{\pi}v_\lambda^{\pi_k}(s) - v_\lambda^{\pi_k}(s) -\lambda \bregmanS{s}{\pi}{\pi_k}} \\
    & + \frac{t_k^2 \norm{\qHpik(s,\cdot)+\lambda\nabla\omegaS{s}{\pi_k}}_{*}^2}{2}+ \frac{1}{2}\norm{\pi_k(\cdot \mid s) - \pi_{k+1}(\cdot \mid s)}^2\\
    & \leq t_k\br*{T_\lambda^{\pi}v_\lambda^{\pi_k}(s) - v_\lambda^{\pi_k}(s) -\lambda \bregmanS{s}{\pi}{\pi_k}}  + \frac{t_k^2 \gradBound^2(k;\lambda)}{2} + \frac{1}{2}\norm{\pi_k(\cdot \mid s) - \pi_{k+1}(\cdot \mid s)}^2,
\end{align*}
where in the last transition we used the third claim of Lemma~\ref{lemma: usefull bounds exact case},

We now continue analyzing $(2)$.
\begin{align*}
    &(2) = \inner*{\nabla_{\pi_{k+1}} \bregmanS{s}{\pi_{k+1}}{\pi_k}, \pi -  \pi_{k+1}(\cdot \mid s) }\\
    & = \inner*{\nabla \omegaS{s}{\pi_{k+1}}-  \nabla \omegaS{s}{\pi_k},\pi -  \pi_{k+1}(\cdot \mid s)}\\
    & = \bregmanS{s}{\pi}{\pi_k} - \bregmanS{s}{ \pi }{\pi_{k+1}} - \bregmanS{s}{\pi_{k+1}}{ \pi_k }\\
    &\leq \bregmanS{s}{\pi}{\pi_k} - \bregmanS{s}{ \pi }{\pi_{k+1}} - \frac{1}{2}\norm{\pi_{k}(\cdot \mid s) -  \pi_{k+1}(\cdot \mid s)}^2.
\end{align*}

The first relation, $\nabla_{\pi_{k+1}} \bregmanS{s}{\pi_{k+1}}{\pi_k} = \nabla \omegaS{s}{\pi_{k+1}}-  \nabla \omegaS{s}{\pi_k}$, holds by simply taking the derivative of any Bregman distance w.r.t. $\pi_{k+1}$. The second relation holds by the three-points lemma (Lemma~\ref{lemma: three-points}). The third relation holds by the strong convexity of the Bregman distance, i.e., $\frac{1}{2}\norm{x-y}^2\leq \bregman{x}{y}$, which is straight forward in the euclidean case, and is the well known Pinsker's inequality in the non-euclidean case.

Plugging the above upper bounds for $(1)$ and $(2)$ into \eqref{eq: lemma improvement stochastic eq 1} we get,
\begin{align*}
    0 &\leq t_k\br*{T_\lambda^{\pi}v_\lambda^{\pi_k}(s) - v_\lambda^{\pi_k}(s)} + \frac{t_k^2 (\gradBound^2(k;\lambda)) }{2}  + (1-\lambda t_k)\bregmanS{s}{\pi}{\pi_k} - \bregmanS{s}{ \pi }{\pi_{k+1}},
\end{align*}

and conclude the proof.
\end{proof}

We now turn state another lemma, which connects the state-wise inequality using the discounted stationary distribution of the optimal policy $d_{\mu,\pi^*}$

\begin{lemma}\label{lemma: toward the mu exact fundamental inequality} 
Assuming~\ref{assumption: stochastic MD}, the following inequality holds for all $\pi$.
\begin{align*}
        0 & \leq  t_kd_{\mu,\pi^*} \br*{T_\lambda^{\pi}v_\lambda^{\pi_k} - v_\lambda^{\pi_k}} + \frac{t_k^2 \gradBound^2(k;\lambda)}{2} +(1-\lambda t_k)d_{\mu,\pi^*} \bregman{\pi}{\pi_k} - d_{\mu,\pi^*}\bregman{ \pi }{\pi_{k+1}}.
\end{align*}
where $h_\omega$ is defined at the third claim of Lemma \ref{lemma: usefull bounds exact case}. 
\begin{proof}
    By Assumption~\ref{assumption: stochastic MD}, for all $s$ for which $d_{\mu,\pi^*}(s)>0$ it also holds that $d_{\nu,\pi_k}(s)>0$. Thus, for all $s$ for which $d_{\mu,\pi^*}(s)>0$ the component-wise relation in Lemma~\ref{lemma: mu exact md point-wise inequality}  holds. By multiplying each inequality by the positive number $d_{\mu,\pi^*}(s)$ and summing over all $s$ we get,
    \begin{align*}
        0 & \leq  t_kd_{\mu,\pi^*} \br*{T_\lambda^{\pi}v_\lambda^{\pi_k} - v_\lambda^{\pi_k}} + \frac{t_k^2 \gradBound^2(k;\lambda)}{2} +(1-\lambda t_k)d_{\mu,\pi^*} \bregman{\pi}{\pi_k} - d_{\mu,\pi^*}\bregman{ \pi }{\pi_{k+1}},
\end{align*}
which concludes the proof.

\end{proof}
\end{lemma}
Using the previous lemma, we are ready to prove the following Lemma:

\begin{lemma}[fundamental inequality of exact TRPO]\label{lemma: mu exact fundamental inequality}
Let $\{\pi_k\}_{k\geq 0}$ be the sequence generated by the TRPO method using stepsizes $\{t_k\}_{k\geq0}$.  Then, for all $k\geq 0$

\begin{align*}
&t_k(1-\gamma) (\mu v^{\pi_k}_\lambda- \mu v^{\pi^*}_\lambda) \leq d_{\mu,\pi^*}\br*{ (1-\lambda t_k)\bregman{\pi^*}{\pi_k} - \bregman{ \pi^* }{\pi_{k+1}}} +\frac{t_k^2 \gradBound^2(k;\lambda)}{2},
\end{align*}
where $\gradBound(k;\lambda)$ is defined in Lemma~\ref{lemma: usefull bounds exact case}.
\end{lemma}
\begin{proof}
Setting $\pi=\pi^*$ in Lemma \ref{lemma: toward the mu exact fundamental inequality}   we get that for any $k$,
\begin{align*}
    &-t_k d_{\mu,\pi^*}\br*{T^{\pi^*}_\lambda \vHpik -  \vHpik} \\
    &\leq  d_{\mu,\pi^*}\br{ (1-\lambda t_k)\bregman{\pi^*}{\pi_k} - \bregman{\pi^* }{\pi_{k+1}}} +\frac{t_k^2 \gradBound^2(k;\lambda)}{2}.
\end{align*}

Furthermore, by the third claim in Lemma~\ref{lemma: greedy to difference},
\begin{align*}
    (1-\gamma)\mu(v_\lambda^*-\vHpik) = d_{\mu,\pi^*}\br*{T^{\pi^*}_\lambda \vHpik -  \vHpik}.
\end{align*}

Combining the two relations and taking expectation on both sides we conclude the proof.

\end{proof}

We are ready to prove the convergence rates for the unregularized and regularized algorithms, much like the equivalent proofs in the case of Uniform TRPO in Appendix~\ref{supp: proof of uniform TRPO}.

\subsection{Convergence proof of Exact TRPO}\label{supp: convergence of exact trpo}

Before proving the theorem, we establish that the policy improves in $k$ for the chosen learning rates.
\begin{lemma}[Exact TRPO Policy Improvement]\label{lemma: policy improvement exact}

Let $\{\pik\}_{k\geq 0}$ be the sequence generated by Exact TRPO. Then, for both the euclidean and non-euclidean versions of the algorithm, for any $\lambda\geq0$, the value improves for all $k$,  $$ v_\lambda^{\pi_k} \geq  v_\lambda^{\pi_{k+1}},$$
and, thus, $\mu v_\lambda^{\pi_k} \geq \mu v_\lambda^{\pi_{k+1}}$
\end{lemma}

\begin{proof}

By \eqref{eq: lemma improvement stochastic eq 1}, for any state $s$ for which $d_{\nu,\pi_k}(s)>0$, and for any policy $\pi\in \simplex$ at state $s$,

\begin{align*}
    0 &\leq t_k\inner*{ \qHpik(s,\cdot) +\lambda\nabla\omegaS{s}{\pi_k} ,\pi - \pi_{k+1}(\cdot \mid s)} + \inner*{\nabla_{\pi_{k+1}} \bregmanS{s}{\pi_{k+1}}{\pi_k}, \pi - \pi_{k+1}(\cdot \mid s) }
\end{align*}

Thus,
% By the first order optimality over the update step in equation \ref{eq: exact TRPO update rule original}, we get that for any policy $\pi\in \simplex$,
\begin{align*}
    & 0\leq t_k\inner*{ \qHpik(s,\cdot) +\lambda\nabla\omegaS{s}{\pi_k} ,\pi - \pi_{k+1}(\cdot \mid s)} + \inner*{\nabla \omegaS{s}{\pi_{k+1}}-  \nabla \omegaS{s}{\pi_k},\pi(\cdot\mid s) -  \pi_{k+1}(\cdot \mid s)} \\
    & \rightarrow 0\leq t_k\inner*{ \qHpik(s,\cdot) +\lambda\nabla\omegaS{s}{\pi_k} ,\pi - \pi_{k+1}(\cdot \mid s)} +\bregmanS{s}{\pi}{\pi_k} - \bregmanS{s}{ \pi }{\pi_{k+1}} - \bregmanS{s}{\pi_{k+1}}{ \pi_k },
\end{align*}
where the first relation is by the derivative of the Bregman distance and the second is by the three-point lemma (\ref{lemma: three-points}).

By choosing $\pi=\pi_k(\cdot \mid s)$,
\begin{align*}
 0\leq t_k\inner*{ \qHpik(s,\cdot) +\lambda\nabla\omegaS{s}{\pi_k} ,\pi_k(\cdot \mid s) - \pi_{k+1}(\cdot \mid s)} - \bregmanS{s}{ \pi_k }{\pi_{k+1}} - \bregmanS{s}{\pi_{k+1}}{ \pi_k },
\end{align*}
where we used the fact that $\bregman{s}{\pi_k}{\pi_k}=0$.

Now, Using equation~\eqref{eq: from inner product to operators} (see Lemma~\ref{lemma: connection between bellman and q function}), we get
\begin{align}\label{eq: improvement first relation exact}
 &0\leq t_k\br*{  v_\lambda^{\pi_k}(s) - T_\lambda^{\pi}v_\lambda^{\pi_k}(s)  +\lambda \bregmanS{s}{\pi_{k+1}}{\pi_k}} - \bregmanS{s}{ \pi_k }{\pi_{k+1}} - \bregmanS{s}{\pi_{k+1}}{ \pi_k } \nonumber \\
 & \rightarrow  \bregmanS{s}{ \pi_k }{\pi_{k+1}} + \br*{1-\lambda t_k}\bregmanS{s}{\pi_{k+1}}{ \pi_k } \leq t_k\br*{  v_\lambda^{\pi_k}(s) - T_\lambda^{\pi_{k+1}}v_\lambda^{\pi_k}(s)} 
\end{align}

The choice of the learning rate and the fact that the Bregman distance is non negative ($\lambda >0$, $\lambda t_k = \frac{1}{k+2} \leq 1$ and for $\lambda=0$ the RHS of~\eqref{eq: improvement first relation exact} is positive), implies that  for all $s\in \brc*{s': d_{\nu,\pi_k}(s')>0}$.

\begin{align}
     &0\leq v_\lambda^{\pi_k}(s) - T_\lambda^{\pi_{k+1}}v_\lambda^{\pi_k}(s), \label{eq: result 1 improvement exact}
\end{align}

For all states $s\in\mathcal{S}$ for which $d_{\nu,\pi_k}(s)= 0$, as we do not update the policy in these states we have that $\pi_{k+1}(\cdot \mid s) = \pi_{k}(\cdot \mid s)$. Thus, for all $s\in \brc*{s' : d_{\nu,\pi_k}(s')=0}$,
\begin{align}
    0 = v_\lambda^{\pi_k}(s) - T_\lambda^{\pi_{k}}v_\lambda^{\pi_k}(s) = v_\lambda^{\pi_k}(s) - T_\lambda^{\pi_{k+1}}v_\lambda^{\pi_k}(s).\label{eq: result 2 improvement exact}
\end{align}

Combining \eqref{eq: result 1 improvement exact}, \eqref{eq: result 2 improvement exact} we get that for all $s\in \mathcal{S}$,
\begin{align}
    v_\lambda^{\pi_k}(s) \geq T_\lambda^{\pi_{k+1}}v_\lambda^{\pi_k}(s).\label{eq: result 3 improvement exact}
\end{align}

% Thus, by the last equation and Assumption~\ref{assumption: stochastic MD}, it holds that for any state $s$ for which $d_{\mu,\pi^*}(s)>0$, 
% \begin{align}
%     v_\lambda^{\pi_k}(s)  \geq T_\lambda^{\pi_{k+1}}v_\lambda^{\pi_k}(s). \label{eq: improvement second relation exact}
% \end{align}
% \todo{Until here everything correct}
% \todo{We finally need to use the monotonocity, but I'm not sure it is correct in this settings, we need to demand 
% $$\mathcal{S}_{d_{\nu,\pi_{k+1}}}\subseteq \mathcal{S}_{d_{\nu,\pi_{k}}}$$}

% For any $s$ for which $d_{\mu,\pi^*}>0$ we have that 

Applying iteratively $T^{\pi_{k+1}}_\lambda$ and using its monotonicty we obtain,
\begin{align*}
    \vHpik  \geq T^{\pi_{k+1}}_\lambda \vHpik \geq (T^{\pi}_\lambda)^2\vHpik\geq \cdot\cdot\cdot \geq \lim_{n\rightarrow \infty} (T^{\pi_{k+1}}_\lambda)^n \vHpik = v^{\pi_{k+1}}_\lambda,
\end{align*}
where in the last relation we used the fact $T^{\pi_{k+1}}_\lambda$ is a contraction operator and its fixed point is $v^{\pi_{k+1}}_\lambda$.

Finally we conclude the proof by multiplying both sides with $\mu$ which gives $\mu v^{\pi_{k+1}}_\lambda \leq \mu v^{\pi_{k}}_\lambda$
\end{proof}

% \todo{end Temp}

The following theorem establish the convergence rates of the Exact TRPO algorithms.

\begin{theorem}[Convergence Rate: Exact TRPO]\label{theorem: Exact TRPO}  
Let $\{\pik\}_{k\geq 0}$ be the sequence generated by Exact TRPO
% and $\{\mu v_{\textbf{best}}^k\}_{k\geq0}$ be the sequence of best achieved values, $\mu v_{\textbf{best}}^N\triangleq \arg\min_{k=0,...,N} \mu v^{\pik}_\lambda-\mu v_\lambda^*$. 
Then, the following holds for all $N\geq1$.
\begin{enumerate}
    \item \emph{(Unregularized)} Let $\lambda=0$, $t_k=\frac{(1-\gamma)}{\euclid \mathrm{C_{\text{max}}} \sqrt{k+1}}$ then
    \begin{align*}
       \mu v^{\pi_N} - \mu \vstar   \leq 
        O\br*{\frac{\euclid \mathrm{C_{\text{max}}}(\bregmanMax + \log N)}{(1-\gamma)^2\sqrt{N}} }    \end{align*}
    \item \emph{(Regularized)} Let $\lambda>0$, $t_k =  \frac{1}{\lambda (k+2)}$ then

\begin{align*}
   \mu  v_\lambda^{\pi_N} - \mu  \vHstar   \leq  O \br*{ \frac{\euclid ^2\CmaxLambda^2 \log N}{\lambda(1
    -\gamma)^3 N} }.
\end{align*}
\end{enumerate}
Where $\euclid =\sqrt{A}, \bregmanMax=1$ for the euclidean case, and $\euclid =1, \bregmanMax=\log A$ for the non-euclidean case.
\end{theorem}

\subsubsection{The Unregularized case}\label{supp: exact TRPO unregularized proof}
\begin{proof}

Applying Lemma~\ref{lemma: mu exact fundamental inequality}  and $\lambda=0$ (the unregularized case),
\begin{align*}
    &t_k(1-\gamma) (\mu\vpik-\mu \vstar) \\
&\leq    d_{\mu,\pi^*}\br*{ \bregman{\pi^*}{\pi_k} - \bregman{ \pi^* }{\pi_{k+1}}} + \frac{t_k^2 \gradBound^2}{2}.
\end{align*}

Summing the above inequality over $k=0,1,...,N$, gives
\begin{align*}
    &\sum\limits_{k=0}^{N}  t_k(1-\gamma)(\mu \vpik- \mu\vstar) \\
    &\leq   d_{\mu,\pi^*}\bregman{\pistar}{\pi_0} - d_{\mu,\pi^*}\bregman{\pistar}{\pi_{N+1}} +  \sum\limits_{k=0}^{N}\frac{t_k^2 \gradBound^2}{2}\\
    &\leq  d_{\mu,\pi^*}\bregman{\pistar}{\pi_0} +  \sum\limits_{k=0}^{N}\frac{t_k^2 \gradBound^2}{2}\\
    &\leq  \omegaBound +  \sum\limits_{k=0}^{N}\frac{t_k^2 \gradBound^2}{2}.
\end{align*}

where in the second relation we used    $\bregman{\pistar}{\pi_{N+1}}\geq 0$ and thus $d_{\mu,\pi^*}\bregman{\pistar}{\pi_{N+1}}\geq 0$, and in the third relation Lemma \ref{lemma: bounds on Dw}.

By the improvement lemma (Lemma~\ref{lemma: policy improvement exact}), 
$$\mu(v^{\pi_N}-\vstar) \sum\limits_{k=0}^{N}  t_k \leq \sum\limits_{k=0}^{N}  t_k(\mu \vpik-\mu \vstar),$$

and by some algebraic manipulations, we get
\begin{align*}
    \mu v^{\pi_N}-\mu \vstar \leq \frac{1}{1-\gamma}
   \frac{\omegaBound +  \sum\limits_{k=0}^{N}\frac{t_k^2 \gradBound^2}{2}}{\sum\limits_{k=0}^{N} {t_k}}\\
   = \frac{1}{1-\gamma}
   \frac{\omegaBound +  \frac{ \gradBound^2}{2}\sum\limits_{k=0}^{N}t_k^2}{\sum\limits_{k=0}^{N} {t_k}},
\end{align*}
Plugging in the stepsizes $t_k = \frac{1}{\gradBound\sqrt{k}}$, we get, 
\begin{align*}
      \mu v^{\pi_N}-\mu \vstar &\leq \frac{\gradBound}{1-\gamma}
   \frac{2\omegaBound + \sum\limits_{k=0}^{N}\frac{1}{k}}{2\sum\limits_{k=0}^{N} \frac{1}{\sqrt{k}}}.
\end{align*}

Bounding the sums using (\citealp{beck2017first}, Lemma 8.27(a)) yields,
\begin{align*}
      \mu v^{\pi_N}-\mu \vstar &\leq O\br*{\frac{\gradBound}{1-\gamma}
   \frac{\omegaBound + \log{N}}{\sum\limits_{k=0}^{N} \frac{1}{\sqrt{k}}}}.
\end{align*}
Plugging the expressions for $\gradBound$ and $\omegaBound$ in Lemma \ref{lemma: usefull bounds exact case} and Lemma \ref{lemma: bounds on Dw}, we get for the euclidean case,
\begin{align*}
     \mu v^{\pi_N}-\mu \vstar \leq O\br*{\frac{\Cmax \sqrt{A}\log N}{(1-\gamma)^2\sqrt{N}} },
\end{align*}
and for the non-euclidean case,
\begin{align*}
    \mu v^{\pi_N}-\mu \vstar \leq O\br*{\frac{\Cmax(\log A +  \log N)}{(1-\gamma)^2\sqrt{N}} }.
\end{align*}

\end{proof}

\subsubsection{The Regularized case}\label{supp: exact TRPO regularized proof}

\begin{proof}
Applying Lemma~\ref{lemma: mu exact fundamental inequality} and setting $t_k= \frac{1}{\lambda(k+2)}$, we get,

\begin{align*}
&\frac{1-\gamma}{\lambda (k+2)} \br*{\mu v^{\pi_k}_\lambda- \mu v^{\pi^*}_\lambda } \\
&\leq   d_{\mu,\pi^*}\br*{  (1- \frac{1}{(k+2)})\bregman{\pi^*}{\pi_k} - \bregman{ \pi^* }{\pi_{k+1}}} +\frac{\gradBound^2(k;\lambda)}{2\lambda^2(k+2)^2}\\
&\leq   d_{\mu,\pi^*}\br*{  \frac{k+1}{k+2}\bregman{\pi^*}{\pi_k} - \bregman{ \pi^* }{\pi_{k+1}}} +\frac{\gradBound^2(N;\lambda)}{2\lambda^2(k+2)^2},
\end{align*}

where in the second relation we used that fact  $\gradBound(k;\lambda)$ is a non-decreasing function of $k$ for both the euclidean and non-euclidean cases. 

Next, multiplying both sides by $\lambda (k+2)$, summing both sides from $k=0$ to $N$ and using the linearity of expectation, we get,

\begin{align*}
    \sum_{k=0}^{N} (1-\gamma)\br*{
    \mu \vHpik-\mu \vHstar}&\leq  d_{\mu,\pi^*}\br*{  \bregman{\pi^*}{\pi_0} - (N+2)\bregman{ \pi^* }{\pi_{N+1}}} +\sum_{k=0}^{N}\frac{\gradBound^2(N;\lambda)}{2\lambda (k+2)}\\
    &\leq  d_{\mu,\pi^*}  \bregman{\pi^*}{\pi_0} +\sum_{k=0}^{N}\frac{\gradBound^2(N;\lambda)}{2\lambda (k+2)}\\
    &\leq  \omegaBound +\sum_{k=0}^{N}\frac{\gradBound^2(N;\lambda)}{2\lambda (k+2)},
\end{align*}
where the second relation holds by the positivity of the Bregman distance, and the third relation by Lemma \ref{lemma: bounds on Dw} for uniformly initialized $\pi_0$.

Bounding $\sum_{k=0}^N\frac{1}{k+2}\leq O(\log N)$, we get
\begin{align*}
    \sum_{k=0}^{N} 
    \mu\vHpik-\mu\vHstar  \leq O\br*{\frac{\omegaBound}{(1-\gamma)} +\frac{\gradBound^2(N;\lambda) \log N} {\lambda (1-\gamma)}}
\end{align*}

Since $ N \br*{\mu v_\lambda^{\pi_N}-\mu \vstar} \ \leq \sum\limits_{k=0}^{N}  \mu \vpik-\mu \vstar$
by Lemma~\ref{lemma: policy improvement exact} and some algebraic manipulations, we obtain

\begin{align*}
    \mu v_\lambda^{\pi_N}-\mu\vHstar \leq O\br*{\frac{\omegaBound}{(1-\gamma)N} +\frac{\gradBound^2(N;\lambda) \log N} {\lambda (1-\gamma)N}}.
\end{align*}

By Plugging the bounds $\omegaBound, \gradBound$ and $\CmaxLambda$, we get in the euclidean case,
\begin{align*}
    \mu v_\lambda^{\pi_N}-\mu\vHstar \leq O\br*{\frac{\br*{\CmaxSquare+\lambda^2} A \log N} {\lambda (1-\gamma)^3 N}},
\end{align*}
and in the non-euclidean case,
\begin{align*}
    \mu v_\lambda^{\pi_N}-\mu \vHstar \leq O\br*{\frac{(\mathrm{C^2_{\text{max}}}+\lambda^2 \log^2 A) \log^3 N} {\lambda (1-\gamma)^3 N}}.
\end{align*}
 \end{proof}

\section{Sample-Based Trust Region Policy Optimization}\label{supp: approximate Mirror descent}

Sample-Based TRPO is a sample-based version of Exact TRPO which was analyzed in previous section (see Appendix~\ref{supp: mu exact policy mirror descenet}). Unlike Uniform TRPO (see Appendix~\ref{supp: uniform TRPO proof}) which accesses the entire state and computes $v^\pi\in \mathbb{R}^S$ in each iteration, Sample-Based TRPO requires solely the ability to sample from an MDP using a $\nu$-restart model. Similarly to~\cite{kakade2003sample} it requires Assumption~\ref{assumption: stochastic MD} to be satisfied. Thus, Sample-Based TRPO operates under much more realistic assumptions, and, more importantly, puts formal ground to first-order gradient based methods such as NE-TRPO \cite{schulman2015trust}, which was so far considered a heuristic method motivated by CPI \cite{kakade2002approximately}. 

In this section we prove Sample-Based TRPO (Section~\ref{sec: approximate MD}, Theorem~\ref{theorem: Sample-Based TRPO All}) converges to an approximately optimal solution with high probability. The analysis in this section relies heavily on the analysis of Exact TRPO in Appendix~\ref{supp: mu exact policy mirror descenet}. We now describe the content of each of the subsections: First, in Appendix~\ref{supp: Exact to Approximate}, we show the connections between Sample-Based TRPO (using unbiased estimation) and Exact TRPO by proving Proposition~\ref{proposition: exact to approximate}. In Appendix~\ref{supp: approximate update rule TRPO}, we analyze the Sample-Based TRPO update rule and formalize the truncated sampling process. In Appendix~\ref{supp: proof sketch Sample-Based TRPO}, we give a detailed proof sketch of the convergence theorem for Sample-Based TRPO, in order to ease readability. Then, we derive a fundamental inequality that will be used to prove the convergence of both unregularized and regularized versions (Appendix \ref{supp: Sample-Based TRPO funamental inequality}). This inequality is almost identical to the fundamental inequality derived for Exact TRPO (Lemma~\ref{lemma: mu exact fundamental inequality}), but with an additional term which arises due to the approximation error. In Appendix~\ref{supp: Approximation Error bounds}, we analyze the sample complexity needed to bound this approximation error. We go on to prove the convergence rates of Sample-Based TRPO for both the unregularized and regularized version (Appendix~\ref{supp: proof of Sample-Based TRPO}). Finally, in Appendix~\ref{supp: sample complexity of Approximate Mirror Descent}, we calculate the overall sample complexity of both the unregularized and regularized Sample-Based TRPO and compare it to CPI.

\subsection{Relation Between Exact and Sample-Based TRPO}\label{supp: Exact to Approximate}

Before diving into the proof of Sample-Based TRPO, we prove Proposition~\ref{proposition: exact to approximate}, which connects the update rules of Exact TRPO and Sample-Based TRPO (in case of an unbiased estimator for $\qHpik$):
\ExactToApproximateupdate*

    % &\langle \nabla \mu\vHpik, \pi - \pi_k \rangle + \frac{1}{t_k(1-\gamma)} d_{\nu,\pi_k} \bregman{\pi}{\pi_k} \\
    % &= \frac{1}{1-\gamma} \underset{\underset{a_m\sim U(\mathcal{A})}{s_m\sim d_{\nu,\pi_k}}}{\E}\left[A \inner{\hat{q}_\lambda(s_m,\cdot)\mathbbm{1}\brc*{\cdot=a_m},\pi-\pi_k} \right.\\
    % &\left. \quad\quad\quad\quad\quad\quad +\frac{1}{t_k(1-\gamma )}\bregmanS{\pi}{\pi_k}{s_m} \mid \mathcal{F}_k \right].
\begin{proof}
For any $m=1,...,M$, we take expectation over the sampling process given the filtration $\mathcal{F}_k$, i.e., $s_m \sim d_{\nu,\pi_k},a_m\sim U(\aset),\hatqHpik\sim\qHpik$ (we assume here an unbiased estimation process where we do not truncate the sample trajectories),

\begin{align*}
 & \E \brs*{\inner{ \hat \nabla \nu\vHpik[m], \pi(\cdot\mid s_m) - \pi_k(\cdot\mid s_m) } + \frac{1}{t_k(1-\gamma)}\bregmanS{s_m}{\pi}{\pi_k} \mid \mathcal{F}_k} \\
  & = \E\brs*{\frac{1}{1-\gamma}\inner{A\hatqHpik(s_m,\cdot,m)\mathbbm{1}\brc*{\cdot=a_m}+\nabla\omegaS{s_m}{\pi_k},\pi(\cdot \mid s_m)- \pi_k(\cdot \mid s_m)} +\frac{1}{t_k(1-\gamma)}\bregmanS{s_m}{\pi}{\pi_k} \mid \mathcal{F}_k } 
 \\ 
 & = \frac{1}{1-\gamma} \E \brs*{{\E}_{\hatqHpik}\brs*{\inner{A \hatqHpik(s_m,\cdot,m)\mathbbm{1}\brc*{\cdot=a_m} + \nabla\omegaS{s_m}{\pi_k},\pi(\cdot \mid s_m)- \pi_k(\cdot \mid s_m)} +\frac{1}{t_k}\bregmanS{s_m}{\pi}{\pi_k}\mid s_m,a_m } \mid \mathcal{F}_k } \\
 & = \frac{1}{1-\gamma} \E  \brs*{\inner{{\E}_{\hatqHpik} \brs*{ \hatqHpik(s_m,\cdot,m)\mathbbm{1}\brc*{\cdot=a_m} \mid s_m,a_m} +\nabla\omegaS{s_m}{\pi_k},\pi(\cdot \mid s_m)- \pi_k(\cdot \mid s_m)} +\frac{1}{t_k}\bregmanS{s_m}{\pi}{\pi_k} \mid \mathcal{F}_k } \\
 & = \frac{1}{1-\gamma} \E \brs*{\inner{A \qHpik(s_m,\cdot)\mathbbm{1}\brc*{\cdot=a_m}+\nabla\omegaS{s_m}{\pi_k},\pi(\cdot \mid s_m)- \pi_k(\cdot \mid s_m)} +\frac{1}{t_k}\bregmanS{s_m}{\pi}{\pi_k} \mid \mathcal{F}_k } \\
 & = (*),
 \end{align*} 
 where first transition is by the definition of $\hat \nabla \nu \vHpik[m]$, the second by the smoothing theorem, the third transition is due to the linearity of expectation and the fourth transition is by taking the expectation and due to the fact that $\mathbbm{1}\brc*{a=a_m}$ is zero for any $a \neq a_m$.

\begin{align*}
 & (*) = \frac{1}{1-\gamma} \E \brs*{\inner{A \qHpik(s_m,\cdot)\mathbbm{1}\brc*{\cdot=a_m}+\nabla\omegaS{s_m}{\pi_k},\pi(\cdot \mid s_m)- \pi_k(\cdot \mid s_m)} +\frac{1}{t_k}\bregmanS{s_m}{\pi}{\pi_k} \mid \mathcal{F}_k } \\
 & = \frac{1}{1-\gamma} \E_{s_m} \brs*{ \sum_{a_m \in \aset} \frac{1}{A} \inner{A\qHpik(s_m,\cdot)\mathbbm{1}\brc*{\cdot=a_m}+\nabla\omegaS{s_m}{\pi_k},\pi(\cdot \mid s_m)- \pi_k(\cdot \mid s_m)} +\frac{1}{t_k}\bregmanS{s_m}{\pi}{\pi_k} \mid \mathcal{F}_k } \\
 & = \frac{1}{1-\gamma} \E_{s_m} \brs*{ \inner{\sum_{a_m \in \aset}\frac{1}{A}A\qHpik(s_m,\cdot)\mathbbm{1}\brc*{\cdot=a_m}+\nabla\omegaS{s_m}{\pi_k},\pi(\cdot \mid s_m)- \pi_k(\cdot \mid s_m)} +\frac{1}{t_k}\bregmanS{s_m}{\pi}{\pi_k} \mid \mathcal{F}_k } \\
 & = \frac{1}{1-\gamma} \E_{s_m} \brs*{ \inner{\qHpik(s_m,\cdot)\sum_{a_m \in \aset}\mathbbm{1}\brc*{\cdot=a_m}+\nabla\omegaS{s_m}{\pi_k},\pi(\cdot \mid s_m)- \pi_k(\cdot \mid s_m)} +\frac{1}{t_k}\bregmanS{s_m}{\pi}{\pi_k} \mid \mathcal{F}_k } \\
 & = \frac{1}{1-\gamma} \E_{s_m} \brs*{ \inner{\qHpik(s_m,\cdot)+\nabla\omegaS{s_m}{\pi_k},\pi(\cdot \mid s_m)- \pi_k(\cdot \mid s_m)} +\frac{1}{t_k}\bregmanS{s_m}{\pi}{\pi_k} \mid \mathcal{F}_k } \\
 & = (**).
 \end{align*}
where the second transition is by taking the expectation over $a_m$, the third transition is by the linearity of the inner product and due to the fact that $\inner*{\nabla\omegaS{s_m}{\pi_k},\pi(\cdot \mid s_m)- \pi_k(\cdot \mid s_m)}$ and $\bregmanS{s_m}{\pi}{\pi_k}$ are independent of $a_m$.

Now, taking the expectation over $s_m\sim d_{\nu,\pi_k}$, 
\begin{align*}
 & (**) = \frac{1}{1-\gamma} \E_{s_m} \brs*{ \inner{\qHpik(s_m,\cdot)+\nabla\omegaS{s_m}{\pi_k},\pi(\cdot \mid s_m)- \pi_k(\cdot \mid s_m)} +\frac{1}{t_k}\bregmanS{s_m}{\pi}{\pi_k} \mid \mathcal{F}_k }  \\
 & =  \frac{1}{1-\gamma}\sum_s d_{\nu,\pi_k}(s)\br*{ \inner{\qHpik(s,\cdot)+\nabla\omegaS{s}{\pi_k},\pi(\cdot \mid s)- \pi_k(\cdot \mid s)} +\frac{1}{t_k}\bregmanS{s}{\pi}{\pi_k}}\\
 &= \frac{1}{1-\gamma} d_{\nu,\pi_k} \inner{\qHpik + \nabla\omega(\pi_k) ,\pi- \pi_k} +\frac{1}{t_k}\frac{1}{1-\gamma} d_{\nu,\pi_k}\bregman{\pi}{\pi_k} \\
 & = \frac{1}{1-\gamma} d_{\nu,\pi_k} \br*{T_\lambda^{\pi} \vHpik - \vHpik  - \lambda \bregman{\pi}{\pi_k} } + \frac{1}{t_k}\frac{1}{1-\gamma} d_{\nu,\pi_k}\bregman{\pi}{\pi_k} \\
 & =  \inner*{\nabla \nu \vHpik, \pi - \pi_k} + \frac{1}{t_k(1-\gamma)} d_{\nu,\pi_k}\bregman{\pi}{\pi_k} ,
 \end{align*}
 where the second transition is by taking the expectation w.r.t. to $s_m$, the the fourth is by using the lemma~\ref{lemma: connection between bellman and q function} which connects the bellman operator and the $q$-functions, and the last transition is due to \eqref{eq: Regularized value directional derivative second} in Proposition \ref{prop: directional derivative for MDPs}, which concludes the proof.

\end{proof}

\subsection{Sample-Based TRPO Update Rule}\label{supp: approximate update rule TRPO}

In each step, we solve the following optimization problem~\eqref{eq: approximate scalar TRPO}:
\begin{align*}
    \pi_{k+1} &\in \argmin_{\pi\in \ssimplex} \Big\{ \frac{1}{M}\sum_{m=1}^M  \inner{ \hat \nabla \nu\vHpik[m], \pi(\cdot\mid s_m) - \pi_k(\cdot\mid s_m) }  + \frac{1}{t_k(1-\gamma)}  \bregmanS{s_m}{\pi}{\pi_k}  \Big\} \\
     &\in \argmin_{\pi\in \ssimplex} \brc*{ \frac{1}{M}\sum_{m=1}^{M} \begin{pmatrix} \inner*{ A \hatqHpik(s_m,\cdot,m)\mathbbm{1}\{\cdot = a_m\} +\lambda \nabla \omegaS{s_m}{\pik} , \pi(\cdot \mid s_m) - \pi_k (\cdot \mid s_m)}  + \frac{1}{t_k}  \bregmanS{s_m}{\pi}{\pi_k} \end{pmatrix} } \\
     &\in \argmin_{\pi\in \ssimplex} \brc*{ \sum_{s\in \mathcal{S}}\sum_{m=1}^{M} \mathbbm{1}\brc*{s=s_m} \begin{pmatrix*} \inner*{ A \hatqHpik(s_m,\cdot,m)\mathbbm{1}\{\cdot = a_m\}+\lambda \nabla \omegaS{s_m}{\pik},\pi(\cdot \mid s_m) - \pi_k (\cdot \mid s_m)} \\+ \frac{1}{t_k}  \bregmanS{s_m}{\pi}{\pi_k} \end{pmatrix*}} ,
\end{align*}
where $s_m\sim d_{\nu,\pi_k}(\cdot)$ , $a_m\sim U(\mathcal{A})$, and $\hatqHpik(s_m,a_m,m)$ is the truncated Monte Carlo estimator of $q_\lambda^{\pi_k}(s_m,a_m)$ in the $m$-th trajectory. The notation $\hatqHpik(s_m,\cdot,m)\mathbbm{1}\brc*{\cdot=a_m}$ is a vector with the estimator value at the index $a_m$, and zero elsewhere. Also, we remind the reader we use the notation $A\triangleq |\mathcal{A}|$.
We can obtain a sample $s_m\sim d_{\nu,\pi_k}(\cdot)$ by a similar process as described in~\cite{kakade2002approximately,kakade2003sample}. Draw a start state $s$ from the $\nu$-restart distribution. Then, $s_m=s$ is chosen w.p. $\gamma$. Otherwise, w.p. $1-\gamma$, an action is sampled according to $a\sim\pi_k(s)$ to receive the next state $s$. This process is repeated until $s_m$ is chosen. If the time $T=\frac{1}{1-\gamma}\log \frac{\epsilon}{8 \HoeffdingBound(k,\lambda)}$ is reached, we accept the current state as $s_m$. Note that $\HoeffdingBound(k,\lambda)$ is defined in Lemma~\ref{lemma: one step approximation error bound unbiased}, and $\epsilon$ is the required final error. Finally, when $s_m$ is chosen, an action $a_m$ is drawn from the uniform distribution, and then the trajectory is unrolled using the current policy $\pi_k$ for $T=\frac{1}{1-\gamma}\log \frac{\epsilon}{8 \HoeffdingBound(k,\lambda)}$ time-steps, to calculate $\hatqHpik(s_m,a_m,m)$. Note that this introduces a bias into the estimation of $\qHpik$ \cite{kakade2003sample}[Sections 2.3.3 and 7.3.4]. Lastly, note that the $A$ factor in the estimator is due to importance sampling. 
% Finally, using a similar analysis to \cite{kakade2003sample}, one can truncate the trajectory in order to overcome the problem of infinite horizon MDPs.

First, the update rule of Sample-Based TRPO can be written as a state-wise update rule for any $s \in \sset$. Observe that,
\begin{align}
    \pi_{k+1} &\in \argmin_{\pi\in \ssimplex} \brc*{ \sum_{m=1}^{M}  \inner*{ A \hatqHpik(s_m,\cdot,m)\mathbbm{1}\{\cdot = a_m\} +\lambda \nabla \omegaS{s_m}{\pik} , \pi(\cdot \mid s_m) - \pi_k (\cdot \mid s_m)}  + \frac{1}{t_k}  \bregmanS{s_m}{\pi}{\pi_k} } \nonumber \\
    &= \argmin_{\pi\in \ssimplex} \brc*{ \sum_{s\in\mathcal{S}}\sum_{m=1}^{M}\mathbbm{1}\brc*{s=s_m} \begin{pmatrix}  \inner*{  A \hatqHpik(s_m,\cdot,m)\mathbbm{1}\{\cdot = a_m\} +\lambda \nabla \omegaS{s_m}{\pik} , \pi(\cdot \mid s_m) - \pi_k (\cdot \mid s_m)} \\ +  \frac{1}{t_k}  \bregmanS{s_m}{\pi}{\pi_k} \end{pmatrix}}, \label{eq: state-wise approximate update rule} 
\end{align}
The first relation is the definition of the update rule~\eqref{eq: approximate scalar TRPO} without the constant factor $\frac{1}{M}$. See that multiplying the optimization problem by the constant $M$ does not change the minimizer. In the second relation we used the fact that summation on $\sum_{s} \mathbbm{1}\brc*{s=s_m}$ leaves the optimization problem unchanged (as the indicator function is $0$ for all states that are not $s_m$).

Thus, using this update rule we can solve the optimization problem individually per $s\in \sset$,
\begin{align}
    \pi_{k+1}(\cdot|s) &= \argmin_{\pi\in \simplex} \brc*{\sum_{m=1}^{M}\mathbbm{1}\brc*{s=s_m} \begin{pmatrix}  \inner*{  A \hatqHpik(s,\cdot,m)\mathbbm{1}\{\cdot = a_m\} +\lambda \nabla \omegaS{s}{\pik} , \pi - \pi_k (\cdot \mid s)} +  \frac{1}{t_k}  \bregmanS{s}{\pi}{\pi_k} \end{pmatrix}}. \label{eq: state-wise approximate update rule final}
\end{align}
Note that using this representation optimization problem, the solution for states which were not encountered in the $k$-th iteration, $s \notin \mathcal{S}_M^k$, is arbitrary. To be consistent, we always choose to keep the current policy, $\pi_{k+1}(\cdot\mid s) = \pi_k(\cdot\mid s)$.

Now, similarly to Uniform and Exact TRPO, the update rule of Sample-Based TRPO can be written such that the optimization problem is solved individually per visited state $s\in \mathcal{S}^k_M$. This results in the final update rule used in Algorithm~\ref{alg: mu Sample-Based TRPO}.

To prove this, let $n(s)=\sum_{a}n(s,a)$ be the number of times the state $s$ was observed at the $k$-th episode. Using this notation and \eqref{eq: state-wise approximate update rule}, the update rule has the following equivalent forms,

\begin{align}
    \pi_{k+1} &\in \argmin_{\pi\in \ssimplex} \brc*{ \sum_{m=1}^{M}  \inner*{ A \hatqHpik(s_m,\cdot,m)\mathbbm{1}\{\cdot = a_m\} +\lambda \nabla \omegaS{s_m}{\pik} , \pi(\cdot \mid s_m) - \pi_k (\cdot \mid s_m)}  + \frac{1}{t_k}  \bregmanS{s_m}{\pi}{\pi_k} } \nonumber \\
    &= \argmin_{\pi\in \ssimplex} \brc*{ \sum_{s\in\mathcal{S}}\sum_{m=1}^{M}\mathbbm{1}\brc*{s=s_m} \begin{pmatrix}  \inner*{  A \hatqHpik(s_m,\cdot,m)\mathbbm{1}\{\cdot = a_m\} +\lambda \nabla \omegaS{s_m}{\pik} , \pi(\cdot \mid s_m) - \pi_k (\cdot \mid s_m)} \\ +  \frac{1}{t_k}  \bregmanS{s_m}{\pi}{\pi_k} \end{pmatrix}} \nonumber \\
    &= \argmin_{\pi\in \ssimplex} \brc*{ \sum_{s\in\mathcal{S}} \begin{pmatrix}  \inner*{  \sum_{m=1}^{M}\mathbbm{1}\brc*{s=s_m} A \hatqHpik(s_m,\cdot,m)\mathbbm{1}\{\cdot = a_m\} +n(s)\lambda \nabla \omegaS{s}{\pik} , \pi(\cdot \mid s) - \pi_k (\cdot \mid s)} \\ + n(s) \frac{1}{t_k}  \bregmanS{s}{\pi}{\pi_k} \end{pmatrix}} \nonumber\\
    &= \argmin_{\pi\in \ssimplex} \brc*{ \sum_{s\in\mathcal{S}_M^k}  \begin{pmatrix}  \inner*{ \sum_{m=1}^{M}\mathbbm{1}\brc*{s=s_m} A \hatqHpik(s_m,\cdot,m)\mathbbm{1}\{\cdot = a_m\} +n(s) \lambda \nabla \omegaS{s}{\pik} , \pi(\cdot \mid s) - \pi_k (\cdot \mid s)}\\  + n(s) \frac{1}{t_k}  \bregmanS{s}{\pi}{\pi_k} \end{pmatrix}} \nonumber\\
    &= \argmin_{\pi\in \ssimplex} \brc*{ \sum_{s\in\mathcal{S}_M^k} \begin{pmatrix}  \inner*{ \frac{1}{n(s)} \sum_{m=1}^{M}\mathbbm{1}\brc*{s=s_m} A \hatqHpik(s_m,\cdot,m)\mathbbm{1}\{\cdot = a_m\} +\lambda \nabla \omegaS{s}{\pik} , \pi(\cdot \mid s) - \pi_k (\cdot \mid s)} \\ + \frac{1}{t_k}  \bregmanS{s}{\pi}{\pi_k} \end{pmatrix}}.\label{supp eq: approximate simplified first}
\end{align}
In the third relation we used the fact for any $\pi,\pi_k$
\begin{align*}
    \sum_{s}\sum_{m=1}^{M} \bregmanS{s_m}{\pi}{\pi_k} \mathbbm{1}\brc*{s=s_m} &=\sum_{s}\bregmanS{s}{\pi}{\pi_k} \sum_{m=1}^{M} \mathbbm{1}  \brc*{s=s_m}=\sum_{s}\bregmanS{s}{\pi}{\pi_k}n(s).
\end{align*}
The fourth relation holds as the optimization problem is not affected by $s\notin \mathcal{S}_M^k$, and the last relation holds by dividing by $n(s)>0$ as $s\in\mathcal{S}_M^k$ and using linearity of inner product.

Lastly, we observe that~\eqref{supp eq: approximate simplified first} is a sum of functions of $\pi(\cdot\mid s)$, i.e., 
\begin{align*}
    \pi_{k+1} &\in \arg\min_{\pi\in \ssimplex}\brc*{ \sum_{s\in \mathcal{S}_M^k} f(\pi(\cdot\mid s))},
\end{align*}
where $f= \inner{g_s,\pi(\cdot\mid s)}+\frac{1}{t_k}  \bregmanS{s}{\pi}{\pi_k}$, $g_s\in \mathbb{R}^A$ is the vector inside the inner product of~\eqref{supp eq: approximate simplified first}. Meaning, the minimization problem is a sum of independent summands. Thus, in order to minimize the function on $\ssimplex$ it is enough to minimize independently each one of the summands. From this observation, we conclude that the update rule~\eqref{eq: approximate scalar TRPO} is equivalent to update the policy for all $s\in\mathcal{S}_M^k$ by
\begin{align}\label{eq: update rule approximate final}
    \pi_{k+1}(\cdot\mid s) &\in \arg\min_{\pi\in \simplex}\brc*{ \begin{pmatrix} \inner*{ \frac{1}{n(s)} \sum_{m=1}^{M}\mathbbm{1}\brc*{s=s_m} A \hatqHpik(s_m,\cdot,m)\mathbbm{1}\{\cdot = a_m\} +\lambda \nabla \omegaS{s}{\pik} , \pi - \pi_k (\cdot \mid s)} \\ + \frac{1}{t_k}  \bregmanS{s}{\pi}{\pi_k} \end{pmatrix} },
\end{align}
% or, equivalently, \todo{this is not true}
% \begin{align*}
%     \pi_{k+1}(\cdot\mid s) &\in \arg\min_{\pi\in \simplex}\brc*{  t_k\inner*{ \frac{1}{n(s)} \sum_{m=1}^{M}\mathbbm{1}\brc*{s=s_m} A \hatqHpik(s_m,\cdot,m)\mathbbm{1}\{\cdot = a_m\} , \pi} + \br*{1-\lambda t_k} \bregmanS{s}{\pi}{\pi_k}  }.
% \end{align*}
% \begin{align}
%     \pi_{k+1}(\cdot\mid s) &\in \arg\min_{\pi\in \simplex}\brc*{ t_k \inner*{ \frac{1}{n(s)} \sum_{i=1}^{n(s,a)} A \hatqHpik(s,\cdot,m_i) , \pi}  + \br*{1-\lambda t_k}  \bregmanS{s}{\pi}{\pi_k} }.
% \end{align}
Finally, by plugging in $\hat{q}^{\pi_k}_\lambda(s,a)=\frac{1}{ n(s)}\sum_{i=1}^{n(s,a)} \hat{q}_\lambda^{\pi_k}(s,a, m_i)$, we get
\begin{align*}
    \pi_{k+1}(\cdot\mid s) &\in \arg\min_{\pi\in \simplex}\brc*{ t_k \inner*{ \hatqHpik(s,\cdot) + \lambda \nabla\omegaS{s}{\pi_k}, \pi}  + \bregmanS{s}{\pi}{\pi_k} },
\end{align*}
where $m_i$ is the trajectory index of the $i$-th occurrence of the state $s$.

\subsection{Proof Sketch of Theorem~\ref{theorem: Sample-Based TRPO All}}\label{supp: proof sketch Sample-Based TRPO}

In order to keep things organized for an easy reading, we first go through the proof sketch in high level, which serves as map for reading the proof of Theorem~\ref{theorem: Sample-Based TRPO All} in the following sections.

\begin{enumerate}
    \item We use the Sample-Based TRPO optimization problem described in \ref{supp: approximate update rule TRPO}, to derive a fundamental inequality in Lemma~\ref{lemma: mu approximate fundamental inequality} for the sample-based case (in Appendix~\ref{supp: Sample-Based TRPO funamental inequality}): 
    
    \begin{enumerate}
    
        \item We derive a state-wise inequality by applying similar analysis to Exact TRPO, but for the Sample-Based TRPO optimization problem. By adding and subtracting a term which is similar to \eqref{eq: lemma improvement stochastic eq 1} in the state-wise inequality of Exact TRPO (Lemma~\ref{lemma: mu exact md point-wise inequality}), we write this inequality as a sum between the expected error and an approximation error term. 
        \item For each state, we employ importance sampling of $\frac{d_{\mu,\pi^*}(s)}{d_{\nu,\pi_k}(s)}$ to relate the derived state-wise inequality, to a global guarantee w.r.t. the optimal policy $\pi^*$ and measure $\mu$. This importance sampling procedure is allowed by assumption~\ref{assumption: stochastic MD}, which states that for any $s$ such that $d_{\mu, \pi^*}(s)>0$ it also holds that $\nu(s)>0$, and thus $d_{\nu, \pi_k}(s)>0$ since  $d_{\nu, \pi_k}(s) \geq (1-\gamma)\nu(s)$.
        
        \item By summing over all states we get the required fundamental inequality which resembles the fundamental inequality of Exact TRPO with an additional term due to the approximation error.
    \end{enumerate}

    \item In Appendix \ref{supp: Approximation Error bounds}, we show that the approximation error term is made of two sources of errors: (a) a sampling error due to the finite number of trajectories in each iteration; (b) a truncation error due to the finite length of each trajectory, even in the infinite-horizon case.
    \begin{enumerate} 
    \item In Lemma~\ref{lemma: one step approximation error bound unbiased} we deal with the sampling error. We show that this error is caused by the difference between an empirical mean of i.i.d. random variables and their expected value. Using Lemma~\ref{lemma: usefull bounds approximate} and Lemma~\ref{lemma: difference of gradient of omegas approximate}, we show that these random variables are bounded, and also that they are proportional to the step size $t_k$. Then, similarly to \cite{kakade2003sample}, we use Hoeffding's inequality and the union bound over the policy space (in our case, the space of deterministic policies), in order to bound this error term uniformly. This enables us to find the number of trajectories needed in the $k$-th iteration to reach an error proportional to $C^{\pi^*}t_k\epsilon = \norm{\frac{d_{\mu,\pi^*}}{\nu}}_\infty t_k \epsilon$ with high probability. The common concentration efficient $C^{\pi^*}$, arises due to  $\frac{d_{\mu,\pi^*}(s)}{d_{\nu,\pi_k}(s)}$, the importance sampling ratio used for the global convergence guarantee. 
    \item In Lemma~\ref{lemma: truncation error bound} we deal with the truncation error. We show that we can bound this error to be proportional to $C^{\pi^*}t_k\epsilon$, by using $O\br*{\frac{1}{1-\gamma}}$ samples in each trajectory.
    \end{enumerate}
    Finally, in Lemma~\ref{lemma: cumulative approximation error}, we use the union bound over all $k\in \mathbb{N}$ in order to uniformly bound the error propagation over $N$ iterations of Sample-Based TRPO.
    
    \item In Appendix~\ref{supp: proof of Sample-Based TRPO} we use a similar analysis to the one used for the rates guarantees of Exact TRPO (Appendix~\ref{supp: exact TRPO unregularized proof}), using the above results. The only difference is the approximation term which we bound in~\ref{supp: Approximation Error bounds}. There, we make use of the fact that the approximation term is proportional to the step size $t_k$ and thus decreasing with the number of iterations, to prove a bounded approximation error for any $N$. 
    % This concludes the proof of Theorem~\ref{theorem: Sample-Based TRPO All}.
    \item Lastly, in Appendix~\ref{supp: sample complexity of Approximate Mirror Descent}, we calculate the overall sample complexity -- previously we bounded the number of needed iterations and the number of samples needed in every iteartion --  for each of the four cases of Sample-Based TRPO (euclidean vs. non-euclidean, unregularized vs. regularized).
\end{enumerate}

\subsection{Fundamental Inequality of Sample-Based TRPO}\label{supp: Sample-Based TRPO funamental inequality}

\begin{lemma}[sample-based state-wise inequality]\label{lemma: mu approximate md point-wise inequality} 
Let $\{\pi_k\}_{k\geq 0}$ be the sequence generated by Aproximate TRPO using stepsizes $\{t_k\}_{k\geq0}$.  Then, for all states $s$ for which $d_{\nu,\pi_k}(s)>0$ the following inequality holds for all $\pi\in \ssimplex$,
\begin{align*}
        0 & \leq t_k\br*{T_\lambda^{\pi}v_\lambda^{\pi_k}(s) - v_\lambda^{\pi_k}(s)} + \frac{t_k^2 \gradBound^2(k;\lambda)}{2} +(1-\lambda t_k) \bregmanS{s}{\pi}{\pi_k} - \bregmanS{s}{ \pi }{\pi_{k+1}} + \epsilon_k(s,
        \pi).
\end{align*}
where $h_\omega$ is defined at the third claim of Lemma \ref{lemma: usefull bounds approximate}. 
\end{lemma}

\begin{proof}

Using the first order optimality condition for the update rule \eqref{eq: state-wise approximate update rule final}, the following holds for any $s\in\mathcal{S}$ and thus for any $s\in \brc*{s' : d_{\nu,\pi_k}(s)>0}$,
\begin{align*}
    0 &\leq \frac{1}{M} \sum_{m=1}^{M} \mathbbm{1}\brc*{s=s_m} \inner*{ t_k\left(A \hatqHpik(s_m,\cdot,m)\mathbbm{1}\brc*{\cdot=a_m}+\lambda \nabla\omegaS{s_m}{\pi_k}\right)+   \nabla_{\pi_{k+1}} \bregmanS{s_m}{\pi_{k+1}}{\pi_k}, \pi - \pi_{k+1}(\cdot \mid s_m) }.
\end{align*}

Dividing by $d_{\nu,\pi_k}(s)$ which is strictly positive for all $s$ such that $\mathbbm{1}\brc*{s=s_m}=1$ and  adding and subtracting the term $$\inner*{ t_k\left(\qHpik(s,\cdot)+\lambda \nabla\omegaS{s}{\pi_k}\right)+   \nabla_{\pi_{k+1}} \bregmanS{s}{\pi_{k+1}}{\pi_k}, \pi - \pi_{k+1}(\cdot \mid s) },$$
we get
\begin{align}
    0 \leq & \underbrace{\inner*{ t_k\left(\qHpik(s,\cdot)+\lambda \nabla\omegaS{s}{\pi_k}\right)+   \nabla_{\pi_{k+1}} \bregmanS{s}{\pi_{k+1}}{\pi_k}, \pi - \pi_{k+1}(\cdot \mid s) }}_{(*)} + \epsilon_k (s,\pi)
    \label{eq: first term in state wise approximate inequality},
\end{align}
where we defined $\epsilon_k(s,\pi)$,
\begin{align}
    & \epsilon_k(s,\pi) \nonumber \\ 
    &\triangleq \frac{1}{d_{\nu,\pi_k}(s)} \frac{1}{M}\sum_{m=1}^{M} \! \mathbbm{1}\brc*{s=s_m} \inner*{ t_k\left(A\hatqHpik(s_m,\cdot,m)\!\mathbbm{1}\brc*{\cdot=a_m} \! +\!\lambda \nabla\omegaS{s_m}{\pi_k}\right)\!+  \! \nabla_{\pi_{k+1}} \bregmanS{s_m}{\pi_{k+1}}{\pi_k}, \pi \!-\! \pi_{k+1}(\cdot \!\mid\! s_m) } \nonumber\\
    &-\inner*{ t_k\left(\qHpik(s,\cdot)+\lambda \nabla\omegaS{s}{\pi_k}\right)+   \nabla_{\pi_{k+1}} \bregmanS{s}{\pi_{k+1}}{\pi_k}, \pi - \pi_{k+1}(\cdot \mid s) } \nonumber \\
    &=\frac{1}{d_{\nu,\pi_k}(s)} \frac{1}{M}\!\sum_{m=1}^{M} \! \mathbbm{1}\brc*{s=s_m} \inner*{ t_k\left(A\hatqHpik(s_m,\cdot,m)\mathbbm{1}\brc*{\cdot=a_m} \! + \!\lambda \nabla\omegaS{s_m}{\pi_k}\right) \! + \! \nabla \omegaS{s}{\pi_{k+1}} \! - \! \nabla \omegaS{s}{\pik}, \pi \! - \! \pi_{k+1}(\cdot \! \mid \! s_m) }\nonumber\\
    &-\inner*{ t_k\left(\qHpik(s,\cdot)+\lambda \nabla\omegaS{s}{\pi_k}\right)+   \nabla \omegaS{s}{\pi_{k+1}} -  \nabla \omegaS{s}{\pik}, \pi - \pi_{k+1}(\cdot \mid s) }.\label{def: epsilon in approximate case}
\end{align}

By bounding $(*)$ in \eqref{eq: first term in state wise approximate inequality} using the exact same analysis of Lemma~\ref{lemma: mu exact md point-wise inequality} we conclude the proof.
\end{proof}

Now, we state another lemma which connects the state-wise inequality using the discounted stationary distribution of the optimal policy $d_{\mu,\pi^*}$, similarly to Lemma~\ref{lemma: toward the mu exact fundamental inequality}. 

\begin{lemma}\label{lemma: toward the mu approximate fundamental inequality} 
Let Assumption~\ref{assumption: stochastic MD} hold and let $\{\pi_k\}_{k\geq 0}$ be the sequence generated by Aproximate TRPO using stepsizes $\{t_k\}_{k\geq0}$. Then, for all $k\geq 0$ Then, the following inequality holds for all $\pi$,
\begin{align*}
        0 & \leq  t_kd_{\mu,\pi^*} \br*{T_\lambda^{\pi}v_\lambda^{\pi_k} - v_\lambda^{\pi_k}} + \frac{t_k^2 \gradBound^2(k;\lambda)}{2} +(1-\lambda t_k)d_{\mu,\pi^*} \bregman{\pi}{\pi_k} - d_{\mu,\pi^*}\bregman{ \pi }{\pi_{k+1}} +d_{\mu,\pi^*}\epsilon_k(\cdot,\pi).
\end{align*}
where $h_\omega$ is defined in the third claim of Lemma \ref{lemma: usefull bounds approximate}. 
\begin{proof}
    By Assumption~\ref{assumption: stochastic MD}, for all $s$ for which $d_{\mu,\pi^*}(s)>0$ it also holds that $d_{\nu,\pi_k}(s)>0$. Thus, for all $s$ for which $d_{\mu,\pi^*}(s)>0$ the component-wise relation in Lemma~\ref{lemma: mu approximate md point-wise inequality}  holds. By multiplying each inequality by the positive number $d_{\mu,\pi^*}(s)$ and summing over all $s$ we get,
    \begin{align*}
        0 & \leq  t_kd_{\mu,\pi^*} \br*{T_\lambda^{\pi}v_\lambda^{\pi_k} - v_\lambda^{\pi_k}} + \frac{t_k^2 \gradBound^2(k;\lambda)}{2} +(1-\lambda t_k)d_{\mu,\pi^*} \bregman{\pi}{\pi_k} - d_{\mu,\pi^*}\bregman{ \pi }{\pi_{k+1}} +d_{\mu,\pi^*}\epsilon_k(\cdot,\pi),
\end{align*}
which concludes the proof.

\end{proof}
\end{lemma}

\begin{lemma}[fundamental inequality of Sample-Based TRPO.]\label{lemma: mu approximate fundamental inequality}
Let $\{\pi_k\}_{k\geq 0}$ be the sequence generated by Aproximate TRPO using stepsizes $\{t_k\}_{k\geq0}$.  Then, for all $k\geq 0$

\begin{align*}
&t_k(1-\gamma) (\mu v^{\pi_k}_\lambda- \mu v^{\pi^*}_\lambda) \leq d_{\mu,\pi^*}\br*{ (1-\lambda t_k)\bregman{\pi^*}{\pi_k} - \bregman{ \pi^* }{\pi_{k+1}}} +\frac{t_k^2 \gradBound^2(k;\lambda)}{2}+d_{\mu,\pi^*}\epsilon_k,
\end{align*}
where $\gradBound(k;\lambda)$ is defined in Lemma~\ref{lemma: usefull bounds approximate} and $\epsilon_k\triangleq \epsilon_k(\cdot,\pi^*)$ where the latter defined in~\eqref{def: epsilon in approximate case}.
\end{lemma}
\begin{proof}
Setting $\pi=\pi^*$ in Lemma \ref{lemma: toward the mu approximate fundamental inequality} and denoting $\epsilon_k\triangleq \epsilon_k(\cdot,\pi^*)$, we get that for any $k$,
\begin{align*}
    &-t_k d_{\mu,\pi^*}\br*{T^{\pi^*}_\lambda \vHpik -  \vHpik} \\
    &\leq  d_{\mu,\pi^*}\br{ (1-\lambda t_k)\bregman{\pi^*}{\pi_k} - \bregman{\pi^* }{\pi_{k+1}}} +\frac{t_k^2 \gradBound^2(k;\lambda)}{2} + d_{\mu,\pi^*}\epsilon_k.
\end{align*}

Furthermore, by the third claim of Lemma~\ref{lemma: greedy to difference},
\begin{align*}
    (1-\gamma)\mu(v_\lambda^*-\vHpik) = d_{\mu,\pi^*}\br*{T^{\pi^*}_\lambda \vHpik -  \vHpik}.
\end{align*}

Combining the two relations on both sides we concludes the proof.

\end{proof}

\subsection{Approximation Error Bound}\label{supp: Approximation Error bounds}

In this section we deal with the approximation error, the term $d_{\mu,\pi^*}\epsilon_k$ in Lemma~\ref{lemma: mu approximate fundamental inequality}. Two factors effects $d_{\mu,\pi^*}\epsilon_k$: (1) the error due to Monte-Carlo sampling, which we bound using Hoeffding's inequality and the union bound; (2) the error due to the truncation in the sampling process (see Appendix~\ref{supp: approximate update rule TRPO}). The next two lemmas bound these two sources of error. We first discuss the analysis of using an unbiased sampling process (Lemma~\ref{lemma: one step approximation error bound unbiased}), i.e., when no truncation is taking place, and then move to discuss the use of the truncated trajectories (Lemma~\ref{lemma: truncation error bound}). Finally, in Lemma~\ref{lemma: one step approximation error bound} we combine the two results to bound $d_{\mu,\pi^*}\epsilon_k$ in the case of the full truncated sampling process discussed in Appendix~\ref{supp: approximate update rule TRPO}.

The unbiased $q$-function estimator uses a full unrolling of a trajectory, i.e., calculates the (possibly infinite) sum of retrieved costs following the policy $\pi_k$ in the $m$-th trajecotry of the $k$-th iteration,
$$ \hatqHpik(s_m,a_m,m) \triangleq \sum_{t=0}^{\infty} {\gamma^t \br*{c\br*{s^{k,m}_t,a^{k,m}_t} + \lambda \omegaS{s^{k,m}_t}{\pi_k}}},$$
where the notation $s^{k,m}_t$ refer to the state encountered in the $m$-th trajectory of the $k$-th iteration, at the $t$ step of estimating the $\qHpik$ function. Moreover, $(s_m,a_m)=(s^{k,m}_0,a^{k,m}_0)$ and $\hatqHpik(s,a,m)=0$ for any $(s,a) \neq (s_m,a_m) $.

The truncated biased $q$-function estimator, truncates the trajectory after $T$ interactions with the MDP, where $T$ is predefined:
$$ \hat q_{\lambda,\text{trunc}}^{\pi_k}(s,a,m) \triangleq \sum_{t=0}^{T-1} \gamma^t \br*{c\br*{s^{k,m}_t,a^{k,m}_t} + \lambda \omegaS{s^{k,m}_t}{\pi_k}}  $$

The following lemma describes the number of trajectories needed in the $k$-th update, in order to bound the error to be proportional to $\epsilon$ w.p. $1-\delta'$, using an unbiased estimator.

\begin{lemma}[Approximation error bound with unbiased sampling]\label{lemma: one step approximation error bound unbiased}
For any $\epsilon,\tilde \delta>0$, if the number of trajectories in the $k$-th iteration is
$$ M_k \geq \frac{2\HoeffdingBound(k,\lambda)^2}{\epsilon^2} \br*{ S \log {2A} + \log{1 / \tilde \delta}},$$
then with probability of $1- \tilde \delta$,
\begin{align*}
    d_{\mu,\pi^*}\epsilon_k \leq t_k \norm{ \frac{d_{\mu,\pi^*}}{d_{\nu,\pi_k}} }_\infty \frac{\epsilon}{2},
\end{align*}
where $\HoeffdingBound(k,\lambda)=\frac{4  A\CmaxLambda}{1-\gamma}$ and $\HoeffdingBound(k,\lambda)=  \frac{4 A\CmaxLambda}{1-\gamma} \br*{1+ \mathbbm{1}\{\lambda \neq 0 \} \log k}$ in the~euclidean and non-euclidean settings respectively.

\begin{proof}
Plugging the definition of $\epsilon_k\triangleq\epsilon_k(\cdot,\pi^*)$ in \eqref{def: epsilon in approximate case}, we get,
\begin{align*}
    &d_{\mu,\pi^*}\epsilon_k \\
    &= \sum_s \! \frac{d_{\mu,\pi^*}(s)}{M_k d_{\nu,\pi_k}(s)}  \! \sum_{m=1}^{M_k} \mathbbm{1}\brc*{s=s_m} \inner*{ t_k\left(A\hatqHpik(s_m,\cdot,m)\!+ \! \lambda \nabla\omegaS{s_m}{\pi_k}\right) \! + \!   \nabla \omegaS{s}{\pi_{k+1}}\! - \! \nabla \omegaS{s}{\pik},\pi^*(\cdot \! \mid \! s) \! - \! \pi_{k+1}(\cdot \! \mid  \! s_m) }\\
    &- \sum_s d_{\mu,\pi^*}(s) \inner*{ t_k\left(\qHpik(s,\cdot)+\lambda \nabla\omegaS{s}{\pi_k}\right)+   \nabla \omegaS{s}{\pi_{k+1}} -  \nabla \omegaS{s}{\pik}, \pi^*(\cdot \mid s) - \pi_{k+1}(\cdot \mid s) }\\
    &= \frac{1}{M_k} \! \sum_{m=1}^{M_k} \! \sum_s \mathbbm{1} \brc*{s=s_m} \frac{d_{\mu,\pi^*}(s)}{ d_{\nu,\pi_k}(s)}  \inner*{ t_k\left(A\hatqHpik(s_m,\cdot,m) \!+ \!\lambda \nabla\omegaS{s_m}{\pi_k}\right) \! + \!  \nabla \omegaS{s}{\pi_{k+1}} \! - \! \nabla \omegaS{s}{\pik}, \pi^*(\cdot \! \mid \! s) \! - \! \pi_{k+1}(\cdot \! \mid \! s_m) }\\
    &- \sum_s d_{\nu,\pi_k}(s) \frac{d_{\mu,\pi^*}(s)}{d_{\nu,\pi_k}(s)} \inner*{ t_k\left(\qHpik(s,\cdot)+\lambda \nabla\omegaS{s}{\pi_k}\right)+   \nabla \omegaS{s}{\pi_{k+1}} -  \nabla \omegaS{s}{\pik}, \pi^*(\cdot \mid s) - \pi_{k+1}(\cdot \mid s) },
\end{align*}
where in the last transition we used the fact that for every $s \neq s_m$ the identity function $\mathbbm{1}\brc*{s=s_m}=0$.

We define, 
\begin{align}
    &\hat{X}_k(s_m,\cdot,m) \triangleq  t_k\left(A\hatqHpik(s_m,\cdot,m)+\lambda \nabla\omegaS{s_m}{\pi_k}\right)+  \nabla \omegaS{s_m}{\pi_{k+1}} -  \nabla \omegaS{s_m}{\pik},\label{eq: hat X def}\\
    &X_k(s,\cdot) \triangleq  t_k\left(\qHpik(s,\cdot)+\lambda \nabla\omegaS{s}{\pi_k}\right)+   \nabla \omegaS{s}{\pi_{k+1}} -  \nabla \omegaS{s}{\pik} \label{eq: x def}.
\end{align}

Using this definition, we have,
\begin{align}
    d_{\mu,\pi^*}\epsilon_k =& \frac{1}{M_k}\sum_{m=1}^{M_k}\sum_s \mathbbm{1}\brc*{s=s_m} \frac{d_{\mu,\pi^*}(s)}{ d_{\nu,\pi_k}(s)}  \inner*{ \hat{X}_k(s_m,\cdot,m), \pi^*(\cdot \mid s_m) - \pi_{k+1}(\cdot \mid s_m) } \nonumber\\
    &- \sum_s d_{\mu,\pi^*}(s) \inner*{ X_k(s,\cdot), \pi^*(\cdot \mid s) - \pi_{k+1}(\cdot \mid s) }.   \label{eq: bounding the approximation error 1}
\end{align}
In order to remove the dependency on the randomness of $\pi_{k+1}$, we can bound this term in a uniform way:
\begin{align}
    d_{\mu,\pi^*}\epsilon_k \leq & \max_{\pi'} \Big\{ \frac{1}{M_k}\sum_{m=1}^{M_k}\sum_s \mathbbm{1}\brc*{s=s_m} \frac{d_{\mu,\pi^*}(s)}{ d_{\nu,\pi_k}(s)}  \inner*{ \hat{X}_k(s_m,\cdot,m), \pi^*(\cdot \mid s_m) - \pi'(\cdot \mid s_m) } \nonumber\\
    &- \sum_s d_{\mu,\pi^*}(s) \inner*{ X_k(s_m,\cdot),\pi^*(\cdot \mid s)- \pi'(\cdot \mid s) } \Big\}.  \label{eq: bounding the approximation error 2}
\end{align}

In this lemma, we analyze the case where no truncation is taken into account. In this case we, we will now show that for any $\pi'$
$$ \E \brs*{\sum_s \mathbbm{1}\brc*{s=s_m} \frac{d_{\mu,\pi^*}(s)}{ d_{\nu,\pi_k}(s)}  \inner*{ \hat{X}_k(s,\cdot,m), \pi^*(\cdot \mid s_m)- \pi'(\cdot \mid s_m) }} =\sum_s d_{\mu,\pi^*}(s) \inner*{ X_k(s,\cdot), \pi^*(\cdot \mid s) - \pi'(\cdot \mid s) }, $$
which means that $\frac{1}{M_k}\sum_{m=1}^{M_k}\sum_s \mathbbm{1}\brc*{s=s_m} \frac{d_{\mu,\pi^*}(s)}{ d_{\nu,\pi_k}(s)}  \inner*{ \hat{X}_k(s,\cdot,m), \pi^*(\cdot \mid s_m) - \pi'(\cdot \mid s_m) } $ is an unbiased estimator.

This fact comes from the from the following relations:

\begin{align}\label{eq: expectation of sample}
    &\E[\sum_s \mathbbm{1}\brc*{s=s_m} \frac{d_{\mu,\pi^*}(s)}{ d_{\nu,\pi_k}(s)}  \inner*{ \hat{X}_k(s,\cdot,m), \pi^*(\cdot \mid s_m)- \pi'(\cdot \mid s_m) }] \nonumber\\
    &= \E\brs*{\E\brs*{\sum_s \mathbbm{1}\brc*{s=s_m} \frac{d_{\mu,\pi^*}(s)}{ d_{\nu,\pi_k}(s)}  \inner*{ \hat{X}_k(s,\cdot,m), \pi^*(\cdot \mid s) - \pi'(\cdot \mid s) }\mid s_m}} \nonumber \\
    &=\E\brs*{\E\brs*{\frac{d_{\mu,\pi^*}(s_m)}{ d_{\nu,\pi_k}(s_m)}  \inner*{ \hat{X}_k(s_m,\cdot,m), \pi^*(\cdot \mid s_m) - \pi'(\cdot \mid s_m) }\mid s_m}}\nonumber\\
    &=\E\brs*{\frac{d_{\mu,\pi^*}(s_m)}{ d_{\nu,\pi_k}(s_m)} \E\brs*{ \inner*{ \hat{X}_k(s_m,\cdot,m), \pi^*(\cdot \mid s_m) - \pi'(\cdot \mid s_m) }\mid s_m}}\nonumber\\
    &=\E\brs*{\frac{d_{\mu,\pi^*}(s_m)}{ d_{\nu,\pi_k}(s_m)}  \inner*{ \E\brs*{\hat{X}_k(s_m,\cdot,m)\mid s_m}, \pi^*(\cdot \mid s_m) - \pi'(\cdot \mid s_m) }}\nonumber\\
    &=\E\brs*{\frac{d_{\mu,\pi^*}(s_m)}{ d_{\nu,\pi_k}(s_m)}  \inner*{ X_k(s_m,\cdot), \pi^*(\cdot \mid s_m) - \pi'(\cdot \mid s_m) }}\nonumber\\
    &=\sum_s d_{\nu,\pi_k}(s)\frac{d_{\mu,\pi^*}(s)}{ d_{\nu,\pi_k}(s)}  \inner*{ X_k(s,\cdot), \pi^*(\cdot \mid s) - \pi'(\cdot \mid s) }\nonumber\\
    &=\sum_s d_{\mu,\pi^*}(s) \inner*{ X_k(s,\cdot), \pi^*(\cdot \mid s) - \pi'(\cdot \mid s) },
\end{align}
where the first transition is by law of total expectation; the second transition is by the fact the indicator function is zero for every $s \neq s_m$; the third transition is by the fact $s_m$ is not random given $s_m$; the fourth transition is by the linearity of expectation and the fact that $\pi^*(\cdot \mid s_m) - \pi'(\cdot \mid s_m)$ is not random given $s_m$; the fifth transition is by taking the expectation of $\hat X$ in the state $s_m$; finally, the sixth transition is by explicitly taking the expectation over the probability that $s_m$ is drawn from $d_{\nu,\pi_k}$ in the $m$-th trajectory (by following $\pi_k$ from the restart distribution $\nu$).

Meaning, \eqref{eq: bounding the approximation error 2} is a difference between an empirical mean of $M_k$ random variables and their mean for a the fixed policy $\pi'$, which maximizes the following expression
\begin{align}
    d_{\mu,\pi^*}\epsilon_k \leq & \max_{\pi'} \Bigg\{ \frac{1}{M_k}\sum_{m=1}^{M_k}\sum_s \mathbbm{1}\brc*{s=s_m} \frac{d_{\mu,\pi^*}(s)}{ d_{\nu,\pi_k}(s)}  \inner*{ \hat{X}_k(s,\cdot,m), \pi^*(\cdot \mid s) - \pi'(\cdot \mid s) } \nonumber\\
    &- \E\brs*{\sum_s \mathbbm{1}\brc*{s=s_m} \frac{d_{\mu,\pi^*}(s)}{ d_{\nu,\pi_k}(s)}  \inner*{ \hat{X}_k(s,\cdot,m), \pi^*(\cdot \mid s) - \pi'(\cdot \mid s) }}\Bigg\}. \label{eq: maximum bound over policies}
\end{align}

As we wish to obtain a uniform bound on $\pi'$, we can use the common approach of bounding \eqref{eq: maximum bound over policies} uniformly for all $\pi'\in \ssimplex$ using the union bound. Note that the above optimization problem is a linear programming optimization problem in $\pi'$, where $\pi'\in \ssimplex$. It is a well known fact that for linear programming, there is an extreme point which is the optimal solution of the problem \cite{bertsimas1997introduction}[Theorem 2.7]. The set of extreme points of $\ssimplex$ is the set of all deterministic policies denoted by $\Pi^{\text{det}}$. Therefore, in order to bound the maximum in \eqref{eq: maximum bound over policies}, it suffices to uniformly bound all policies $\pi' \in \Pi^{\text{det}}$.

Now, notice that $\frac{d_{\mu,\pi^*}(s_m)}{ d_{\nu,\pi_k}(s_m)}\inner*{\hat X_k(s_m,\cdot,m), \pi^*(\cdot \mid s_m) - \pi'(\cdot \mid s_m)}$ is bounded for all $s_m$ and $\pi'$,
\begin{align} \label{eq: bounded RV hat X}
& \frac{d_{\mu,\pi^*}(s_m)}{ d_{\nu,\pi_k}(s_m)}\inner*{\hat X_k(s_m,\cdot,m), \pi^*(\cdot \mid s_m) - \pi'(\cdot \mid s_m)} \nonumber \\
&  =\inner*{\frac{d_{\mu,\pi^*}(s_m)}{ d_{\nu,\pi_k}(s_m)}\hat X_k(s_m,\cdot,m), \pi^*(\cdot \mid s_m) - \pi'(\cdot \mid s_m)}  \nonumber \\
& \leq \norm{\frac{d_{\mu,\pi^*}(s_m)}{ d_{\nu,\pi_k}(s_m)}\hat X_k(s_m,\cdot,m)}_\infty \norm{\pi^*(\cdot \mid s_m) - \pi'(\cdot \mid s_m)}_1 \nonumber \\ 
& \leq 2 \frac{d_{\mu,\pi^*}(s_m)}{ d_{\nu,\pi_k}(s_m)}\norm{\hat X_k(s_m,\cdot,m)}_\infty  \nonumber \\
& \leq 2\norm{\frac{d_{\mu,\pi^*}}{ d_{\nu,\pi_k}}}_\infty\norm{\hat X_k(s_m,\cdot,m)}_\infty  \nonumber \\
& = 2 \norm{\frac{d_{\mu,\pi^*}}{ d_{\nu,\pi_k}}}_\infty\norm{ t_k\left(A\hatqHpik(s_m,\cdot,m)+\lambda \nabla\omegaS{s_m}{\pi_k}\right)+  \nabla \omegaS{s_m}{\pi_{k+1}} -  \nabla \omegaS{s_m}{\pik}}_\infty \nonumber \\
& \leq 2 \norm{\frac{d_{\mu,\pi^*}}{ d_{\nu,\pi_k}}}_\infty\br*{ t_k  \norm{A\hatqHpik(s_m,\cdot,m)+\lambda \nabla\omegaS{s_m}{\pi_k}}_\infty +   \norm{\nabla \omegaS{s_m}{\pi_{k+1}} -  \nabla \omegaS{s_m}{\pik}}_\infty} \nonumber \\
& \leq 2 t_k \norm{\frac{d_{\mu,\pi^*}}{ d_{\nu,\pi_k}}}_\infty \br*{ \hatgradBound(k;\lambda) + 2 \diffBound(k)} \nonumber \\
& = 2 \norm{\frac{d_{\mu,\pi^*}}{ d_{\nu,\pi_k}}}_\infty (t_k \hatgradBound(k;\lambda) + \diffBound(k)) \nonumber \\
& \triangleq t_k \norm{\frac{d_{\mu,\pi^*}}{ d_{\nu,\pi_k}}}_\infty \HoeffdingBound(k,\lambda),
\end{align}
where the second transition is due to H{\"o}lder's inequality; the third transition is due to the bound of the $TV$ distance between two random variables; the sixth transition is due to the triangle inequality; finally, the seventh transition is by plugging in the bounds in Lemma~\ref{lemma: usefull bounds approximate} and Lemma~\ref{lemma: difference of gradient of omegas approximate}. Also, we defined $\HoeffdingBound(k,\lambda)=\frac{4  A\CmaxLambda}{1-\gamma}$ and $\HoeffdingBound(k,\lambda)=  \frac{4  A\CmaxLambda}{1-\gamma} \br*{1+\mathbbm{1}\{ \lambda \neq 0 \} \log k }$ in the euclidean and non-euclidean cases respectively.

Thus, by Hoeffding and the union bound over the set of deterministic policies,
$$ P\br*{ d_{\mu,\pi^*}\epsilon_k \geq  t_k \norm{\frac{d_{\mu,\pi^*}}{ d_{\nu,\pi_k}}}_\infty \frac{\epsilon}{2}  } \leq 2|\Pi^{\text{det}}| \exp\br*{-\frac{M_k\epsilon^2}{2\HoeffdingBound(k,\lambda)^2}}=\tilde \delta .$$

% $$ P\br*{ d_{\mu,\pi^*} \epsilon_k \geq  t_k \norm{\frac{d_{\mu,\pi^*}}{ d_{\nu,\pi_k}}}_\infty \epsilon  } \leq 2|\Pi^{\text{det}}| \exp\br*{-\frac{2M_k\epsilon^2}{\HoeffdingBound(k,\lambda)^2}} .$$

In other words, in order to guarantee that 
\begin{align*}
d_{\mu,\pi^*}\epsilon_k \leq t_k \norm{\frac{d_{\mu,\pi^*}}{d_{\nu,\pi_k}}}_\infty \frac{\epsilon}{2},
\end{align*}
we need the number of trajectories $M_k$ to be at least
\begin{align*}
    & M_k \geq \frac{2\HoeffdingBound(k,\lambda)^2}{\epsilon^2} \br*{ S \log {2A} + \log{1 / \tilde \delta}},
\end{align*}
where we used the fact that there are $|\Pi^{\text{det}}|=A^S$ deterministic policies.

which concludes the result.

\end{proof}

\end{lemma}

The following lemma described with error due to the use of truncated trajectories:

\begin{lemma}[Truncation error bound]\label{lemma: truncation error bound}
The bias of the truncated sampling process in the $k$-th iteration, with maximal trajectory length of $T=\frac{1}{1-\gamma}\log \frac{\epsilon}{8 \HoeffdingBound(k,\lambda)}$ is $t_k  \norm{\frac{d_{\mu,\pi^*}}{d_{\nu,\pi_k}}}_\infty \frac{\epsilon}{4}$,
where $\HoeffdingBound(k,\lambda)=\frac{4  A\CmaxLambda}{1-\gamma}$ and $\HoeffdingBound(k,\lambda)=  \frac{2 A\CmaxLambda}{1-\gamma} \br*{\frac{1}{1-\lambda t_k}+1+\lambda \log k}$ in the~euclidean and non-euclidean settings respectively.
\begin{proof}

We start this proof by defining notation related to the truncated sampling process. 
First, denote $d_{\nu,\pi_k}^{\text{trunc}}(s)$, the probability to choose a state $s$, using the truncated biased sampling process of length $T$, as described in Appendix~\ref{supp: approximate update rule TRPO}. Observe that
$$ d_{\nu,\pi_k}^{\text{trunc}}(s) = (1-\gamma) \sum_{t=0}^{T-1} \gamma^t p(s_t=s \mid \nu, \pi_k) + \gamma^{T} p(s_{T}=s \mid \nu, \pi_k) $$

We also make use in this proof in the following definitions (see \eqref{eq: hat X def} and \eqref{eq: x def}),
\begin{align*}
    &\hat{X}_k(s_m,\cdot,m) \triangleq  t_k\left(A \hat q_{\lambda,\text{trunc}}^{\pi_k}(s_m,\cdot,m)+\lambda \nabla\omegaS{s_m}{\pi_k}\right)+  \nabla \omegaS{s_m}{\pi_{k+1}} -  \nabla \omegaS{s_m}{\pik},\\
    &X_k(s,\cdot) \triangleq  t_k\left(\qHpik(s,\cdot)+\lambda \nabla\omegaS{s}{\pi_k}\right)+   \nabla \omegaS{s}{\pi_{k+1}} -  \nabla \omegaS{s}{\pik} .
\end{align*}

Lastly, we denote the expectation of $\hat X_k (s,\cdot,m)$ using the truncated sampling process as $X^\text{trunc}_k(s,\cdot)$,
$$ X^{\text{trunc}}_k(s,a) = \E \hat X_k(s,a,m) $$

Now, we move on to the proof. We first split the bias to two different sources of bias:
    \begin{align*}
        & \E_{s\sim d_{\nu,\pi_k}^{\text{trunc}}} \frac{d_{\mu,\pi^*}(s)}{d_{\nu,\pi_k}(s)}\inner*{X_k^{\text{trunc}}(s,\cdot),\pi(\cdot\mid s)-\pi'(\cdot\mid s)} - \E_{s\sim d_{\nu,\pi_k}} \frac{d_{\mu,\pi^*}(s)}{d_{\nu,\pi_k}(s)}\inner*{X_k(s,\cdot),\pi(\cdot\mid s)-\pi'(\cdot\mid s)} \\ 
        &=\br*{\E_{s\sim d_{\nu,\pi_k}^{\text{trunc}}} \frac{d_{\mu,\pi^*}(s)}{d_{\nu,\pi_k}(s)}\inner*{X_k^{\text{trunc}}(s,\cdot),\pi(\cdot\mid s)-\pi'(\cdot\mid s)}  - \E_{s\sim d_{\nu,\pi_k}}\frac{d_{\mu,\pi^*}(s)}{d_{\nu,\pi_k}(s)}\inner*{X_k^{\text{trunc}}(s,\cdot),\pi(\cdot\mid s)-\pi'(\cdot\mid s)}} \\
        & + \br*{\E_{s\sim d_{\nu,\pi_k}}\frac{d_{\mu,\pi^*}(s)}{d_{\nu,\pi_k}(s)}\inner*{X_k^{\text{trunc}}(s,\cdot),\pi(\cdot\mid s)-\pi'(\cdot\mid s)}  - \E_{s\sim d_{\nu,\pi_k}}\frac{d_{\mu,\pi^*}(s)}{d_{\nu,\pi_k}(s)}\inner*{X_k(s,\cdot),\pi(\cdot\mid s)-\pi'(\cdot\mid s)}}.
    \end{align*}
    
    The first source of bias is due to the truncation of the state sampling after $T$ iterations, and the second source of bias is due to the truncation done in the estimation of $\qHpik(s,a)$, for the chosen state $s$ and action $a$.

    First, we bound the first error term.
    Observe that for any $s$,
    
    \begin{align}\label{eq: truncation of stationary distribution}
        \sum_s \abs*{ d_{\nu,\pi_k}^{\text{trunc}}(s) - d_{\nu,\pi_k}(s) }  &= \sum_s \abs*{(1-\gamma)\sum_{t=0}^{T-1} \gamma^t p(s_t=s\mid\nu,\pi_k) + \gamma^{T} p(s_{T}=s \mid \nu,\pi_k) - (1-\gamma)\sum_{t=0}^\infty \gamma^t p(s_t=s\mid\nu,\pi_k) }\nonumber \\
        & = \sum_s \abs*{\gamma^{T} p(s_{T}=s \mid \nu,\pi_k) - (1-\gamma)\sum_{t=T}^\infty \gamma^t p(s_t=s\mid\nu,\pi_k) } \nonumber\\
        & \leq  \sum_s\abs*{\gamma^{T} p(s_{T}=s \mid \nu,\pi_k)} +\sum_s\abs*{ (1-\gamma)\sum_{t=T}^\infty \gamma^t p(s_t=s\mid\nu,\pi_k) } \nonumber \\
        & =  \sum_s{\gamma^{T} p(s_{T}=s \mid \nu,\pi_k)} +\sum_s { (1-\gamma)\sum_{t=T}^\infty \gamma^t p(s_t=s\mid\nu,\pi_k) } \nonumber \\
        & = {\gamma^{T} \sum_sp(s_{T}=s \mid \nu,\pi_k)} +\ { (1-\gamma)\sum_{t=T}^\infty \gamma^t \sum_s p(s_t=s\mid\nu,\pi_k) } \nonumber \\
        & \leq \gamma^{T} + (1-\gamma)\sum_{t=T}^\infty \gamma^t  \nonumber \\
        & = 2\gamma^T
    \end{align}
   
where the third transition is due to the triangle inequality, the fourth transition is due to the fact that for any $t$, $\gamma^t p(s_t \mid \nu,\pi_k) \geq 0$ and the sixth transition is by the fact that $\sum_s p(s_t=s|\nu,\pi_k)\leq 1$ for any $t$ as a probability distribution.
    
Thus,
\begin{align*}
    & \E_{s\sim d_{\nu,\pi_k}^{\text{trunc}}}  \frac{d_{\mu,\pi^*}(s)}{d_{\nu,\pi_k}(s)} \inner*{X_k^{\text{trunc}}(s,\cdot),\pi(\cdot\mid s)-\pi'(\cdot\mid s)}  - \E_{s\sim d_{\nu,\pi_k}} \frac{d_{\mu,\pi^*}(s)}{d_{\nu,\pi_k}(s)} \inner*{X_k^{\text{trunc}}(s,\cdot),\pi(\cdot\mid s)-\pi'(\cdot\mid s)} \\
    & = \sum_s \br*{d_{\nu,\pi_k}^{\text{trunc}}(s) - d_{\nu,\pi_k}(s)} \frac{d_{\mu,\pi^*}(s)}{d_{\nu,\pi_k}(s)} \inner*{X_k^{\text{trunc}}(s,\cdot),\pi(\cdot\mid s)-\pi'(\cdot\mid s)} \\
    & \leq \sum_s \abs*{d_{\nu,\pi_k}^{\text{trunc}}(s) - d_{\nu,\pi_k}(s)}\abs*{ \frac{d_{\mu,\pi^*}(s)}{d_{\nu,\pi_k}(s)} \inner*{X_k^{\text{trunc}}(s,\cdot),\pi(\cdot\mid s)-\pi'(\cdot\mid s)}} \\
    & \leq \max_s \abs*{\frac{d_{\mu,\pi^*}(s)}{d_{\nu,\pi_k}(s)} \inner*{X_k^{\text{trunc}}(s,\cdot),\pi(\cdot\mid s)-\pi'(\cdot\mid s)}} \sum_s \abs*{d_{\nu,\pi_k}^{\text{trunc}}(s) - d_{\nu,\pi_k}(s)} \\
    & \leq 2 \gamma^T \norm{\frac{d_{\mu,\pi^*}}{d_{\nu,\pi_k}}}_\infty\max_s \abs*{\inner*{X_k^{\text{trunc}}(s,\cdot),\pi(\cdot\mid s)-\pi'(\cdot\mid s)}} \\
    & \leq  \norm{\frac{d_{\mu,\pi^*}}{d_{\nu,\pi_k}}}_\infty t_k  \HoeffdingBound(k,\lambda)2\gamma^T,
\end{align*}
where the fourth transition is by plugging in \eqref{eq: truncation of stationary distribution} and the last transition is by repeating similar analysis to \eqref{eq: bounded RV hat X}.

Now, by simple arithmetic, for any $\epsilon>0$, if the trajectory length $T>\frac{1}{1-\gamma}\log \frac{\epsilon}{16\HoeffdingBound(k,\lambda)}$, we get that the first bias term is bounded,
\begin{align}\label{eq: error due to a finite horizon of chosen state}
     &\E_{s\sim d_{\nu,\pi_k}^{\text{trunc}}}  \frac{d_{\mu,\pi^*}(s)}{d_{\nu,\pi_k}(s)} \inner*{X_k^{\text{trunc}}(s,\cdot),\pi(\cdot\mid s)-\pi'(\cdot\mid s)}  - \E_{s\sim d_{\nu,\pi_k}} \frac{d_{\mu,\pi^*}(s)}{d_{\nu,\pi_k}(s)} \inner*{X_k^{\text{trunc}}(s,\cdot),\pi(\cdot\mid s)-\pi'(\cdot\mid s)} \nonumber \\
     &\leq \norm{\frac{d_{\mu,\pi^*}}{d_{\nu,\pi_k}}}_\infty t_k \frac{\epsilon}{8}
\end{align}

Next, we bound the second error term.

First, observe that for any $s,a$,
\begin{align}\label{eq: difference between biased and unbiased q}
    &\abs*{\E \hat q_{\lambda,\text{trunc}}^{\pi_k}(s,a,m) - \qHpik(s,a)} \nonumber\\
    &= \abs*{ \E \brs*{\sum_{t=0}^{T-1} \gamma^t \br*{c_t(s_t,a_t) + \lambda \omegaS{s_t}{\pi_k}} \mid s_0=s,a_0=a} -\E \brs*{\sum_{t=0}^\infty \gamma^t \br*{c_t(s_t,a_t) + \lambda \omegaS{s_t}{\pi_k}} \mid s_0=s,a_0=a} }\nonumber\\
    &= \abs*{\E \brs*{\sum_{t=0}^{T-1} \gamma^t \br*{c_t(s_t,a_t) + \lambda \omegaS{s_t}{\pi_k}} - \sum_{t=0}^\infty \gamma^t \br*{c_t(s_t,a_t) + \lambda \omegaS{s_t}{\pi_k}} \mid s_0=s,a_0=a} } \nonumber\\
    & = \abs*{\E \brs*{\sum_{t={T}}^\infty \gamma^t \br*{c_t(s_t,a_t) + \lambda \omegaS{s_t}{\pi_k}} \mid s_0=s,a_0=a} }\nonumber\\
    & \leq \gamma^{T}\frac{\CmaxLambda}{1-\gamma}
\end{align}

Now,

\begin{align*}
    &\E_{s\sim d_{\nu,\pi_k}}\frac{d_{\mu,\pi^*}(s)}{d_{\nu,\pi_k}(s)} \inner*{X_k^{\text{trunc}}(s,\cdot),\pi(\cdot\mid s)-\pi'(\cdot\mid s)}  - \E_{s\sim d_{\nu,\pi_k}}\frac{d_{\mu,\pi^*}(s)}{d_{\nu,\pi_k}(s)} \inner*{X_k(s,\cdot),\pi(\cdot\mid s)-\pi'(\cdot\mid s)} \\
    &= \E_{s\sim d_{\nu,\pi_k}}\frac{d_{\mu,\pi^*}(s)}{d_{\nu,\pi_k}(s)} \inner*{X_k^{\text{trunc}}(s,\cdot)-X_k(s,\cdot),\pi(\cdot\mid s)-\pi'(\cdot\mid s)} \\
    &= \sum_s d_{\nu,\pi_k}(s)\frac{d_{\mu,\pi^*}(s)}{d_{\nu,\pi_k}(s)} \inner*{X_k^{\text{trunc}}(s,\cdot)-X_k(s,\cdot),\pi(\cdot\mid s)-\pi'(\cdot\mid s)} \\
    & \leq \max_s \frac{d_{\mu,\pi^*}(s)}{d_{\nu,\pi_k}(s)} \inner*{X_k^{\text{trunc}}(s,\cdot)-X_k(s,\cdot),\pi(\cdot\mid s)-\pi'(\cdot\mid s)} \\
    & \leq t_k \norm{\frac{d_{\mu,\pi^*}}{d_{\nu,\pi_k}}}_\infty\max_s \inner*{X_k^{\text{trunc}}(s,\cdot)-X_k(s,\cdot),\pi(\cdot\mid s)-\pi'(\cdot\mid s)} \\
    & = t_k \norm{\frac{d_{\mu,\pi^*}}{d_{\nu,\pi_k}}}_\infty\max_s\inner*{\E \hat q_{\lambda,\text{trunc}}^{\pi_k}(s,\cdot,m)-\qHpik(s,\cdot),\pi(\cdot\mid s)-\pi'(\cdot\mid s)} \\
    & \leq t_k \norm{\frac{d_{\mu,\pi^*}}{d_{\nu,\pi_k}}}_\infty\max_s \norm{\E \hat q_{\lambda,\text{trunc}}^{\pi_k}(s,\cdot,m)-\qHpik(s,\cdot)}_\infty \norm{\pi(\cdot\mid s)-\pi'(\cdot\mid s)}_1 \\
    &\leq 2 \norm{\frac{d_{\mu,\pi^*}}{d_{\nu,\pi_k}}}_\infty t_k \frac{\CmaxLambda}{1-\gamma} \gamma^T,
\end{align*}
where the first transition is due to the linearity of expectation, the third transition is by the fact the summation of $d_{\nu,\pi_k}$ is convex, the fourth transition is by the fact $\frac{d_{\mu,\pi^*}(s)}{d_{\nu,\pi_k}(s)}$ is non-negative for any $s$ and by maximizing each term separately, the fifth transition is by using the definitions of $X_k$ and $X_k^{\text{trunc}}$, the sixth is using H{\"o}lder's inequality and the last transition is due to \eqref{eq: difference between biased and unbiased q}.

Now, using the same $T$, by the fact $\HoeffdingBound(k,\lambda) > \frac{2\CmaxLambda}{1-\gamma}$, we have that

\begin{align}\label{eq: error due to a finite horizon of q estimation}
    &\E_{s\sim d_{\nu,\pi_k}}\frac{d_{\mu,\pi^*}(s)}{d_{\nu,\pi_k}(s)} \inner*{X_k^{\text{trunc}}(s,\cdot),\pi(\cdot\mid s)-\pi'(\cdot\mid s)}  - \E_{s\sim d_{\nu,\pi_k}}\frac{d_{\mu,\pi^*}(s)}{d_{\nu,\pi_k}(s)} \inner*{X_k^{\text{trunc}}(s,\cdot),\pi(\cdot\mid s)-\pi'(\cdot\mid s)} \nonumber\\
    &\leq  \norm{\frac{d_{\mu,\pi^*}}{d_{\nu,\pi_k}}}_\infty t_k \frac{\epsilon}{8}.
\end{align}

Finally, combining \eqref{eq: error due to a finite horizon of chosen state} and \eqref{eq: error due to a finite horizon of q estimation} concludes the results.
\end{proof}

\end{lemma}

In the next lemma we combine the results of Lemmas~\ref{lemma: one step approximation error bound unbiased} and \ref{lemma: truncation error bound} to bound the overall approximation error due to both sampling and truncation.

\begin{lemma}[Approximation error bound using truncated biased sampling]\label{lemma: one step approximation error bound}
For any $\epsilon,\tilde \delta>0$, if the number of trajectories in the $k$-th iteration is
$$ M_k \geq \frac{8\HoeffdingBound(k,\lambda)^2}{\epsilon^2} \br*{ S \log {2A} + \log{1 / \tilde \delta}},$$
and the number of samples in the truncated sampling process is of length 

$$T_k\geq\frac{1}{1-\gamma}\log \frac{\epsilon}{8\HoeffdingBound(k,\lambda)},$$
then with probability of $1- \tilde \delta$,
\begin{align*}
    d_{\mu,\pi^*}\epsilon_k \leq t_k \norm{ \frac{d_{\mu,\pi^*}}{d_{\nu,\pi_k}} }_\infty \frac{\epsilon}{2},
\end{align*}
and the overall number of interaction with the MDP is in the $k$-th iteration is
$$ O\br*{\frac{\HoeffdingBound(k,\lambda)^2 \br*{S\log A + \log1/\tilde \delta}}{(1-\gamma)\epsilon^2}},$$
where $\HoeffdingBound(k,\lambda)=\frac{4  A\CmaxLambda}{1-\gamma}$ and $\HoeffdingBound(k,\lambda)=  \frac{2 A\CmaxLambda}{1-\gamma} \br*{\frac{1}{1-\lambda t_k}+1+\lambda \log k}$ in the~euclidean and non-euclidean settings respectively.
\end{lemma}

\begin{proof}

Repeating the same steps of Lemma~\ref{lemma: one step approximation error bound unbiased}, we re-derive
equation \eqref{eq: bounding the approximation error 2},
\begin{align*}
    d_{\mu,\pi^*}\epsilon_k \leq & \max_{\pi'} \Big\{ \frac{1}{M_k}\sum_{m=1}^{M_k}\sum_s \mathbbm{1}\brc*{s=s_m} \frac{d_{\mu,\pi^*}(s)}{ d_{\nu,\pi_k}(s)}  \inner*{ \hat{X}_k(s,\cdot,m), \pi^*(\cdot \mid s_m) - \pi'(\cdot \mid s_m) } \nonumber\\
    &- \sum_s d_{\mu,\pi^*}(s) \inner*{ X_k(s,\cdot),\pi^*(\cdot \mid s)- \pi'(\cdot \mid s) } \Big\}.
\end{align*}

Now, we move on to deal with a truncated trajectory: In Appendix~\ref{supp: approximate update rule TRPO} we defined a nearly unbiased estimation process for $\qHpik$, i.e., $\frac{1}{M_k}\sum_{m=1}^{M_k}\sum_s \mathbbm{1}\brc*{s=s_m} \frac{d_{\mu,\pi^*}(s)}{ d_{\nu,\pi_k}(s)}  \inner*{ \hat{X}_k(s,\cdot,m), \pi^*(\cdot \mid s_m) - \pi'(\cdot \mid s_m) } $ is no longer an unbiased estimator as in Lemma~\ref{lemma: one step approximation error bound unbiased}. In what follows we divide the error to two sources of error, one due to the finite sampling error (finite number of trajectories) and the other due to the bias admitted by the truncation.

For any $\pi'$, denote the following variables,
\begin{align*}
    & \hat Y_m(\pi') \triangleq \frac{d_{\mu,\pi^*}(s_m)}{ d_{\nu,\pi_k}(s_m)}\inner{\hat X_k(s_m,\cdot,m), \pi^*(\cdot \mid s_m) - \pi'(\cdot \mid s_m)} \\
    & Y(\pi') \triangleq  \sum_s d_{\mu,\pi^*}(s) \inner*{ X_k(s,\cdot), \pi^*(\cdot \mid s) - \pi'(\cdot \mid s) }.
\end{align*}
By plugging this new notation in \eqref{eq: bounding the approximation error 2},  we can write,
\begin{align}\label{eq: estimation error with bias}
    d_{\mu,\pi^*}\epsilon_k &\leq \max_{\pi'}  \frac{1}{M} \sum_{m=1}^M \hat Y_m(\pi') - Y(\pi')  \nonumber\\
    &= \max_{\pi'}  \frac{1}{M} \sum_{m=1}^M \hat Y_m(\pi') - \E \hat Y_m(\pi') + \E \hat Y_m(\pi') - Y(\pi')  \nonumber \\
    &\leq \underbrace{\max_{\pi'} \frac{1}{M} \sum_{m=1}^M \hat Y_m(\pi') - \E \hat Y_m(\pi')}_{(1)} + \underbrace{\max_{\pi'}  \vphantom{\frac{1}{M}\sum_{m=1}^M} \E \hat Y_m(\pi') - Y(\pi') }_{(2)}, 
\end{align}
where the first inequality is by plugging in the definition of $Y(\pi'),\hat Y_M(\pi')$ in \eqref{eq: bounding the approximation error 2} and the last transition is by maximizing each of the terms in the sum independently. Note that (1) describes the error due to the finite sampling and (2) describes the error due to the truncation of the trajectories. Importantly, notice that in the case where we do not truncate the trajectory, the second term (2) equals zero by \eqref{eq: expectation of sample}. We will now use Lemma~\ref{lemma: one step approximation error bound unbiased} and Lemma~\ref{lemma: truncation error bound} to bound (1) and (2) respectively:

% As we wish to obtain a uniform bound on $\pi'$, we can use the common approach of bounding both terms uniformly on all $\pi'\in \ssimplex$ using the union bound. Note that the above optimization problem is a linear programming optimization problem in $\pi'$, where $\pi'\in \ssimplex$. It is a well known fact that for linear programming, there is an extreme point which is the optimal solution of the problem \cite{bertsimas1997introduction}[Theorem 2.7]. The set of extreme points of $\ssimplex$ is the set of all deterministic policies denoted by $\Pi^{\text{det}}$. Therefore, in order to bound the maximum in \eqref{eq: maximum bound over policies}, it suffices to uniformly bound all policies $\pi' \in \Pi^{\text{det}}$.

First, look at the \textbf{first term} (1). By definition it an unbiased estimation process. Furthermore, by equation \eqref{eq: bounded RV hat X}, $\hat Y_m(\pi')$ is bounded for all $s_m$ and $\pi'$ by
\begin{align*}
\hat Y_m(\pi') \leq  t_k \norm{\frac{d_{\mu,\pi^*}}{ d_{\nu,\pi_k}}}_\infty \HoeffdingBound(k,\lambda),
\end{align*}

Thus by applying Lemma~\ref{lemma: one step approximation error bound unbiased} we get that in order to guarantee that 
\begin{align}\label{eq: Hoeffding term bound}
\max_{\pi'} \frac{1}{M} \sum_{m=1}^M \br*{ \hat Y_m(\pi') - \E \hat Y_m(\pi')}\leq t_k \norm{\frac{d_{\mu,\pi^*}}{d_{\nu,\pi_k}}}_\infty \frac{\epsilon}{4},
\end{align}
we need the number of trajectories $M_k$ to be at least
\begin{align*}
    & M_k \geq \frac{8\HoeffdingBound(k,\lambda)^2}{\epsilon^2} \br*{ S \log {2A} + \log{1 / \tilde \delta}}.
\end{align*}

Next, we bound the \textbf{second term} (2).
By Lemma~\ref{lemma: truncation error bound}, using a trajectory of maximal length
$\frac{1}{1-\gamma} \log \frac{\epsilon}{8\HoeffdingBound (k,\lambda)}$,
the errors due to the truncated estimation process are bounded as follows,
\begin{align}\label{eq: bias term bound}
\max_{\pi'}  \vphantom{\frac{1}{M}\sum_{m=1}^M} \E \hat Y_m(\pi') - Y(\pi') \leq t_k \norm{\frac{d_{\mu,\pi^*}}{d_{\nu,\pi_k}}}_\infty \frac{\epsilon}{4}
\end{align}

Bounding the two terms by \eqref{eq: Hoeffding term bound} and \eqref{eq: bias term bound}, and plugging them back in \eqref{eq: estimation error with bias}, we get that using $M_k$ trajectories, where each trajectory is of length $O(\frac{1}{1-\gamma} \log \epsilon)$, we have that w.p. $1-\tilde \delta$
$$ d_{\mu,\pi^*} \epsilon_k \leq t_k \norm{\frac{d_{\mu,\pi^*}}{d_{\nu,\pi_k}}}_\infty \frac{\epsilon}{4} + t_k \norm{\frac{d_{\mu,\pi^*}}{d_{\nu,\pi_k}}}_\infty \frac{\epsilon}{4} = t_k \norm{\frac{d_{\mu,\pi^*}}{d_{\nu,\pi_k}}}_\infty \frac{\epsilon}{2}, $$
which concludes the result.

\end{proof}

So far, we proved the number of samples needed for a bounded error with high probability in the $k$-th iteration of Sample-Based TRPO. The following Lemma gives a bound for the accumulative error of Sample-Based TRPO after $k$ iterations.

\begin{lemma}[Cumulative approximation error]\label{lemma: cumulative approximation error}
For any $\epsilon,\delta>0$, if the number of trajectories in the $k$-th iteration is
$$ M_k \geq \frac{8\HoeffdingBound(N,\lambda)^2}{\epsilon^2} \br*{ S \log {2A} + \log{2(k+1)^2 / \delta}},$$
and the number of samples in the truncated sampling process is of length 
$$T\geq\frac{1}{1-\gamma}\log \frac{\epsilon}{8\HoeffdingBound(k,\lambda)},$$
then, with probability greater than $1-\delta$, uniformly on all $k\in \mathbb{N}$,
\begin{align*}
    \sum_{k=0}^N d_{\mu,\pi^*}\epsilon_k \leq \frac{\epsilon/2}{1-\gamma} \norm{ \frac{d_{\mu,\pi^*}}{\nu}}_\infty \sum_{k=0}^N t_k,
\end{align*}
where $\HoeffdingBound(k,\lambda)=\frac{4  A\CmaxLambda}{1-\gamma}$ and $\HoeffdingBound(k,\lambda)=  \frac{4 A \CmaxLambda}{1-\gamma} \br*{1+\mathbbm{1}\{ \lambda \neq 0 \} \log k }$ in the euclidean and non-euclidean settings respectively. 
\begin{proof}

Using Lemma~\ref{lemma: one step approximation error bound} with $\tilde \delta = \frac{6}{\pi^2}\frac{\delta}{(k+1)^2}$ and the union bound over all $k\in \mathbb{N}$, we get that w.p. bigger than $$\sum_{k=0}^{\infty} \frac{6}{\pi^2} \frac{\delta}{(k+1)^2} = \frac{6}{\pi^2} \delta \sum_{k=0}^{\infty} \frac{1}{(k+1)^2}= \delta ,$$
for any $k$, the following inequality holds
$$ d_{\mu,\pi^*}\epsilon_k \leq t_k \norm{ \frac{d_{\mu,\pi^*}}{d_{\nu,\pi_k}}}_\infty \frac{\epsilon}{2}. $$
where we used the solution to Basel's problem (the sum of reciprocals of the squares of the natural numbers) for calculating $\sum_{k=0}^\infty \frac{1}{(k+1)^2}$.

Thus, by summing the inequalities for $k=0,1,...,N$, we obtain
\begin{align*}
    \sum_{k=0}^N d_{\mu,\pi^*}\epsilon_k &\leq \frac{\epsilon}{2} \sum_{k=0}^N t_k\norm{ \frac{d_{\mu,\pi^*}}{d_{\nu,\pi_k}}}_\infty.
\end{align*}

Now, Using the fact that $\norm{ \frac{d_{\mu,\pi^*}}{d_{\nu,\pi_k}}}_\infty \leq \frac{1}{1-\gamma}\norm{ \frac{d_{\mu,\pi^*}}{\nu}}_\infty$, we have that w.p. of at least $\delta$,
\begin{align*}
    \sum_{k=0}^N d_{\mu,\pi^*}\epsilon_k \leq \frac{\epsilon/2}{1-\gamma} \norm{ \frac{d_{\mu,\pi^*}}{\nu}}_\infty\sum_{k=0}^N t_k.
\end{align*}
Lastly, by bounding $\pi^2/6\leq 2$ we conclude the proof.

\end{proof}
\end{lemma}

We are ready to prove the convergence rates for the unregularized and regularized algorithms, similarly to the proofs of Exact TRPO (see Appendix~\ref{supp: exact TRPO unregularized proof}).

\subsection{Proof of Theorem \ref{theorem: Sample-Based TRPO All}}\label{supp: proof of Sample-Based TRPO}

For the sake of completeness and readability, we restate here Theorem~\ref{theorem: Sample-Based TRPO All}, this time including all logarithmic factors, but excluding higher orders in $\lambda$ (All constants are in the proof):

\begin{theorem*}[Convergence Rate: Sample-Based TRPO]
Let $\{\pik\}_{k\geq 0}$ be the sequence generated by Sample-Based TRPO, using $M_k \geq \frac{\HoeffdingBound(N,\lambda)^2}{2\epsilon^2} \br*{ S \log {2A} + \log{\pi^2(k+1)^2 / 6\delta}}$ trajectories in each iteration, and $\{\mu v_{\textbf{best}}^k\}_{k\geq0}$ be the sequence of best achieved values, $\mu v_{\textbf{best}}^N\triangleq \arg\min_{k=0,...,N} \mu v^{\pik}_\lambda-\mu v_\lambda^*$. Then, with probability greater than $1-\delta$ for every $\epsilon>0$ the following holds for all $N\geq1$.
\begin{enumerate}
    \item \emph{(Unregularized)} Let $\lambda=0$, $t_k=\frac{(1-\gamma)^2}{\euclid \Cmax  \sqrt{k+1}}$ then
    \begin{align*}
       &\mu v_{\textbf{best}}^N - \mu \vstar  \\
       &\leq O\br*{\frac{\euclid \Cmax (\bregmanMax + \log N)}{(1-\gamma)\sqrt{N}}+ \frac{C^{\pi^*}\epsilon}{(1-\gamma)^2}}
    \end{align*}
    \item \emph{(Regularized)} Let $\lambda>0$, $t_k =  \frac{1}{\lambda (k+2)}$ then

\begin{align*}
  \mu  v_{\textbf{best}}^N - \mu  \vHstar   \leq  O \br*{ \frac{\euclid ^2 \approxAgrowth \CmaxLambda^2 \log N}{\lambda(1
    -\gamma)^3 N} + \frac{C^{\pi^*}\epsilon}{(1-\gamma)^2}}.
\end{align*}
\end{enumerate}
Where $\euclid =\sqrt{A}, \approxAgrowth=1, \bregmanMax=1, \HoeffdingBound(k,\lambda)=\frac{4  A\CmaxLambda}{1-\gamma}$ for the euclidean case, and $\euclid =1 , \approxAgrowth=A^2, \bregmanMax=\log A, \HoeffdingBound(k,\lambda)=  \frac{4 A \CmaxLambda}{1-\gamma} \br*{1+\mathbbm{1}\{ \lambda \neq 0 \} \log k }$ for the non-euclidean case.

\end{theorem*}

The proof of this theorem follows the almost identical steps as the proof of Theorem~\ref{theorem: Exact TRPO} in Appendix~\ref{supp: exact TRPO unregularized proof}, but two differences: The first, is the fact we also have the additional approximation error term $d_{\mu,\pi^*}\epsilon_k$. The second, is that for the sample-based case, as we don't have improvement guarantees such as in Lemma~\ref{lemma: policy improvement exact}, we prove convergence for best policy in hindsight, which have the value $\mu v_{\textbf{best}}^N\triangleq \arg\min_{k=0,...,N} \mu v^{\pik}_\lambda-\mu v_\lambda^*$.

\subsubsection{The Unregularized Case}

\begin{proof}
Applying Lemma~\ref{lemma: mu approximate fundamental inequality} and $\lambda=0$ (the unregularized case),
\begin{align*}
    &t_k(1-\gamma) (\mu\vpik-\mu \vstar) \\
&\leq    d_{\mu,\pi^*}\br*{ \bregman{\pi^*}{\pi_k} - \bregman{ \pi^* }{\pi_{k+1}}} + \frac{t_k^2 \gradBound^2}{2}+d_{\mu,\pi^*}\epsilon_k.
\end{align*}

Summing the above inequality over $k=0,1,...,N$, gives
\begin{align*}
    &\sum\limits_{k=0}^{N}  t_k(1-\gamma)(\mu \vpik- \mu\vstar) \\
    &\leq   d_{\mu,\pi^*}\bregman{\pistar}{\pi_0} - d_{\mu,\pi^*}\bregman{\pistar}{\pi_{N+1}} +  \sum\limits_{k=0}^{N}\frac{t_k^2 \gradBound^2}{2} + \sum_{k=0}^{N} d_{\mu,\pi^*}\epsilon_k\\
    &\leq  d_{\mu,\pi^*}\bregman{\pistar}{\pi_0} +  \sum\limits_{k=0}^{N}\frac{t_k^2 \gradBound^2}{2} + \sum_{k=0}^{N} d_{\mu,\pi^*}\epsilon_k\\
    &\leq  \omegaBound +  \sum\limits_{k=0}^{N}\frac{t_k^2 \gradBound^2}{2} + \sum_{k=0}^{N} d_{\mu,\pi^*}\epsilon_k.
\end{align*}

where in the second relation we used    $\bregman{\pistar}{\pi_{N+1}}\geq 0$ and thus $d_{\mu,\pi^*}\bregman{\pistar}{\pi_{N+1}}\geq 0$, and in the third relation Lemma \ref{lemma: bounds on Dw}.

Using the definition of $v_{\text{best}}^N$, we have that
$$\mu(v_{\text{best}}^N-\vstar) \sum\limits_{k=0}^{N}  t_k \leq \sum\limits_{k=0}^{N}  t_k(\mu \vpik-\mu \vstar),$$

and by some algebraic manipulations, we get
\begin{align*}
    \mu v_{\text{best}}^N-\mu \vstar \leq \frac{1}{1-\gamma}
   \frac{\omegaBound +  \sum\limits_{k=0}^{N}\frac{t_k^2 \gradBound^2}{2}+\sum\limits_{k=0}^{N} d_{\mu,\pi^*}\epsilon_k}{\sum\limits_{k=0}^{N} {t_k}}\\
   = \frac{1}{1-\gamma}
   \frac{\omegaBound +  \frac{ \gradBound^2}{2}\sum\limits_{k=0}^{N}t_k^2}{\sum\limits_{k=0}^{N} {t_k}} +\frac{1}{1-\gamma}\frac{\sum\limits_{k=0}^{N} d_{\mu,\pi^*}\epsilon_k}{{\sum\limits_{k=0}^{N} {t_k}}},
\end{align*}
Plugging in the stepsizes $t_k = \frac{1}{\gradBound\sqrt{k}}$, we get, 
\begin{align*}
      \mu v_{\textbf{best}}^N-\mu \vstar &\leq \frac{\gradBound}{1-\gamma}
   \frac{2\omegaBound + \sum\limits_{k=0}^{N}\frac{1}{k}}{2\sum\limits_{k=0}^{N} \frac{1}{\sqrt{k}}} + \frac{1}{1-\gamma}\frac{\sum\limits_{k=0}^{N} d_{\mu,\pi^*}\epsilon_k}{{\sum\limits_{k=0}^{N} {t_k}}}.
\end{align*}

Bounding the sums using (\citealp{beck2017first}, Lemma 8.27(a)) yields,

\begin{align*}
      \mu v_{\textbf{best}}^N-\mu \vstar &\leq O\br*{\frac{\gradBound}{1-\gamma}
   \frac{\omegaBound + \log{N}}{\sum\limits_{k=0}^{N} \frac{1}{\sqrt{k}}} + \frac{1}{\sum_{k=0}^N t_k}\frac{1}{1-\gamma}\sum_{k=0}^N d_{\mu,\pi^*}\epsilon_k}.
\end{align*}

Plugging in Lemma~\ref{lemma: cumulative approximation error}, we get that for any $\br*{\epsilon,\delta}$, if the number of trajectories in the $k$-th iteration is
$$ M_k \geq \frac{\HoeffdingBound(N,\lambda)^2}{2\epsilon^2} \br*{ S \log {2A} + \log{\pi^2(k+1)^2 / 6\delta}},$$
then, with probability greater than $1-\delta$,
\begin{align*}
      \mu v_{\textbf{best}}^N-\mu \vstar &\leq O\br*{\frac{\gradBound}{1-\gamma}
   \frac{\omegaBound + \log{N}}{\sum\limits_{k=0}^{N} \frac{1}{\sqrt{k}}} + \frac{1}{\sum_{k=0}^N t_k}\frac{\epsilon}{(1-\gamma)^2} \norm{ \frac{d_{\mu,\pi^*}}{\nu}}_\infty \sum_{k=0}^N t_k},
\end{align*}
where $\HoeffdingBound(k,\lambda)=\frac{4  A\CmaxLambda}{1-\gamma}$ and $\HoeffdingBound(k,\lambda)=  \frac{4 A \CmaxLambda}{1-\gamma} \br*{1+\mathbbm{1}\{ \lambda \neq 0 \} \log k }$ in the euclidean and non-euclidean settings respectively.

By rearranging, we get,
\begin{align*}
      \mu v_{\textbf{best}}^N-\mu \vstar &\leq O\br*{\frac{\gradBound}{1-\gamma}
   \frac{\omegaBound + \log{N}}{\sum\limits_{k=0}^{N} \frac{1}{\sqrt{k}}} + \frac{\epsilon}{(1-\gamma)^2} \norm{ \frac{d_{\mu,\pi^*}}{\nu}}_\infty }.
\end{align*}

Thus, for the euclidean case,
\begin{align*}
     \mu v_{\textbf{best}}^N-\mu \vstar \leq O\br*{\frac{\Cmax \sqrt{A}\log N}{(1-\gamma)^2\sqrt{N}} +\frac{1}{(1-\gamma)^2}\norm{\frac{d_{\mu,\pi^*}}{\nu}}\epsilon},
\end{align*}
and for the non-euclidean case,
\begin{align*}
    \mu v_{\textbf{best}}^N-\mu \vstar \leq O\br*{\frac{\Cmax(\log A +  \log N)}{(1-\gamma)^2\sqrt{N}} +\frac{1}{(1-\gamma)^2}\norm{\frac{d_{\mu,\pi^*}}{\nu}}\epsilon}.
\end{align*}

\end{proof}

\subsubsection{The Regularized Case}
\begin{proof}

Applying Lemma~\ref{lemma: mu approximate fundamental inequality} and setting $t_k= \frac{1}{\lambda(k+2)}$, we get,

\begin{align*}
&\frac{1-\gamma}{\lambda (k+2)} \br*{\mu v^{\pi_k}_\lambda- \mu v^{\pi^*}_\lambda } \\
&\leq   d_{\mu,\pi^*}\br*{  (1- \frac{1}{(k+2)})\bregman{\pi^*}{\pi_k} - \bregman{ \pi^* }{\pi_{k+1}}} +\frac{\gradBound^2(k;\lambda)}{2\lambda^2(k+2)^2} + d_{\mu,\pi^*}\epsilon_k\\
&\leq   d_{\mu,\pi^*}\br*{  \frac{k+1}{k+2}\bregman{\pi^*}{\pi_k} - \bregman{ \pi^* }{\pi_{k+1}}} +\frac{\gradBound^2(N;\lambda)}{2\lambda^2(k+2)^2} + d_{\mu,\pi^*}\epsilon_k,
\end{align*}

where in the second relation we used that fact  $\gradBound(k;\lambda)$ is a non-decreasing function of $k$ for both the euclidean and non-euclidean cases. 

Next, multiplying both sides by $\lambda (k+2)$, summing both sides from $k=0$ to $N$ and using the linearity of expectation, we get,

\begin{align*}
    \sum_{k=0}^{N} (1-\gamma)\br*{
    \mu \vHpik-\mu \vHstar}&\leq  d_{\mu,\pi^*}\br*{  \bregman{\pi^*}{\pi_0} - (N+2)\bregman{ \pi^* }{\pi_{N+1}}} +\sum_{k=0}^{N}\frac{\gradBound^2(N;\lambda)}{2\lambda (k+2)} + \sum\limits_{k=0}^{N} \lambda (k+2) d_{\mu,\pi^*}\epsilon_k\\
    &\leq  d_{\mu,\pi^*}  \bregman{\pi^*}{\pi_0} +\sum_{k=0}^{N}\frac{\gradBound^2(N;\lambda)}{2\lambda (k+2)} + \sum\limits_{k=0}^{N} \lambda (k+2) d_{\mu,\pi^*}\epsilon_k\\
    &\leq  \omegaBound +\sum_{k=0}^{N}\frac{\gradBound^2(N;\lambda)}{2\lambda (k+2)} + \sum\limits_{k=0}^{N} \lambda (k+2)d_{\mu,\pi^*}\epsilon_k \\ 
    & = \omegaBound +\sum_{k=0}^{N}\frac{\gradBound^2(N;\lambda)}{2\lambda (k+2)} + \sum\limits_{k=0}^{N} \frac{1}{t_k} d_{\mu,\pi^*}\epsilon_k,
\end{align*}
where the second relation holds by the positivity of the Bregman distance, the third relation by Lemma \ref{lemma: bounds on Dw} for uniformly initialized $\pi_0$, and the last relation by plugging back $t_k=\frac{1}{\lambda(k+2)}$ in the last term..

Bounding $\sum_{k=0}^N\frac{1}{k+2}\leq O(\log N)$, we get
\begin{align*}
    \sum_{k=0}^{N} 
    \mu\vHpik-\mu\vHstar  \leq O\br*{\frac{\omegaBound}{(1-\gamma)} +\frac{\gradBound^2(N;\lambda) \log N} {\lambda (1-\gamma)} + \frac{1}{1-\gamma}\sum\limits_{k=0}^{N} \frac{1}{t_k} d_{\mu,\pi^*}\epsilon_k}.
\end{align*}

By the definition of $v^N_{\text{best}}$, which gives $ (N+1) \br*{\mu v^N_{\text{best}}-\mu \vstar} \leq \sum\limits_{k=0}^{N}  \mu \vpik-\mu \vstar$,
and some algebraic manipulations, we obtain

\begin{align*}
    \mu v_{\textbf{best}}^N-\mu\vHstar \leq O\br*{\frac{\omegaBound}{(1-\gamma)N} +\frac{\gradBound^2(N;\lambda) \log N} {\lambda (1-\gamma)N} + \frac{1}{1-\gamma}\frac{1}{N}\sum_{k=0}^N \frac{1}{t_k}d_{\mu,\pi^*}\epsilon_k }.
\end{align*}

Plugging in Lemma~\ref{lemma: one step approximation error bound}, we get that for any $\br*{\epsilon,\delta}$, if the number of trajectories in the $k$-th iteration is
$$ M_k \geq \frac{\HoeffdingBound(k,\lambda)^2}{2\epsilon^2} \br*{ S \log {2A} + \log{\pi^2(k+1)^2 / 6\delta}},$$
then with probability of at least $1-\delta$, 
\begin{align*}
    \mu v_{\textbf{best}}^N-\mu\vHstar \leq O\br*{\frac{\omegaBound}{(1-\gamma)N} +\frac{\gradBound^2(N;\lambda) \log N} {\lambda (1-\gamma)N} + \frac{\epsilon}{(1-\gamma)^2} \norm{\frac{d_{\mu,\pi^*}}{\nu}}_\infty}.
\end{align*}
where $\HoeffdingBound(k,\lambda)=\frac{4  A\CmaxLambda}{1-\gamma}$ and $\HoeffdingBound(k,\lambda)=  \frac{4 A \CmaxLambda}{1-\gamma} \br*{1+\mathbbm{1}\{ \lambda \neq 0 \} \log k }$ in the~euclidean and non-euclidean settings respectively.

By Plugging the bounds $\omegaBound, \gradBound$ and $\CmaxLambda$, we get in the euclidean case,
\begin{align*}
    \mu v_{\textbf{best}}^N-\mu\vHstar \leq O\br*{\frac{\br*{\CmaxSquare+\lambda^2} A \log N} {\lambda (1-\gamma)^3 N}+\frac{1}{(1-\gamma)^2}\norm{\frac{d_{\mu,\pi^*}}{\nu}}_\infty\epsilon},
\end{align*}
and in the non-euclidean case,
\begin{align*}
    \mu v_{\textbf{best}}^N-\mu \vHstar \leq O\br*{\frac{(\mathrm{C^2_{max}}+\lambda^2 \log^2 A) A^2 \log^3 N} {\lambda (1-\gamma)^3 N}+\frac{1}{(1-\gamma)^2}\norm{\frac{d_{\mu,\pi^*}}{\nu}}_\infty\epsilon},
\end{align*}

\end{proof}

\subsection{Sample Complexity of Sample-Based TRPO}\label{supp: sample complexity of Approximate Mirror Descent}

In this section we calculate the overall \emph{sample complexity} of Sample-Based TRPO, i.e., the number interactions with the MDP the algorithm does in order to reach a close to optimal solution.

By Lemma~\ref{lemma: cumulative approximation error}, in order to have $\frac{1}{(1-\gamma)^2}\norm{\frac{d_{\mu,\pi^*}}{\nu}}_\infty \frac{\epsilon}{2}$ approximation error, we need $M_k\geq O\br*{ \frac{\HoeffdingBound(k,\lambda)^2}{\epsilon^2} \br*{ S \log {2A} + \log{(k+1)^2 / \delta}} }$ trajectories in each iteration, and the number of samples in each truncated trajectory is $T_k \geq O\br*{\frac{1}{1-\gamma}\log \frac{\epsilon}{\HoeffdingBound(k,\lambda)}},$ where $\HoeffdingBound(k,\lambda)=  \frac{4 A \CmaxLambda}{1-\gamma} \br*{1+\mathbbm{1}\{ \lambda \neq 0 \} \log k }$ in the~euclidean and non-euclidean settings respectively.

Therefore, the number of samples in each iteration required to guarantee a $\frac{1}{(1-\gamma)^2}\norm{\frac{d_{\mu,\pi^*}}{\nu}}_\infty \frac{\epsilon}{2}$ error is 
$$ O\br*{ \frac{\HoeffdingBound(k,\lambda)^2\log \frac{\epsilon}{\HoeffdingBound(k,\lambda)}}{(1-\gamma)\epsilon^2} \br*{ S \log {2A} + \log{(k+1)^2 / \delta}} }. $$ 

The overall sample complexity is acquired by multiplying the number of iterations $N$ required to reach an $\frac{\epsilon/2}{(1-\gamma)^2}$  optimization error multiplied with the iteration-wise sample complexity, given above.
Combining the two errors and using the fact that $C^{\pi^*}\geq 1$, we have that the overall error 
$$\frac{1}{(1-\gamma)^2}\br*{1+C^{\pi^*}}\frac{\epsilon}{2} \leq \frac{2}{(1-\gamma)^2}C^{\pi^*}\frac{\epsilon}{2} = \frac{1}{(1-\gamma)^2}C^{\pi^*}\epsilon.$$
In other words, the overall error of the algorithm is bounded by $\frac{1}{(1-\gamma)^2}C^{\pi^*}\epsilon$

Finally, the sample complexity to reach a $\frac{1}{(1-\gamma)^2}C^{\pi^*}\epsilon$ error for the different cases is arranged in the following table (the complete analysis is provided the the next section):
\begin{table}[ht]
\centering
\renewcommand*{\arraystretch}{2}
% \setlength{\extrarowheight}{10pt}
% \begin{tabular}{|m{2.5cm}|m{6cm}|m{6cm}|}
\begin{tabular}{|c|c|c|}
    \hline
    & \textbf{Euclidean} & \textbf{Non-Euclidean (KL)} \\ \hline
     \textbf{Unregularized} &  $ \frac{A^3 C_{\text{max}}^4  }{\br*{1-\gamma}^3\epsilon^4}\br*{\log |\Pi^{\text{det}}| + \log \frac{1}{\delta}}$ & $\frac{A^2 C_{\text{max}}^4  }{\br*{1-\gamma}^3\epsilon^4}\br*{\log |\Pi^{\text{det}}| + \log \frac{1}{\delta}}$ \\ \hline
     \textbf{Regularized} & $\frac{A^3 C_{\text{max},\lambda}^4}{\lambda\br*{1-\gamma}^4\epsilon^3}\br*{\log |\Pi^{\text{det}}| + \log \frac{1}{\delta}}$ & $\frac{A^4 C_{\text{max},\lambda}^4   }{\lambda \br*{1-\gamma}^4\epsilon^3}\br*{\log |\Pi^{\text{det}}| + \log \frac{1}{\delta}}$ \\ \hline
\end{tabular}
\end{table}

The same bound for CPI as given in \cite{kakade2003sample} is 
$$ \frac{A^2 C_{\text{max}}^4}{(1-\gamma)^5\epsilon^4}\br*{\log{|\Pi^{\text{det}}|+\log \frac{1}{\delta}}}, $$ 
where we omitted logarithmic factors in $1-\gamma$ and $\epsilon$. Notice that this bound is similar to the bound of Sample-Based TRPO observed in this paper, as expected.

In order to translate this bound using our notation bound, we used \cite{kakade2003sample}[Theorem 7.3.3] with $H=\frac{1}{1-\gamma}$, which states that in order to guarantee a bounded advantage of for any policy $\pi'$,
$\mathbb{A}_{\pi}(\nu,\pi')\leq (1-\gamma)\epsilon$ we need $O\br*{\frac{\log \epsilon (\log {\Pi}^{\text{det}} + \log \frac{1}{\delta}}{(1-\gamma)^5 \epsilon^4}}$ samples. Then, by \cite{kakade2002approximately}[Corollary 4.5] with $\mathbb{A}_{\pi}(\nu,\pi')\leq (1-\gamma)\epsilon$ we get that $(1-\gamma)(\mu v^\pi - \mu v^*) \leq \frac{\epsilon}{1-\gamma} \norm{\frac{d_{\mu,\pi^*}}{\nu}}_\infty$, or $\mu v^\pi - \mu v^* \leq \frac{\epsilon}{(1-\gamma)^2} \norm{\frac{d_{\mu,\pi^*}}{\nu}}_\infty$. Finally, the $C_{\text{max}}^4$ factor comes from using a non-normalized MDP, where the maximum reward is $C_\text{max}$. We get $C_{\text{max}}^2$ from number of iterations needed for convergence, and the number of samples in each iteration is also proportional to $C_{\text{max}}^2$

\subsubsection{The Unregularized Case}
\paragraph{The euclidean case:}

The error after $N$ iterations is bounded by
\begin{align*}
     \mu v_{\textbf{best}}^N-\mu \vstar \leq O\br*{\frac{\Cmax \sqrt{A}\log N}{(1-\gamma)^2\sqrt{N}} +\frac{1}{(1-\gamma)^2}\norm{\frac{d_{\mu,\pi^*}}{\nu}}\frac{\epsilon}{2}}.
\end{align*}

Thus, in order to reach an error of  $\frac{1}{(1-\gamma)^2} C^{\pi^*} \epsilon$ error, we need 
\begin{align*}
     N \leq O\br*{\frac{C_{\text{max}}^2 A\log \epsilon}{\epsilon^2}}.
\end{align*}

Thus, the sample complexity to reach $\frac{1}{(1-\gamma)^2} C^{\pi^*} \epsilon$ error when logarithmic factors are omitted is
$$ \tilde O\br*{\frac{A^3 C_{\text{max}}^4  }{\br*{1-\gamma}^3\epsilon^4}\br*{\log |{\Pi}^{\text{det}}| + \log \frac{1}{\delta}}}$$

\paragraph{The non-euclidean case:}

The error after $N$ iterations is bounded by
\begin{align*}
    \mu v_{\textbf{best}}^N-\mu \vstar \leq O\br*{\frac{\Cmax(\log A +  \log N)}{(1-\gamma)^2\sqrt{N}} +\frac{1}{(1-\gamma)^2}\norm{\frac{d_{\mu,\pi^*}}{\nu}}\frac{\epsilon}{2}}.
\end{align*}

Thus, in order to reach an error of  $\frac{1}{(1-\gamma)^2}C^{\pi^*}\epsilon$ error, we need 
\begin{align*}
     N \leq O\br*{\frac{C_{\text{max}}^2 \log^2 A \log^2 \epsilon }{\epsilon^2}}.
\end{align*}

Thus, the sample complexity to reach $\frac{1}{(1-\gamma)^2}C^{\pi^*}\epsilon$ error when logarithmic factors are omitted is
$$ \tilde O\br*{\frac{A^2 C_{\text{max}}^4  }{\br*{1-\gamma}^3\epsilon^4}\br*{\log |{\Pi}^{\text{det}}| + \log \frac{1}{\delta}}}$$

\subsubsection{The Regularized Case}
\paragraph{The euclidean case:}

The error after $N$ iterations is bounded by
\begin{align*}
    \mu v_{\textbf{best}}^N-\mu\vHstar \leq O\br*{\frac{\CmaxLambdaSquare A \log N} {\lambda (1-\gamma)^3 N}+\frac{1}{(1-\gamma)^2}\norm{\frac{d_{\mu,\pi^*}}{\nu}}\frac{\epsilon}{2}},
\end{align*}
Thus, in order to reach an error of  $\frac{1}{(1-\gamma)^2} C^{\pi^*} \epsilon$ error, we need 
\begin{align*}
     N \leq O\br*{\frac{\CmaxLambdaSquare A \log \epsilon} {\lambda (1-\gamma) \epsilon}}
\end{align*}

Thus, the sample complexity to reach $\frac{1}{(1-\gamma)^2} C^{\pi^*} \epsilon$ error when logarithmic factors are omitted is
$$ \tilde O\br*{\frac{A^3 C_{\text{max},\lambda}^4 }{\lambda\br*{1-\gamma}^4\epsilon^3}\br*{\log |{\Pi}^{\text{det}}| + \log \frac{1}{\delta}}}$$

\paragraph{The non-euclidean case:}

The error after $N$ iterations is bounded by
% \begin{align*}
%     \mu v_{\textbf{best}}^N-\mu \vHstar \leq O\br*{\frac{(\mathrm{C^2_{max}}+\lambda^2 \log^2 A) A^2 \log^3 N} {\lambda (1-\gamma)^3 N}}+\frac{1}{1-\gamma}\norm{\frac{d_{\mu,\pi^*}}{\mu}}\epsilon.
% \end{align*}
\begin{align*}
    \mu v_{\textbf{best}}^N-\mu \vHstar \leq O\br*{\frac{\log A}{(1-\gamma)N}+\frac{\CmaxLambdaSquare A^2 \log^3 N}{\lambda(1-\gamma)^3 N}}+\frac{1}{(1-\gamma)^2}\norm{\frac{d_{\mu,\pi^*}}{\nu}}_\infty\frac{\epsilon}{2}.
\end{align*}

Rearranging, we get,
\begin{align*}
    \mu v_{\textbf{best}}^N-\mu \vHstar \leq O\br*{\frac{(\CmaxSquare+\lambda^2\log^2 A)A^2 \log^3 N}{\lambda(1-\gamma)^3 N}}+\frac{1}{(1-\gamma)^2}\norm{\frac{d_{\mu,\pi^*}}{\nu}}_\infty \frac{\epsilon}{2},
\end{align*}
which can also be written with $\CmaxLambdaSquare$
\begin{align*}
    \mu v_{\textbf{best}}^N-\mu \vHstar \leq O\br*{\frac{\CmaxLambdaSquare A^2 \log^3 N}{\lambda(1-\gamma)^3 N}}+\frac{1}{(1-\gamma)^2}\norm{\frac{d_{\mu,\pi^*}}{\nu}}\epsilon.
\end{align*}

Thus, in order to reach an error of  $\frac{1}{(1-\gamma)^2} C^{\pi^*} \epsilon$ error, we need 
\begin{align*}
     N \leq \tilde O\br*{\frac{\CmaxLambdaSquare A^2 } {\lambda (1-\gamma) \epsilon}},
\end{align*}
omitting logarithmic factors.

Thus, the sample complexity to reach $\frac{1}{(1-\gamma)^2} C^{\pi^*} \epsilon$ error when logarithmic factors are omitted is
$$ \tilde O\br*{\frac{ A^4 C_{\text{max},\lambda}^4  }{\lambda \br*{1-\gamma}^4\epsilon^3}\br*{\log |{\Pi}^{\text{det}}| + \log \frac{1}{\delta}}}$$

\section{Useful Lemmas}

The next lemmas will provide useful bounds for uniform, exact and Sample-Based TRPO.
In this section, we define $\norm{\cdot}_*$ to be the dual norm of $\norm{\cdot}$.
\begin{lemma}[Connection between the regularized Bellman operator and the $q$-function]\label{lemma: connection between bellman and q function}
For any $\pi,\pi'$ the following holds:
\begin{align*}
    \inner*{q_\lambda^{\pi} + \lambda\nabla \omega(\pi), \pi' - \pi} = T_\lambda^{\pi'}\vHpi - \vHpi - \lambda \bregman{\pi'}{\pi}
\end{align*}
\begin{proof}

First, note that for any $s$
\begin{align*}
& \inner*{q_\lambda^{\pi}(s,\cdot) , \pi'(\cdot\mid s) } \\
& = \sum_a \pi'(a\mid s) q_\lambda^{\pi}(s,a) \\
&= \sum_a \pi'(a\mid s) \br*{c_\lambda^\pi(s,a) + \gamma \sum_{s'} p(s'|s,a) \vHpi } \\
&=\sum_a \pi'(a\mid s)  \br*{c(s,a) + \lambda \omegaS{s}{\pi} + \gamma \sum_{s'} p(s'|s,a) \vHpi } \\
&= \sum_a \pi'(a\mid s) \br*{c(s,a) + \lambda \omegaS{s}{\pi'} - \lambda \omegaS{s}{\pi'} + \lambda \omegaS{s}{\pi} + \gamma \sum_{s'} p(s'|s,a) \vHpi } \\
&= \sum_a \pi'(a\mid s) \br*{c(s,a) + \lambda \omegaS{s}{\pi'} + \gamma \sum_{s'} p(s'|s,a) \vHpi } + \lambda \omegaS{s}{\pi} - \lambda \omegaS{s}{\pi'} \\
&= c_\lambda^{\pi'}(s) + \gamma P^{\pi'} \vHpi(s) + \lambda \omegaS{s}{\pi} - \lambda \omegaS{s}{\pi'} \\
&= T_\lambda^{\pi'}\vHpi (s) + \lambda \omegaS{s}{\pi} - \lambda \omegaS{s}{\pi'},
\end{align*}
where the second transition is by the definition of $\qHpi$, the third is by the definition of $c_\lambda^\pi$, the fourth is by adding and subtracting $\lambda \omegaS{s}{\pi'}$, the fifth is by the fact $\lambda \omegaS{s}{\pi'}$ is independent of $a$ and the seventh is by the definition of the regularized Bellman operator.

Thus,
$$ \inner*{q_\lambda^{\pi} , \pi' } =  T_\lambda^{\pi'}\vHpi + \lambda \omega(\pi) - \lambda \omega (\pi') $$

Now, note that by the definition of the $q$-function $\inner*{\qHpi,\pi}=\vHpi$ and thus,
$$ \inner*{q_\lambda^{\pi} , \pi' -\pi } =  T_\lambda^{\pi'}\vHpi - \vHpi + \lambda \omega(\pi) - \lambda \omega (\pi') .$$

Finally, by adding to both sides $\inner*{\lambda\nabla \omega (\pi),\pi' -\pi}$, we get,
$$ \inner*{q_\lambda^{\pi} + \lambda \nabla \omega (\pi), \pi' -\pi } =  T_\lambda^{\pi'}\vHpi - \vHpi + \lambda \omega(\pi) - \lambda \omega (\pi') + \lambda\inner*{ \nabla  \omega (\pi),\pi' -\pi} .$$

To conclude the proof, note that by the definition of the Bregman distance we have,
$$ \inner*{q_\lambda^{\pi} + \lambda \nabla \omega (\pi), \pi' -\pi } =  T_\lambda^{\pi'}\vHpi - \vHpi - \lambda \bregman{\pi'}{\pi}.$$

\end{proof}
\end{lemma}

\begin{lemma}[Bounds regarding the updates of Uniform and Exact TRPO]\label{lemma: usefull bounds exact case}
For any $k\geq0$ and state $s$, which is updated in the $k$-th iteration,
the following relations hold for both Uniform TRPO \eqref{eq: minimzation exact TRPO update final} and Exact TRPO~\eqref{eq: supp iterates uniform TRPO}:

\begin{enumerate}
    \item $\norm{\nabla \omega(\pi_k \br*{\cdot|s})}_* \leq O(1)$ and $\norm{\nabla \omega(\pi_k \br*{\cdot|s})}_* \leq O(\frac{\CmaxLambda \log k}{\lambda(1-\gamma)})$, in the euclidean and non-euclidean cases, respectively.
    \item $\norm{\qHpik \br*{s,\cdot}}_*\leq \gradBound$, where $\gradBound=  O(\frac{\sqrt{A}\CmaxLambda}{1-\gamma})$ and $\gradBound=  O(\frac{\CmaxLambda}{1-\gamma})$ in the euclidean and non-euclidean cases, respectively. 
    \item $\norm{\qHpik\br*{s,\cdot} + \lambda \nabla \omega(\pi_k \br*{\cdot|s})}_*\leq \gradBound(k;\lambda)$, where $\gradBound(k;\lambda)=  O(\frac{\sqrt{A}\CmaxLambda}{1-\gamma})$ and $\gradBound(k;\lambda)=  O(\frac{\CmaxLambda \br*{1+\mathbbm{1}\{\lambda \neq 0\}\log k}}{1-\gamma})$ in the euclidean and non-euclidean cases, respectively, and $\mathbbm{1}\{\lambda \neq 0\}=0$ in the unregularized case ($\lambda$=0) and $\mathbbm{1}\{\lambda \neq 0\}=1$ otherwise. 
\end{enumerate}
Where for every state $s$, $\norm{\cdot}_*$ denotes the dual norm over the action space, which is $L_1$ in the euclidean case, and $L_\infty$ in non-euclidean cases.

\begin{proof}

We start by proving the {\bf first claim}:

For the {\bf euclidean case}, $\omega \br*{\cdot}= \frac{1}{2}\norm{\cdot}^2_2$. Thus, for every state $s$,
\begin{align*}
    \norm{\nabla\omega(\pi(\cdot|s))}_2 = \norm{\pi(\cdot|s)}_2\leq\norm{\pi(\cdot|s)}_1= 1,
\end{align*}
where the inequality is due to the fact that $\norm{\cdot}_2\leq\norm{\cdot}_1$.

The statement holds by the properties of $\frac{1}{2}\norm{\cdot}^2_2$ and thus holds for both the uniform and exact versions.

For the {\bf non-euclidean case}, $\omega(\cdot)=H(\cdot)+\log{\aset}$. Now, consider exact TRPO~\eqref{eq: supp iterates uniform TRPO}. By taking the logarithm of \eqref{eq: mirror descent iteration KL}, we have,

\begin{align}
\log\pi_{k}(a\mid s) &= \log \pi_{k-1}(a\mid s) \nonumber \\ &- t_{k-1} \br*{ q_\lambda^{\pi_{k-1}}(s,a) + \lambda \log \pi_{k-1}(a\mid s) } \nonumber \\
&-\log\br*{\sum_{a'} \pi_{k-1}(a'\mid s)\exp\br*{- t_{k-1} \br*{q_\lambda^{\pi_{k-1}}(s,a') +\lambda \log \pi_{k-1}(a'\mid s)}}}.
\label{eq: non-euclidean bound eq 1}
\end{align}

Notice that for $k\geq0$, for every state-action pair, $q_\lambda^{\pi_{k}}(a|s)\geq 0$. Thus, 
\begin{align}
    \log\br*{\sum_{a'} \pi_{k}(a'\mid s)\exp\br*{-t_k\br*{ q_\lambda^{\pi_k}(s,a')+\lambda \log \pi_{k}(a'\mid s)}}} & \leq
    \log\br*{\sum_{a'} \pi_{k}(a'\mid s)\exp\br*{-t_k \lambda \log \pi_{k}(a'\mid s)}}  \nonumber\\
    &= \log\br*{\sum_{a'} \pi_{k}(a'\mid s) \pi_{k}^{-\lambda t_k}(a'\mid s)}.\label{eq: bound on logsumexp first}
\end{align}
Where the first relation holds since $q_\lambda^{\pi}(s,a)\geq 0 $. Applying Jensen's inequality we can further bound the above.
\begin{align}
    \eqref{eq: bound on logsumexp first}&= \log\br*{A\sum_{a'}\frac{1}{A} \pi_{k}^{1-\lambda t_k}(a'\mid s)} \nonumber \\
    &= \log\br*{A\sum_{a'} \frac{1}{A}\pi_{k}^{1-\lambda t_k}(a'\mid s)} \nonumber \\
    &\leq \log\br*{A \left(\sum_{a'} \frac{1}{A}\pi_{k}(a'\mid s) \right)^{1-\lambda t_k}} \nonumber \\
    &= \log\br*{A \left(\frac{1}{A}\sum_{a'} \pi_{k}(a'\mid s) \right)^{1-\lambda t_k}} \nonumber \\
    &= \log\br*{A \left(\frac{1}{A} \right)^{1-\lambda t_k}} = \log\br*{A^{\lambda t_k} } = \lambda t_k \log A.  \label{eq: bound on logsumexp}
\end{align}
In the third relation we applied Jensen's inequality for concave functions. As $0\leq 1-\lambda t_k\leq 1$ (by the choice of the learning rate in the regularized case) we have that $X^{1-\lambda t_k}$ is a concave function in $X$, and thus ${\sum_{a'=1}^{A}\frac{1}{A}\pi_{k}^{1-\lambda t_k}(a'\mid s) \leq \left(\sum_{a'=1}^A \frac{1}{A}\pi_{k}(a'\mid s) \right)^{1-\lambda t_k}}$ by Jensen's inequality. Combining this inequality with the fact that $A$ is positive and $\log$ is monotonic function establishes the third relation.

Furthermore, note that for every $k$, and for every $s,a$
\begin{align}\label{eq: upper bound on log pi}
    \log \pi_k(a|s) \leq 0 
\end{align}

Plugging \eqref{eq: bound on logsumexp} and \eqref{eq: upper bound on log pi} in \eqref{eq: non-euclidean bound eq 1}, we get,

\begin{align}\label{eq: exact lower bound on log pi k}
    \log \pi_{k}(a\mid s) &\geq \log \pi_{k-1}(a\mid s) -t_{k-1} \br*{q_\lambda^{\pi_{k-1}}(s,a) +\lambda \log A} \nonumber\\
    & \geq \log \pi_0(a|s) - \sum_{i=0}^{k-1} t_k   \br*{q_\lambda^{\pi_i}(s,a) + \lambda \log A} \nonumber\\
    & \geq - \log A - \br*{\frac{\CmaxLambda}{1-\gamma} + \lambda \log A}\sum_{i=0}^{k-1} t_i \nonumber\\
    & = - \log A - \br*{\frac{\CmaxLambda}{\lambda(1-\gamma)} +\log A}\sum_{i=0}^{k-1} \frac{1}{i+2} \nonumber\\
    & \geq - \log A - \br*{\frac{\CmaxLambda}{\lambda(1-\gamma)}+ \log A} (1+\log k) \nonumber \\
    &  \geq - \frac{\mathrm{C_{\text{max}}}+3\lambda \log    A}{\lambda(1-\gamma)} (1+\log k) \nonumber \\ 
    &  \geq - \frac{\mathrm{3\CmaxLambda} }{\lambda(1-\gamma)} (1+\log k),
\end{align} 
where the second relation holds by unfolding the recursive formula for each $k$ and the fourth by plugging in the stepsizes for the regularized case, i.e. $t_k=\frac{1}{\lambda(k+2)}$. The final relation holds since $\CmaxLambda =\mathrm{C_{\text{max}} }+\lambda \log{A}$.

To conclude, since $\log \pi_k(a\mid s)\leq 0$ and $\nabla \omega(\pi)=\nabla H(\pi)= 1+ \log \pi$, we get that for the non-euclidean case, $$\norm{\nabla \omega(\pi_k)}_\infty \leq O\br*{\frac{\CmaxLambda}{\lambda(1-\gamma)}\log k}.$$

This concludes the proof of the first claim for both the euclidean and non-euclidean cases, in both exact scenarios. Interestingly, in the non-euclidean case, the gradients can grow to infinity due to the fact that the gradient of the entropy of a deterministic policy is unbounded. However, this result shows that a deterministic policy can only be obtained after an infinite time, as the gradient is bounded by a logarithmic rate.

Next, we prove the {\bf second claim}:

It holds that for any state-action pair $ q_\lambda^{\pi_{k}}\br*{s,a} \in \brs*{0,\frac{\CmaxLambda}{1-\gamma}} $.

For the {\bf euclidean case}, we have that
$$ \norm{ q_\lambda^{\pi_{k}}\br*{s,\cdot}}_* = \norm{ q_\lambda^{\pi_{k}}\br*{s,\cdot}}_2 \leq \sqrt{\sum_{a\in \aset} \br*{\frac{\CmaxLambda}{1-\gamma}}^2}=\frac{\sqrt{A}\CmaxLambda}{1-\gamma}. $$

For the {\bf non-euclidean case}, we have that
$$ \norm{ q_\lambda^{\pi_{k}}\br*{s,\cdot}}_* = \norm{ q_\lambda^{\pi_{k}}\br*{s,\cdot}}_\infty \leq \frac{\CmaxLambda}{1-\gamma}, $$
which concludes the proof of the second claim.

Finally, we prove the {\bf third claim}:
For any state $s$, by the triangle inequality,
$$\norm{ q_\lambda^{\pi_{k}}\br*{s,\cdot} + \lambda\nabla \omega(\pi_k \br*{\cdot|s})}_* \leq \norm{ q_\lambda^{\pi_{k}}\br*{s,\cdot}}_* + \lambda\norm{\nabla \omega(\pi_k\br*{\cdot|s})}_* ,$$
by plugging the two former claims for the euclidean and non-euclidean cases, we get the required result.
\end{proof}
\end{lemma}

The next lemma follows similar derivation to Lemma \ref{lemma: usefull bounds exact case}, with small changes tailored for the sample-based case. Note that in the sample-based case, and $A$ factor is added in claims 1,3 and 4.

\begin{lemma}[Bounds regarding the updates of Sample-Based TRPO]\label{lemma: usefull bounds approximate}
For any $k\geq0$ and state $s$, which is updated in the $k$-th iteration,
the following relations hold for Sample-Based TRPO~\eqref{eq: update rule approximate final}:

\begin{enumerate}
    \item $\norm{\nabla \omega(\pi_k \br*{\cdot|s})}_* \leq O(1)$ and $\norm{\nabla \omega(\pi_k \br*{\cdot|s})}_* \leq O(\frac{A\CmaxLambda \log k}{\lambda(1-\gamma)})$, in the euclidean and non-euclidean cases, respectively.
    \item $\norm{\qHpik \br*{s,\cdot}}_*\leq \gradBound$, where $\gradBound=  O(\frac{\sqrt{A}\CmaxLambda}{1-\gamma})$ and $\gradBound=  O(\frac{\CmaxLambda}{1-\gamma})$ in the euclidean and non-euclidean cases, respectively. 
    \item $\norm{\qHpik\br*{s,\cdot} + \lambda \nabla \omega(\pi_k \br*{\cdot|s})}_*\leq \gradBound(k;\lambda)$, where $\gradBound(k;\lambda)=  O(\frac{\sqrt{A}\CmaxLambda}{1-\gamma})$ and $\gradBound(k;\lambda)=  O(\frac{\CmaxLambda \br*{1+ \mathbbm{1}\{\lambda \neq 0\} A\log k}}{1-\gamma})$ in the euclidean and non-euclidean cases, respectively, and $\mathbbm{1}\{\lambda \neq 0\}=0$ in the unregularized case ($\lambda$=0) and $\mathbbm{1}\{\lambda \neq 0\}=1$ in the regularized case ($\lambda>0$).
    \item $\norm{A\hatqHpik\br*{s,\cdot,m} + \lambda \nabla \omega(\pi_k \br*{\cdot|s})}_\infty\leq \hatgradBound(k;\lambda)$, where $\hatgradBound(k;\lambda)=  O(\frac{A\CmaxLambda}{1-\gamma})$ and $\hatgradBound(k;\lambda)=  O(\frac{A\CmaxLambda \br*{1+\mathbbm{1}\{\lambda\neq 0
    \}\log k}}{1-\gamma})$ in the euclidean and non-euclidean cases, respectively, and $\mathbbm{1}\{\lambda \neq 0\}=0$ in the unregularized case ($\lambda$=0) and $\mathbbm{1}\{\lambda \neq 0\}=1$ in the regularized case ($\lambda>0$).
\end{enumerate}
Where for every state $s$, $\norm{\cdot}_*$ denotes the dual norm over the action space, which is $L_1$ in the euclidean case, and $L_\infty$ in non-euclidean cases.
\end{lemma}

\begin{proof}

We start by proving the {\bf first claim}:

For the {\bf euclidean case}, in the same manner as in the exact cases, $\omega \br*{\cdot}= \frac{1}{2}\norm{\cdot}^2_2$. Thus, for every state $s$,
\begin{align*}
    \norm{\nabla\omega(\pi(\cdot|s))}_2 = \norm{\pi(\cdot|s)}_2\leq\norm{\pi(\cdot|s)}_1= 1,
\end{align*}
where the inequality is due to the fact that $\norm{\cdot}_2\leq\norm{\cdot}_1$.

For the {\bf non-euclidean case}, $\omega(\cdot)=H(\cdot)+\log{\aset}$. 
The bound for the sample-based version for the non-euclidean choice of $\omega$ follows similar reasoning with mild modification. By \eqref{supp eq: approximate simplified first}, in the sample-based case, a state $s$ is updated in the $k$-th iteration using the approximation of the $\qHpik(s,a)$ in this state,
\begin{align*}
    \hatqHpik(s,a) \triangleq \frac{A\sum_{m=1}^{M} \mathbbm{1}\brc*{s=s_m,a=a_m} \hat q_\lambda^{\pik}(s_m,\cdot,m)}{n(s)} \leq  \frac{A\sum_{m=1}^{M} \mathbbm{1}\brc*{s=s_m,a=a_m}  \frac{\CmaxLambda}{1-\gamma}}{n(s)} \leq \frac{A \CmaxLambda}{1-\gamma},
\end{align*}
where we denoted $n(s)=\sum_{a}n(s,a)$ the number of times the state $s$ was observed at the $k$-th episode and used the fact $\hatqHpik(s_m,\cdot,m_i)$ is sampled by unrolling the MDP. Thus, it holds that $$ \hatqHpik(s,a)\leq \frac{A\CmaxLambda}{1-\gamma}.$$
Interestingly, because we use the importance sampling factor $A$ in the approximation of $\qHpik$, we obtain an additional $A$ factor.

Thus, by repeating the analysis in Lemma \ref{lemma: usefull bounds exact case}, equation \eqref{eq: exact lower bound on log pi k}, we obtain,
\begin{align}\label{eq: bound on log pik in the sample based case}
    \log(\pi_{k}(a\mid s)) &\geq \log (\pi_{k-1}(a\mid s)) - t_{k-1} \br*{\hat q_\lambda^{\pi_{k-1}}(s,a) + \lambda \log A} \nonumber\\
    & \geq \log \pi_0(a|s) - \sum_{i=0}^{k-1} t_i \br*{\hat q_\lambda^{\pi_i}(s,a) + \lambda \log A} \nonumber \\
    & \geq - \log A - \br*{\frac{A\CmaxLambda}{1-\gamma}+\lambda \log A}\sum_{i=0}^{k-1} t_i \nonumber\\
    & = - \log A - \br*{\frac{A\CmaxLambda}{\lambda(1-\gamma)}+\log A}\sum_{i=0}^{k-1} \frac{1}{i+2} \nonumber\\
    & \geq - \log A - \br*{\frac{A\CmaxLambda}{\lambda(1-\gamma)}+\log A} (1+\log k)  \nonumber\\
    &  \geq - \frac{A\mathrm{C_{\text{max}}}+3A\lambda \log A}{\lambda(1-\gamma)} (1+\log k) \nonumber \\
    & \geq - \frac{3 A\CmaxLambda}{\lambda(1-\gamma)}, 
\end{align} 
where the second relation holds by unfolding the recursive formula for each $k$ and the fourth by plugging in the stepsizes for the regularized case, i.e. $t_k=\frac{1}{\lambda(k+2)}$. The final relation holds since $\CmaxLambda =\mathrm{C_{\text{max}} }+\lambda \log{A}$.
Thus,
\begin{align*}
    \log(\pi_{k}(a\mid s)) &  \geq - \frac{3A \CmaxLambda}{\lambda(1-\gamma)} (1+\log k),
\end{align*} 
This concludes the proof of the first claim for both the euclidean and non-euclidean cases.

As in the exact case, in the non-euclidean case, the gradients can grow to infinity due to the fact that the gradient of the entropy of a deterministic policy is unbounded. However, this result shows that a deterministic policy can only be obtained after an infinite time, as the gradient is bounded by a logarithmic rate.

Next, we prove the {\bf second claim}:

It holds that for any state-action pair $ q_\lambda^{\pi_{k}}\br*{s,a} \in \brs*{0,\frac{\CmaxLambda}{1-\gamma}} $.

For the {\bf euclidean case}, we have that
$$ \norm{ q_\lambda^{\pi_{k}}\br*{s,\cdot}}_* = \norm{ q_\lambda^{\pi_{k}}\br*{s,\cdot}}_2 \leq \sqrt{\sum_{a\in \aset} \br*{\frac{\CmaxLambda}{1-\gamma}}^2}=\frac{\sqrt{A}\CmaxLambda}{1-\gamma}. $$

For the {\bf non-euclidean case}, we have that
$$ \norm{ q_\lambda^{\pi_{k}}\br*{s,\cdot}}_* = \norm{ q_\lambda^{\pi_{k}}\br*{s,\cdot}}_\infty \leq \frac{\CmaxLambda}{1-\gamma}, $$
which concludes the proof of the second claim.

Next, we prove the {\bf third claim}:
For any state $s$, by the triangle inequality,
$$\norm{ q_\lambda^{\pi_{k}}\br*{s,\cdot} + \lambda\nabla \omega(\pi_k \br*{\cdot|s})}_* \leq \norm{ q_\lambda^{\pi_{k}}\br*{s,\cdot}}_* + \lambda\norm{\nabla \omega(\pi_k\br*{\cdot|s})}_* ,$$
by plugging the two former claims for the euclidean and non-euclidean cases, we get the required result.

Finally, the {\bf fourth claim} is the same as the third claim, but with an additional $A$ factor due to the importance sampling factor,
$$\norm{ A\hat q_\lambda^{\pi_{k}}\br*{s,\cdot,m} + \lambda\nabla \omega(\pi_k \br*{\cdot|s})}_\infty \leq A\norm{ \hat q_\lambda^{\pi_{k}}\br*{s,\cdot,m}}_\infty + \lambda\norm{\nabla \omega(\pi_k\br*{\cdot|s})}_\infty .$$
\end{proof}

Using the same techniques of the last lemma, we prove the following technical lemma, regarding the change in the gradient of the Bregman generating function $\omega$ of two consecutive iterations of TRPO, in the sample-based case.

\begin{lemma}[{bound on the difference of the gradient of $\omega$ between two consecutive policies in the sample-based case}]\label{lemma: difference of gradient of omegas approximate}
For each state-action pair, $s,a$, the difference between two consecutive policies of Sample-Based TRPO is bounded by:
\begin{align*}
    \norm{\nabla \omega({\pi_{k+1}}) - \nabla \omega({\pik})}_{\infty,\infty} \leq \diffBound(k),
\end{align*}
where $\diffBound(k)=t_k \frac{A^{3/2}\CmaxLambda}{1-\gamma}$ and $\diffBound(k)=t_k \frac{A\CmaxLambda \log k}{1-\gamma}$ in the euclidean and non-euclidean cases respectively, $k$ is the iteration number and $t_k$ is the step size used in the update.

\begin{proof}

In both the euclidean in non-euclidean cases, we discuss optimization problem \eqref{eq: update rule approximate final} for the sample-based case. Thus, 
for any visited state in the $k$-th iteration, $s\in \mathcal{S}^k_M \triangleq \brc*{s'\in \sset: \sum_{m=1}^{M} \mathbbm{1}\brc*{s'=s_m}>0}$, by \eqref{supp eq: approximate simplified first}
\begin{align*}
    \hatqHpik(s,a) \triangleq \frac{A\sum_{m=1}^{M} \mathbbm{1}\brc*{s=s_m,a=a_m}  \hat q_\lambda^{\pik}(s_m,\cdot,m)}{n(s)} \leq  \frac{A\sum_{m=1}^{M} \mathbbm{1}\brc*{s=s_m,a=a_m}  \frac{\CmaxLambda}{1-\gamma}}{n(s)} \leq \frac{A \CmaxLambda}{1-\gamma},
\end{align*}
% $$\hatqHpik(s,a)= \frac{1}{n(s,a)}\sum_{i=1}^{n(s,a)} \hat{q}_\lambda^{\pi_k}(s,a; m_i),$$
% where $n(s,a)$ is the number of visits to the state-action $s,a$, and $m_i$ is the trajectory number where the $i$-th visit occured.

where we denoted $n(s)=\sum_{a}n(s,a)$ the number of times the state $s$ was observed at the $k$-th episode and used the fact $\hatqHpik(s_m,\cdot,m_i)$ is sampled by unrolling the MDP. Thus, it holds that $$\hatqHpik(s,a)\leq \frac{A\CmaxLambda}{1-\gamma}.$$
Interestingly, because we use the importance sampling factor $A$ in the approximation of $\qHpik$, we obtain an additional $A$ factor.

First, notice that for states which were not encountered in the $k$-th iteration, i.e., all states $s$ for which $\sum_{m=1}^{M} \mathbbm{1}\brc*{s=s_m}=0$, the solution of the optimization problem is $\pi_{k+1}(\cdot\mid s)=\pi_{k}(\cdot\mid s)$. Thus, $\nabla \omegaS{s}{\pi_{k+1}}=\nabla \omegaS{s}{\pi_{k}}$ and the inequality trivially holds.

We now turn to discuss the case where $\sum_{m=1}^{M} \mathbbm{1}\brc*{s=s_m}>0$, i.e., $s \in \mathcal{S}^k_M$. We separate here the analysis for the euclidean and non-euclidean cases:
 
For the \textbf{euclidean case}, $\omega(\cdot)=\frac{1}{2}\norm{\cdot}_2^2$. Thus, the derivative of $\omega$ at a state $s$ is,
\begin{align}\label{eq: derivative of euclidean bregman generator approximate}
    \nabla \omegaS{s}{\pi} = \pi(\cdot \mid s).
\end{align}
By the first order optimality condition, for any state $s$ and policy $\pi$,
$$  \inner*{\nabla \omegaS{s}{\pi_{k+1}} - \nabla \omegaS{s}{\pi_k}, \pi_{k+1} (\cdot\mid s) - \pi}\leq t_k \inner*{\hat q_\lambda^{\pi_k}(s , \cdot) +\lambda \nabla \omegaS{s}{\pi_k} , \pi - \pi_{k+1}(\cdot\mid s)}. $$
Plugging in $\pi\triangleq \pi_k(\cdot \mid s)$, we get
\begin{align*} 
&\inner*{\nabla \omegaS{s}{\pi_{k+1}} - \nabla \omegaS{s}{\pi_k}, \pi_{k+1} (\cdot\mid s) - \pi_k(\cdot\mid s)}\leq t_k \inner{\hat q_\lambda^{\pi_k}(s , \cdot) + \lambda \nabla \omegaS{s}{\pi_k}, \pi_k(\cdot\mid s) - \pi_{k+1}(\cdot\mid s)}. 
\end{align*}
Plugging in \eqref{eq: derivative of euclidean bregman generator approximate}, we have that
\begin{align*} 
&\inner*{\nabla \omegaS{s}{\pi_{k+1}} - \nabla \omegaS{s}{\pi_k}, \nabla \omegaS{s}{\pi_{k+1}} - \nabla \omegaS{s}{\pi_k}}\leq t_k \inner*{\hat q_\lambda^{\pi_k}(s , \cdot) + \lambda \nabla \omegaS{s}{\pi_k} , \pi_k(\cdot\mid s) - \pi_{k+1}(\cdot\mid s)} ,
\end{align*}
which can be also written as
\begin{align*}
\norm{\nabla \omegaS{s}{\pi_{k+1}} - \nabla \omegaS{s}{\pi_k}}_2^2 \leq t_k \inner*{\hat q_\lambda^{\pi_k}(s , \cdot) + \lambda \nabla \omegaS{s}{\pi_k} , \pi_k(\cdot\mid s) - \pi_{k+1}(\cdot\mid s)}.
\end{align*}
Bounding the RHS using the Cauchy-Schwartz inequality, we get,
\begin{align*}
\norm{\nabla \omegaS{s}{\pi_{k+1}} - \nabla \omegaS{s}{\pi_k}}_2^2 \leq t_k \norm{\hat q_\lambda^{\pi_k}(s , \cdot) + \lambda \nabla \omegaS{s}{\pi_k}}_2 \norm{\pi_k(\cdot\mid s) - \pi_{k+1}(\cdot\mid s)}_2,
\end{align*}
which is the same as
\begin{align*}
\norm{\nabla \omegaS{s}{\pi_{k+1}} - \nabla \omegaS{s}{\pi_k}}_2^2 \leq t_k \norm{\hat q_\lambda^{\pi_k}(s , \cdot) + \lambda \nabla \omegaS{s}{\pi_k}}_2 \norm{\nabla \omegaS{s}{\pi_{k+1}} - \nabla \omegaS{s}{\pi_k}}_2,
\end{align*}
Dividing by $\norm{\nabla \omegaS{s}{\pi_{k+1}} - \nabla \omegaS{s}{\pi_k}}_2>0$ and noticing that in case it is $0$ the bound is trivially satisfied,
\begin{align*}
\norm{\nabla \omegaS{s}{\pi_{k+1}} - \nabla \omegaS{s}{\pi_k}}_2 \leq t_k \norm{\hat q_\lambda^{\pi_k}(s , \cdot) + \lambda \nabla \omegaS{s}{\pi_k}}_2.
\end{align*}
Finally, using the norm equivalence we get,
\begin{align*}
\norm{\nabla \omegaS{s}{\pi_{k+1}} - \nabla \omegaS{s}{\pi_k}}_\infty\leq \norm{\nabla \omegaS{s}{\pi_{k+1}} - \nabla \omegaS{s}{\pi_k}}_2 \leq t_k \norm{\hat q_\lambda^{\pi_k}(s , \cdot) + \lambda \nabla \omegaS{s}{\pi_k}}_2.
\end{align*}
Using the fourth claim of Lemma \ref{lemma: usefull bounds approximate} (in the euclidean setting), and the fact the this inequality holds uniformly for all $s \in \mathcal{S}^k_M$ concludes the result.

For the \textbf{non-euclidean case}, $ \omegaS{s}{\pi} = \sum_a \pi(a \mid s)\log\pi(a\mid s)$. Thus, the derivative at the state action pair, $s,a$, is
\begin{align*}
    \nabla_{\pi(a\mid s)} \omegaS{s}{\pi} = 1 + \log \pi(a \mid s).
\end{align*}
Thus, the difference between two consecutive policies is:
\begin{align*}
    \nabla_{\pi_{k+1}(a\mid s)} \omegaS{s}{\pi_{k+1}} - \nabla_{\pi_k(a\mid s)} \omegaS{s}{\pi_{k}} = \log \pi_{k+1}(a \mid s) - \log \pi_{k}(a \mid s)
\end{align*}

Restating \eqref{eq: non-euclidean bound eq 1},
\begin{align*}
\log\pi_{k+1}(a\mid s) &= \log \pi_{k}(a\mid s) \nonumber \\ &- t_{k} \br*{ \hat q_\lambda^{\pi_{k}}(s,a) + \lambda \log \pi_{k}(a\mid s) } \nonumber \\
&-\log\br*{\sum_{a'} \pi_{k}(a'\mid s)\exp\br*{- t_{k} \br*{\hat q_\lambda^{\pi_{k}}(s,a') +\lambda \log \pi_{k}(a'\mid s)}}}.
\end{align*}

First, we will bound $\log\pi_{k+1}(a\mid s)- \log \pi_{k}(a\mid s)$ from below:

Similarly to equation \ref{eq: bound on logsumexp}, bounding the last term in the RHS,
\begin{align*}
\log\br*{\sum_{a'} \pi_{k}(a'\mid s)\exp\br*{-t_{k} \br*{\hat q_\lambda^{\pi_{k}}(s,a')+\lambda \log \pi_k(a' \mid s) }}} \leq t_k\lambda\log A.
\end{align*}
Together with the fact that $\lambda t_k\log \pi_k(a\mid s) \leq 0 $, we obtain,
$$ \log\pi_{k+1}(a\mid s) - \log \pi_{k}(a\mid s)  \geq -  t_{k} \br*{\hat q_\lambda^{\pi_{k}}(s,a) +\lambda \log A}\geq -t_k \br*{\frac{A \CmaxLambda}{1-\gamma} +\lambda \log A} \geq -2t_k \frac{A \CmaxLambda}{1-\gamma}, $$ 
where the last relation is by the definition of $\CmaxLambda$

Next, it is left to bound $\log\pi_{k+1}(a\mid s) - \log \pi_{k}(a\mid s)$  from above.
Notice that,
\begin{align*}
\log\br*{\sum_{a'} \pi_{k}(a'\mid s)\exp\br*{-t_{k} \br*{\hat q_\lambda^{\pi_{k}}(s,a')+\lambda \log \pi_k(a' \mid s) }}} &\geq \log \sum_{a'} \pi_k(a' \mid s) \exp \br*{-t_k \frac{A\CmaxLambda}{1-\gamma} -\lambda t_k \log \pi (a' \mid s) } \\
& \geq \log \sum_{a'} \pi_k (a' \mid s) \exp \br*{-t_k \frac{A\CmaxLambda}{1-\gamma}  }  \\
& = \log \sum_{a'} \pi_k (a' \mid s)  + \log\exp \br*{-t_k \frac{A\CmaxLambda}{1-\gamma}  }  \\
 & = -t_k \frac{A\CmaxLambda}{1-\gamma} ,
\end{align*}

where in the first transition we used the fact that in the sample-based case $\norm{\hat q_\lambda^{\pi_k}}_{\infty,\infty}\leq \frac{A\CmaxLambda}{1-\gamma} $ due to the importance sampling applied in the estimation process, in the second transition we used the fact that the exponent is minimized when $\lambda t_k \log \pi(a'|s)$ is maximized and the fact that $\log \pi(a'|s)\leq 0$, and the last transition is by the fact $\sum_{a'} \pi_k(a'|s) = 0$.

Thus, we have
\begin{align*}
    \log \pi_{k+1}(a\mid s) - \log \pi_{k}(a\mid s) 
    &\leq  - t_k \br*{ \hat q_\lambda^{\pi_{k}}(s,a) + \lambda \log \pi_{k}(a\mid s) } + t_k \frac{A \CmaxLambda}{1-\gamma} \\
    & \leq  t_k \frac{A\CmaxLambda}{1-\gamma} - \lambda t_k \log \pi_{k}(a\mid s) \\
    & \leq  t_k \frac{A\CmaxLambda}{1-\gamma} + \lambda t_k \frac{A \Cmax + 2A\lambda \log A}{\lambda (1-\gamma)}(1+\log k)\\
    & \leq  t_k \frac{4 A\CmaxLambda \log k}{1-\gamma} ,
\end{align*}
where the third transition is due to \eqref{eq: bound on log pik in the sample based case}, and the last transition is by the the definition of $\CmaxLambda$.

Combining the two bounds we have,
\begin{align*}
-2t_k \frac{A\CmaxLambda}{1-\gamma} &\leq \log\pi_{k+1}(a\mid s) - \log \pi_{k}(a\mid s) \leq  4 t_{k}\frac{A\CmaxLambda}{1-\gamma}\log k\\
\iff -2t_k \frac{A\CmaxLambda}{1-\gamma} &\leq 1+ \log\pi_{k+1}(a\mid s) - (1-\log \pi_{k}(a\mid s) \leq  4 t_{k}\frac{A\CmaxLambda}{1-\gamma}\log k\\
\iff -2t_k \frac{A\CmaxLambda}{1-\gamma} &\leq \nabla_{\pi_{k+1}(a\mid s)} \omegaS{s}{\pi_{k+1}} - \nabla_{\pi_{k}(a\mid s)} \omegaS{s}{\pi_{k}}  \leq 4 t_{k}\frac{A\CmaxLambda}{1-\gamma}\log k,
\end{align*}
which concludes the proof.

\end{proof}
\end{lemma}

\begin{lemma}[{bounds on initial distance $\omegaBound$}]\label{lemma: bounds on Dw}
Let $\pi_0$ be the uniform policy over all states, and $D_\omega$ be an upper bound on $\max_\pi\norm{\bregman{\pi_0}{\pi}}_\infty$, i.e., ${ \max_\pi\norm{\bregman{\pi_0}{\pi}} \leq D_\omega}$. Then, the following claims hold.
\begin{enumerate}
    \item For $\omega(\cdot)=\frac{1}{2}\norm{\cdot}_2^2$, $\omegaBound =  1.$
    \item For $\omega(\cdot)=H(\cdot)$, $\omegaBound = \log A.$
\end{enumerate}
\end{lemma}
\begin{proof}
For brevity, without loss of generality we omit the dependency on the state $s$. We start by proving the first claim. For the euclidean case,
\begin{align*}
   \bregman{\pi}{\pi_0} &= \frac{1}{2}\norm{\pi-\pi_0}_2^2\\
   &=\frac{1}{2}\sum_{a}(\pi(a)-\frac{1}{A})^2\\
   &\leq\frac{1}{2}\sum_{a}\pi^2(a)+\sum_{a}\frac{1}{A^2}\\
   &=\frac{1}{2A} + \frac{1}{2}\sum_{a}\pi^2(a)\\
   &\leq \frac{1}{2A} + \frac{1}{2}\sum_{a}\pi(a)= \frac{1}{2A} +\frac{1}{2},
\end{align*}
where the fifth relation holds since $x^2\leq x$ for $x\in [0,1]$, and the sixth relation holds since $\pi$ is a probability measure.

For the non-euclidean case the following relation holds.
\begin{align*}
    \bregman{\pi}{\pi_0} &= \dkl{\pi}{\pi_0}\\
                        &= \sum_a \pi(a)\log A\pi(a)\\
                        &= \sum_a \pi(a)\log \pi(a) + \sum_a \pi(a)\log A\\
                        &= \sum_a \pi(a)\log \pi(a) + \log A\sum_a \pi(a)\\
                        &= H(\pi) + \log A,
\end{align*}
where $H$ is the negative entropy. Since $H(\pi)\leq 0$ we get that $\bregman{\pi}{\pi_0}\leq \log A$ and conclude the proof.

\end{proof}

The following Lemma as many instances in previous literature (e.g., \cite{scherrer2014local}[Lemma 1]) in the unregularized case, when $\lambda=0$. Here we generalize it to the regularized case, for $\lambda>0$.
\begin{lemma}[value difference to Bellman differences]\label{lemma: greedy to difference}
For any policies $\pi$ and $\pi'$, the following claims hold:
\begin{enumerate}
    \item $v_\lambda^{\pi'}-v_\lambda^\pi = (I-\gamma P^{\pi'})^{-1}(T_\lambda^{\pi'} v_\lambda^\pi - v_\lambda^\pi).$
    \item $T_\lambda^{\pi'} v_\lambda^\pi - v_\lambda^\pi = (I-\gamma P^{\pi'})(v_\lambda^{\pi'}-v_\lambda^\pi).$
    \item $\mu\br*{v_\lambda^{\pi'}-v_\lambda^\pi} = \frac{1}{1-\gamma}d_{\mu,\pi'}(T_\lambda^{\pi'} v_\lambda^\pi - v_\lambda^\pi)$.
\end{enumerate}

\begin{proof}
The first claim holds by the following relations.
\begin{align*}
     v_\lambda^{\pi'}-v_\lambda^\pi &= (I - \gamma P^{\pi'})^{-1} c_\lambda^{\pi'} - (I-\gamma P^{\pi'})^{-1} (I-\gamma P^{\pi'}) v_\lambda^\pi \\
     & = (I-\gamma P^{\pi'})^{-1} (c_\lambda^{\pi'} + \gamma P^{\pi'} v_\lambda^\pi - v_\lambda^\pi)\\
     &= (I-\gamma P^{\pi'})^{-1} (T_\lambda^{\pi'} v_\lambda^\pi - v_\lambda^\pi).
\end{align*}

The second claim follows by multiplying both sides by $(I-\gamma P^{\pi'})$. The third claim holds by multiplying both sides of the first claim by $\mu$ and using the definition $d_{\mu,\pi'}=(1-\gamma)\mu(I-\gamma P^{\pi'})^{-1}$.
\end{proof}
\end{lemma}

\section{Useful Lemmas from Convex Analysis}\label{supp: Useful Lemmas from Convex Analysis}

We state two basic results which are essential to the analysis of  convergence. A full proof can be found in \cite{beck2017first}.
\begin{lemma}[Beck 2017, Lemma 9.11, three-points lemma]\label{lemma: three-points} Suppose that $\omega:\E\rightarrow(-\infty,\infty]$ is proper closed and convex. Suppose in addition that $\omega$ is differentiable over $\dom(\partial\omega)$. Assume that $\mathbf{a},\mathbf{b}\in\dom(\partial\omega)$ and $\textbf{c}\in\dom(\omega)$. Then the following equality holds:
\begin{align*}
    \langle \nabla\omega(\mathbf{b})- \nabla\omega(\mathbf{a}), \mathbf{c} - \mathbf{a} \rangle = \bregman{\mathbf{c}}{\mathbf{a}} + \bregman{\mathbf{a}}{\mathbf{b}} - \bregman{\mathbf{c}}{\mathbf{b}}.
\end{align*}
\end{lemma}

\begin{theorem}[Beck 2017, Theorem 9.12, non-euclidean second prox theorem]\label{theorem: second prox}
\begin{itemize} 
    \item[]
    \item $\omega:\E\rightarrow (-\infty,\infty]$ be a proper closed and convex function differentiable over $\dom(\partial\omega)$.
    
    \item $\psi:\E\rightarrow(-\infty,\infty]$ be a proper closed and convex function satisfying $\dom(\psi) \subseteq \dom(\omega)$.
    
    \item $\omega+\delta_{\dom(\psi)}$ be $\sigma$-strongly convex ($\sigma>0$).
    
\end{itemize}
Assume that $\mathbf{b}\in\dom(\partial\omega)$, and let $\mathbf{a}$ be defined by
\begin{align*}
    \mathbf{a}={\arg\min}_{\mathbf{x}\in \E} \{\psi(\mathbf{x})+\bregman{\mathbf{x}}{\mathbf{b}}\}.
\end{align*}

Then $\mathbf{a}\in\dom(\partial\omega)$ and for all $\mathbf{u}\in\dom(\psi)$,
\begin{align*}
        \langle \nabla\omega(\mathbf{b})- \nabla\omega(\mathbf{a}), \mathbf{u} - \mathbf{a} \rangle \leq \psi(\mathbf{u}) -\psi(\mathbf{a}).
\end{align*}
\end{theorem}

\end{appendices}

\end{document}